\documentclass[12pt]{article}
\usepackage[utf8]{inputenc}
\usepackage{hyperref}
\usepackage{amsmath, amssymb, amsthm, amsfonts}
\usepackage{times}
\usepackage{xcolor}
\usepackage{graphicx}
\usepackage[letterpaper, margin=1.25in]{geometry}
\usepackage{enumitem}
\usepackage{cleveref}
\usepackage{titlesec}
\usepackage{mathrsfs}
\usepackage{physics}
\usepackage{booktabs}
\usepackage{tensor}
\usepackage{algpseudocode,algorithm}
\usepackage[numbers]{natbib}
\renewcommand{\bar}{\overline}

\newcommand{\mbf}{\bm}

\titlespacing*{\section}{0pt}{6pt}{0pt}
\titlespacing*{\subsection}{0pt}{6pt}{0pt}

\DeclareMathOperator{\E}{\mathbb{E}}

\newcommand{\R}{\mathbb{R}}

\newcommand{\1}{\mathbf{1}}

\DeclareMathOperator{\diag}{diag}

\DeclareMathOperator{\unif}{Unif}

\newcommand{\0}{\mathbf{0}}
\newtheorem{lemma}{Lemma}
\newtheorem{corollary}{Corollary}
\newtheorem{theorem}{Theorem}

\newtheorem{definition}{Definition}

\theoremstyle{remark}

\ifdefined\usebigfont

\usepackage{times}
\usepackage[fontsize=13pt]{scrextend}
\AtBeginDocument{
	\newgeometry{letterpaper,left=1.56in,right=1.56in,top=1.71in,bottom=1.77in}
}
\else
\fi
\DeclareMathOperator{\sts}{STS}
\DeclareMathOperator{\sm}{softmax}
\DeclareMathOperator{\supp}{supp}

\newcommand{\qsa}{\sts_q}
\newcommand{\qsat}{\sts_{q_2}}
\newcommand{\X}{\mbf{X}}

\renewcommand{\v}{\mbf{v}}
\newcommand{\x}{\mbf{x}}
\newcommand{\y}{\mbf{y}}

\newcommand{\z}{{\mbf{z}}}

\newcommand{\bXi}{{\mbf{\xi}}}

\newcommand{\W}{\mbf{W}}
\newcommand{\V}{\mbf{V}}
\newcommand{\Z}{\mbf{Z}}
\newcommand{\Y}{\mbf{Y}}
\newcommand{\beps}{\mbf{\epsilon}}

\newcommand{\Q}{\mbf{Q}}
\newcommand{\K}{\mbf{K}}

\newcommand{\bI}{\mbf{I}}

\renewcommand{\S}{\mbf{S}}
\newcommand{\bE}{\mbf{E}}
\newcommand{\e}{\mbf{e}}
\renewcommand{\u}{\mbf{u}}
\newcommand{\zq}{\z_{\text{query}}}
\newcommand{\onehotWdirection}{\begin{bmatrix}
    \0_{d\times d} &\0_{d\times T}\\
    \0_{T\times d}&\qty(\bI_T - \frac{1}{T}\mathbf{1}_T\mathbf{1}_T^\top)
    \end{bmatrix}}
\newcommand{\randomWdirection}{\begin{bmatrix}
    \0_{d\times d} &\0_{d\times d_e}\\
    \0_{d_e\times d}&\bI_{d_e} 
    \end{bmatrix}}
\newcommand{\onehotVdirection}{\begin{bmatrix}
        \bI_d &\0_{d\times T}
        \end{bmatrix}}
\newcommand{\randomVdirection}{\begin{bmatrix}
        \bI_d &\0_{d\times d_e}
        \end{bmatrix}}
\usepackage{nicematrix}
\usepackage{multirow}
\usepackage{bm}
\usepackage{dsfont}

\usepackage{tikz}
\usepackage{subfigure}
\usepackage{authblk}
\usetikzlibrary{positioning}
\usepackage[mode=buildnew]{standalone}
\NiceMatrixOptions{cell-space-limits = 1pt}

\renewcommand{\S}{\mathcal{S}}

\newcommand\smalldots{\hbox to 1em{.\hss.\hss.}}

\title{Transformers Provably Learn Sparse Token Selection While Fully-Connected Nets Cannot}
\author[1]{Zixuan Wang}
\author[1]{Stanley Wei}
\author[2]{Daniel Hsu}
\author[1]{Jason D. Lee}
\affil[1]{Department of Electrical and Computer Engineering, Princeton University}
\affil[2]{Department of Computer Science, Columbia University}
\date{}

\ifdefined\usebigfont

\usepackage{times}
\usepackage[fontsize=13pt]{scrextend}
\makeatletter
\@ifpackageloaded{geometry}{\AtBeginDocument{\newgeometry{letterpaper,left=1.56in,right=1.56in,top=1.71in,bottom=1.77in}}}{\usepackage[letterpaper,left=1.56in,right=1.56in,top=1.71in,bottom=1.77in]{geometry}}
\AtBeginDocument{\newgeometry{letterpaper,left=1.56in,right=1.56in,top=1.71in,bottom=1.77in}}
\makeatother
\else
\fi

\begin{document}
\maketitle
\allowdisplaybreaks







\begin{abstract}
The transformer architecture has prevailed in various deep learning settings due to its exceptional capabilities to select and compose structural information. Motivated by these capabilities, \citet{sanford2023representational} proposed the \textit{sparse token selection} task, in which transformers excel while fully-connected networks (FCNs) fail in the worst case. Building upon that, we strengthen the FCN lower bound to an average-case setting and establish an algorithmic separation of transformers over FCNs. Specifically, a one-layer transformer trained with gradient descent provably learns the sparse token selection task and, surprisingly, exhibits strong out-of-distribution length generalization.  We provide empirical simulations to justify our theoretical findings.
\end{abstract}

\section{Introduction}
In modern deep learning, transformer networks have established themselves as a fundamental building block, showcasing their versatility across diverse tasks such as language modeling \cite{openai2023gpt4}, computer vision \cite{dosovitskiy2020image}, and reinforcement learning \cite{Jumper2021HighlyAP}. At the core of transformers is the self-attention layer \cite{vaswani2017attention}, a critical component assigning varying attention weights to different segments of the input sequence by discerning relevance between tokens.

The success of transformers is closely tied to their representational capabilities in extracting structural information encoded in token embeddings. Empirical observations reveal that transformers trained with GD-based algorithms exhibit biases towards certain algorithmic solutions in some arithmetic tasks \cite{edelman2022inductive, liu2022transformers, yao2021self, nanda2023progress}. However, few works have presented rigorous mathematical evidence that substantiates their superiority over alternative architectures.

In recent work, \citet{sanford2023representational} introduced a simple task known as the \textit{$q$-sparse averaging} for an input context sequence, where the target output is an average of a $q$-subset of the input tokens, specified in the input.
When the input context length is $T$, the $q$-sparse average was demonstrated to be effectively approximated using a $O(\log T)$-dimensional self-attention layer, whereas any fully-connected network seeking to approximate this task requires a first-layer width of at least $\Omega(T)$ in the worst case. This crucial finding implies that, in theory, the sparse-average task potentially induces an \textit{exponential} separation between transformer models from the fully-connected neural networks (FCNs) with respect to the context length $T$, further reinforcing existing empirical findings that the inductive bias of transformers merits strength in approximating certain arithmetic tasks.

However, these results are limited to the \textit{expressive power} of transformers and do not inherently guarantee that transformers can be \textit{trained} with standard gradient-based methods to achieve such approximation capability as its expressive power would suggest. Thus, a pertinent and natural question beyond theoretical expressiveness arises:
\begin{center}
    \emph{Q: Does the expressivity separation between transformers and FCNs translate to learnability?}
\end{center}
Our results provide an affirmative answer to this question. In our work, we focus on the \textit{sparse token selection task} ($\qsa$). In this task, every token in the input sequence follows the standard Gaussian distribution, and a token subset of size $q$ is drawn uniformly at random. Importantly, we initiate an exploration into the training dynamics of a one-layer transformer with $\sm$ attention using GD for the $\qsa$. We first study the transformer's training dynamics under GD, where the goal is to minimize the expected loss over the input distribution. Notably, we characterize the global convergence of the transformer with the stochastic positional encoding introduced in \citet{shen2023positional}. Additionally, we delve into the provable length generalization capacity of the transformer learned with GD. 
We also establish new lower bounds for the capacity of FCNs to approximate the sparse token selection task within the data distribution. Empirical simulations in practical settings of our architecture verify our theoretical findings and moreover demonstrate the advantages of stochastic positional encoding over a fixed absolute positional encoding in length generalization ability.

\subsection{Our contributions}

\textbf{$\qsa$ is efficiently learnable.} We establish a gradient descent convergence guarantee for a one-layer transformer employing a stochastic positional encoding where only $O(d+q\log T)$ width is necessary (\Cref{sec: stochastic positional encoding}). Under mild conditions on the data distribution, initialization, and hyperparameters, we prove that running GD on a one-layer transformer globally converges when both layers are jointly trained with the same learning rate.  

\textbf{FCN cannot express $\qsa$.}
Complementing our efficient learnability results, we show a separation between FCNs and one-layer transformer networks: all FCNs (regardless of depth or activation function) that can approximate the task must have $\Omega(Td)$ neurons in the first layer, which is exponentially larger than the $O(d+q\log T)$ width used in the transformer.


\textbf{Length generalization on $\qsa.$} We investigate the \textit{length generalization} performance of the trained model with stochastic positional encoding, using out-of-distribution data on longer sequences. Based on the global convergence result, we prove that the length generalization loss also converges to zero. Our experiments demonstrate that when the in-distribution training loss converges to zero, the OOD loss also tends to zero for the stochastic positional encoding, but strictly nonzero using fixed positional encoding.


\subsection{Related works}
\textbf{Expressiveness of transformers.}
The transformer architecture \citep{vaswani2017attention} has been long adopted as the fundamental building block in many recent large language models such as GPT \citep{brown2020language, openai2023gpt4} and Llama \citep{touvron2023llama, touvron2023llama2}. Recent works have begun to study the limitations and strengths of the transformer architecture from a theoretical perspective \citep{yun2019transformers, perez2019turing,yao2021self, bhattamishra2020computational, zhang2206unveiling, liu2022transformers, hahn2020theoretical, bhattamishra2020ability}. One such direction considers its universal approximation power \citep{yun2019transformers, bhattamishra2020ability,  bhattamishra2020computational, dehghani2018universal},  similar to the universal approximation results for fully-connected neural networks. More recent works have focused on analyzing its expressive power on certain statistical tasks \citep{edelman2022inductive, elhage2021mathematical, likhosherstov2021expressive, akyurek2022learning, zhao2023transformers, yao2021self,anil2022exploring, barak2022hidden} as well as in-context learning settings \citep{dong2022survey, huang2023context}. In another line of literature, transformer layers are used to represent gradient descent steps for certain learning tasks \citep{bai2023transformers, garg2022can, von2023transformers, olsson2022context, akyurek2022learning, panigrahi2023trainable, sanford2023representational}. In a recent work, \citet{sanford2023representational}
introduced a sparse average task where recurrent neural networks and fully connected neural networks (FCNs) all have memory complexity scaling polynomially in the input sequence length, while a one-layer transformer has a construction for only $\log T$-width on the task. In all of these works, however, the training process of transformers is not considered; rather, the focus was on the representation power of transformers. Building upon the setting of \citet{sanford2023representational}, our work provides not only a generalized result for the representational power by extending from worst-case to average-case but also investigates the algorithmic aspect of training transformers: we show that the same exponential separation is attainable when using gradient descent on an architecture with stochastic positional encoding. 

\textbf{Training dynamics of transformers.} Several works in the literature have studied the learnability of certain transformer models. \citet{jelassi2022vision} showed a Vision Transformer (ViT) \citep{dosovitskiy2020image} trained by GD with positional-embedding attention matrix can learn spatial structure. \citet{li2023theoretical} analyzed the sample complexity required for achieving good generalization on a similar ViT model. However, their results both hinged on a warm start of initialization near the target pattern, which is a practically infeasible assumption. \citet{tarzanagh2023transformers} established an equivalence between the optimization geometry of self-attention and a hard-margin SVM problem that separates and selects optimal tokens using linear constraints and established global convergence under strong assumptions. \citet{tian2023scan} revealed how the self-attention layer combines input tokens in an SGD-trained transformer, and \citet{tian2023joma} explored the training procedure of multilayer transformers by focusing on the dynamics of MLP layers; however, neither give a provable guarantee for convergence.

Another line of research focuses on the training dynamics of in-context learning. \citet{mahankali2023one} first introduced linear regression as an in-context learning task and showed that a one-layer transformer minimizing the pre-training loss is implementing single-step gradient descent. \citet{zhang2023trained} considered a single-layer linear self-attention layer on this linear regression task and proved global convergence for gradient flow. \citet{huang2023context} first proved GD global convergence of a one-layer transformer with \textit{softmax} attention on the linear regression in-context learning task where in-context tokens are drawn from certain distribution. \citet{chen2024training} extended the single-task linear regression task to a multi-task setup and demonstrated the optimal global convergence of a multi-head attention architecture by applying gradient flow to the population loss with a particular initialization scheme. \citet{nichani2024transformers} theoretically verifies that a simplified two-layer transformer can learn the induction head and generalize it to some in-context latent causal structures.
\citet{li2023theoretical, tian2023scan, zhang2023trained, huang2023context, tarzanagh2023transformers} and our work share a similar reparameterization technique, combining the key and query matrices $\K, \Q$ into one matrix $\W$ to simplify the dynamics of the training process. 

\textbf{Length generalization and positional encoding.} Length generalization is a major challenge for transformers \citep{newman2020eos,dubois2019location,anil2022exploring, ruis2020benchmark,hupkes2020compositionality,kazemnejad2023impact,zhou2023algorithms}. Several methods have been proposed to mitigate this issue, including using linear bias \citep{press2021train}, EOS detection \citep{newman2020eos}, or scratchpads and chain-of-thought \citep{nye2021show, wei2022chain, anil2022exploring, liu2022transformers}.

Recently, \citet{kazemnejad2023impact} empirically investigated length generalization of different positional encodings. They show that many commonly used schemes such as APE \citep{vaswani2017attention}, ALiBi \citep{press2021train}, and rotary \citep{su2021roformer} are ill-suited for length generalization, while T5's relative PE \citep{raffel2020exploring} works for downstream tasks. In addition, they also showed that transformers without positional encoding \citep{tsai2019transformer, haviv2022transformer} outperform all explicit positional encoding schemes. Meanwhile, \citet{shen2023positional} and \citet{ruoss2023randomized} both introduced some randomized positional encoding; \citet{ruoss2023randomized} simulated the positions of longer sequences and randomly sampled a sorted subset to fit the sequence's length; \citet{shen2023positional} used some random positional encoding to represent indicators for tokens positions. \citet{zhou2024transformers} additionally confirmed that randomized positional encoding enhances the length generalization capabilities of transformers in specific effective configurations. All the schemes above significantly improve length generalization, inspiring us in the same spirit to consider a stochastic positional encoding for our theoretical analysis. 

\subsection{Outline of this paper}
The outline of our paper is as follows. In \Cref{sec: settings} we formalize the problem setting, including our $\qsa$ task definition, the positional encoding, and the one-layer transformer architecture. \Cref{sec: main results} contains our main results, consisting of our gradient descent global convergence results, the length generalization theoretical guarantee, and the approximation lower bound for FCNs on expected loss. \Cref{sec: experiments} provides simulations for global convergence
of GD, interprets the learned parameter pattern, and empirically verifies the length generalization advantage from \Cref{sec: main results}.

\section{Settings}\label{sec: settings}
In this section, we present our notations and problem formulations, including the one-layer transformer architecture and positional encoding definitions of our paper.

\textbf{Notations:}
 We use [$T$] to denote the set $\{1,2,...,T\}$. Matrices are represented in upper-case bold letters ($\X, \W$, etc.), and vectors are in lower-case bold letters ($\x,\e$, etc.). For norm, $\|\cdot\|$ denotes $\ell_2$ norm and $\|\cdot\|_F$ denotes the Frobenius norm. For vector $\v$, we use $\v_k$ to denote the $k$-th entry of vector $\v$. For a matrix $\W$, we use $\W[:,i]$ to denote its $i$-th column vector. We use $\mathds{1}\{\cdot\}$ as the indicator function. We use $\tilde{O}(\cdot)$ to hide logarithmic factors.

\subsection{$q$-sparse token selection task $\qsa$}
We simplify the $q$SA framework in \citet{sanford2023representational} to $\qsa$: only one query subset $y$ is fed into the model to compute the sparse average instead of $T$ query subsets.\footnote{Under the population loss we are training on, the sparse token selection problem is equivalent to the original problem. For details, see \Cref{appen sec: equivalence between qsa}.} The objective is to train the parameterized model using a gradient-based algorithm to approximate the $\qsa$ task given a certain data distribution. Similar to $q$SA in \citet{sanford2023representational}, this task is designed to showcase the ability of self-attention units to capture and aggregate dependencies between input tokens, especially the positional information. 
$\qsa$ highlights two salient features of attention matrices observed in practice: they are sparse, and the sparsity pattern depends on the input \citep{likhosherstov2021expressive}.

\begin{definition}\label{def: simplified qsa}
For sparsity $q$, token dimension $d$, and input dimension $dT+q$, consider the input $(\X,y)=(\x_1,\x_2,...,\x_T;y_1, y_2,..., y_q)\in\mathbb{R}^{dT+q}$ where the input tokens $\x_i\in \R^d$ and the query $y\in\binom{[T]}{q}$ is a 
$q$-element subset of $[T]$. Define the $q$-sparse token selection $\qsa(\cdot)$ as
\begin{align*}
    \qsa(\X, y)=\frac{1}{q}\sum_{j=1}^q\x_{y_{j}}
\end{align*}
\end{definition}
\textbf{Data distribution:} We consider samples $(\X,y)\in \R^{dT+q}$ from the following  distribution $\mathcal{D}_{T,q}$ ($T$ is the sequence length, $q$ is the subset size):
The input tokens $\x_i,i=1,2,..., T$ are sampled from standard Gaussian distribution, and the $q$-sparse subset $y$ containing all the averaging indices is uniformly sampled from all $q$-subsets of $[T]$.
$$\X=(\x_1,\x_2,...,\x_T),\x_i\sim \mathcal{N}(0,\bI_d),$$
$$ y\sim \unif\qty(\binom{[T]}{q}),i\in[T]$$

\subsection{Positional encodings}
\label{sec: positional encoding intro}
One of the key features of transformers is the positional encoding (PE). In our $\qsa$ task, the positional encoding turns out to be necessary for the transformer to maintain the positional information in the input sequence. In this paper, 
we consider two different PEs: the one-hot and the near-orthogonal positional encoding. 
\begin{definition}[One-hot PE] The one-hot positional encoding for position $i \in [T]$ is: $$\e_i=\qty(\mathds{1}\{i=1\},\mathds{1}\{i=2\},\cdots,\mathds{1}\{i=T\}) $$The positional encoding matrix is $\bE = [\e_1, \e_2, \cdots,\e_T] = \bI_{T}$, with each column vector $\e_i$ as the PE of the $i$-th token.
\end{definition}

\textbf{Remark.} For clarity, we stress that $\e_i$ is equal to the $i$-th elementary basis vector only when using the one-hot PE. 
\begin{definition}[Near-orthogonal PE]
    A near-orthogonal positional encoding is such that for all positions $i\in[T]$, we have $\e_i\in \{\pm {1}/{\sqrt{d_e}}\}^{d_e}$, and moreover, when we denote the positional encoding matrix $\bE = [\e_1, \e_2, \cdots,\e_T] \in \R^{d_e\times T}$, we have, for some $\delta\in(0,1/2)$, that $d_e=\Theta(q\log T/\delta^2)$ and $\left|\langle \e_i,\e_j\rangle\right|\leq \delta$ for any $i\not =j$, with the $i$-th column vector as the PE of the $i$-th token. 
\end{definition}
The existence of such a set of near-orthogonal positional encoding is guaranteed in \Cref{Lemma: whp random matrices are RIP} (Lemma 12 in \citet{sanford2023representational}) by showing the existence of Rademacher random matrices $\bE$ satisfying the $(q,\delta)$\textit{-restricted isometry property (RIP).} Related background of the near orthogonality and restricted isometry property refers to \Cref{appen sec: backgrounds and prelim}.

\textbf{Encoding for subset $y$:} To utilize the positional information, the query subset $y$ needs some encoding $\e_y$ based on the positional encoding $\bE$. For one-hot encoding, we consider $\e_y = \sum_{i\in y}\e_i$ as the summation of one-hot encoding for all indices in $y$. The encoding scheme for $y$ can separate the positional encoding vector $\e_i$ with index $i\in y$ and those that are not: $\langle \e_y, \e_i\rangle=1, \forall i\in y;\langle \e_y, \e_i\rangle= 0,\forall i\not\in y$.

For the near-orthogonal positional encoding when $y=\{y_i\}_{i=1}^q$, we consider the following encoding based on the RIP of $\bE$: 
\begin{align*}
\e_y &= \bE_y (\bE_y^\top \bE_y)^{-1}\mathbf{1}_q \in \mathbb{R}^{d_e} \\
\bE_y &= \qty[\e_{y_1},\e_{y_2},\cdots,\e_{y_q}] \in \mathbb{R}^{d_e \times T}
\end{align*} 
Here, $\bE_y$ is the concatenation for near-orthogonal PE for indices in $y$. By \Cref{Lemma: RIP matrices has dual certificate}, we know the encoding $\e_y$ can separate the column vectors with index in $y$ and all other columns: $\langle \e_y, \e_i\rangle=1, 
\forall i\in y;\left|\langle \e_y, \e_i\rangle\right|\leq \frac{\delta}{1-2\delta},\forall i\not\in y$. 

For the input format, we consider the positional encodings $\bE$ to be concatenated with the input token matrix $\X$. We then add the query token $\z_{\text{query}} = \qty[\x_{\text{query}}^\top \ \ \e_y^\top]^\top$ as an additional query column in the input matrix. The final input matrix for a transformer will be in the following form, 
\begin{equation}
    [\Z,\z_{\text{query}}] := \begin{bmatrix}
        \x_1 &\x_2 &\cdots &\x_{T-1}&\x_T&\x_{\text{query}}\\
        \e_{1}& \e_{2}& \cdots&\e_{T-1}&\e_{T}&\e_{y}
    \end{bmatrix}.
\end{equation}
which is a  $(d+d_e)$ by $(T+1)$ input matrix.
\subsection{One-layer transformer architecture}\label{subsec: reparam}
A simple one-head, one-layer self-attention transformer has the following architecture:
\begin{align*}
    f (\Z)= \V \Z \sm(\Z^\top \K^\top \Q \Z )
\end{align*}
where $\Q,\K\in \R^{m\times d}, \V\in\R^{d\times (d+d_e)}$ are the query, key, value matrices, respectively, and $m$ is the width\footnote{This is called the \textit{embedding dimension} in \citet{sanford2023representational}.} of the self-attention layer parameters $\Q,\K$. The column-wise softmax operator applied to matrix $\W\in\R^{a\times b}$ outputs $\sm(\W)\in\R^{a\times b}$ with:
$$\sm(\W)_{i,j} = \frac{\exp(\W_{i,j})}{\sum_{i=1}^a \exp(\W_{i,j})}$$
\textbf{Reparameterization:}
Instead of studying dynamics based on the parameters of key, query, and value matrices $\K,\Q,\V$, we introduce the following reparameterization on the architecture:
\begin{enumerate}
    \item In the attention layer, we combine $\Q,\K$ into a single trainable matrix $\W^{\K\Q} \in \mathbb{R}^{(d+d_e)\times(d+d_e)}$.
    \item In the attention layer (with parameter $\W^{\K\Q}$), the $(T+1)$-th query token $\z_\text{query}$ only attends to previous tokens $\Z$. The value matrix $\V$ also only attends to the $\Z$ part.
\end{enumerate}

The first consolidation of the query and key matrices is common in recent theoretical works \cite{tian2023scan, huang2023context, zhang2023trained, jelassi2022vision}. The second modifications for the $\W$ and $\V$ are inspired by \citet{huang2023context, zhang2023trained, chen2024training} to stress the attention between the query token and the input tokens. With those modifications to the architecture, the overall one-layer transformer can be written as follows:
\begin{definition}[Reparameterization]\label{def: reparameterization} Define a reparameterized 1-layer self-attention layer with trainable parameter matrix $\V, \W$ where $\W\in\R^{(d+d_e)\times (d+d_e)},\V\in\R^{d\times (d+d_e)}$:
\begin{equation}
\label{main eqn: practical tf}
\begin{aligned}
    f_{\bm{\theta}}&(\X,y)=\V\Z\sm(\Z^\top \W \z_{\text{query}}),\\ \Z &= [\X^\top \ \bE^\top]^\top, \z_{\text{query}} = \qty[\x_{\text{query}}^\top \ \ \e_y^\top]^\top.
\end{aligned}
\end{equation}

\end{definition}

\subsection{Stochastic positional encoding
}\label{subsec: stochastic positional encoding}
Though a set of fixed near-orthogonal PE may be sufficient for the reparameterized transformer to learn $\qsa$ with a specific sequence length and query set size, in practice we would like the trained model to extrapolate beyond the sequence lengths encountered during training. However, experiments in \Cref{sec: experiments} show that fixing the positional encoding hinders length generalization. This motivates us to use a stochastic positional encoding module proposed in \citet{shen2023positional}, in which a stochastic encoding is used for each token, newly generated for each epoch and during testing. They showed that stochastic positional encoding significantly improves length generalization capability on certain arithmetic tasks. 

For a one-layer transformer $f$ with freshly sampled random encoding $\bE$ at each inference, we can combine the stochasticity into the population loss. In expectation, we consider the stochastic architecture $\E[f]$ when training over the population loss. 
In our setting, we further condition on the event that the positional encoding matrix $\bE$ has the $(q,\delta)$-restricted isometry and orthogonality property (RIP) for some constant $\delta$, take the conditional expectation over the encoding distribution of the transformer model output, and let this be the stochastic architecture output. For simplicity, we denote the conditional expectation of a random variable $\bm{\zeta}$ as:
$$\E_{\bE}^{(R)}\qty[\bm{\zeta}] = \E_{\bE}\qty[\bm{\zeta}\Big|\bE\text{ satisfies $(q,\delta)$-RIP}]
$$
We define a one-layer transformer with stochastic positional encoding as follows.
\begin{definition}[Transformer with stochastic positional encoding] Define a reparameterized 1-layer self-attention layer with stochastic positional encoding as the following model with trainable parameter matrices $\V, \W$ where $\W\in\R^{(d+d_e)\times (d+d_e)},\V\in\R^{d\times (d+d_e)}$:
\begin{align*}
    f^{(s)}_{\bm{\theta}} &(\X,y)=\E_{\bE}^{(R)}\qty[\V\Z\sm(\Z^\top \W \z_{\text{query}})],\\ &\Z = [\X^\top \ \bE^\top]^\top, \z_{\text{query}} = \qty[\x_{\text{query}}^\top \ \ \e_y^\top]^\top.
\end{align*}
\end{definition}

\textbf{Training algorithm:} To train the neural network on the $\qsa$ task, we minimize the population squared loss for parameterized model $f_{\bm{\theta}}(\X)$, similar to \citet{huang2023context} and \citet{zhang2023trained}:
\begin{equation}
    \mathcal{L}_{T,q}(\bm{\theta}) = \frac{1}{2}\E_{\X,y\sim\mathcal{D}_{T,q}}\qty[\|\qsa(\X,y)-f_{\bm{\theta}}(\X,y)\|_2^2]
    \label{eqn: training objective for qsa}
\end{equation}
The expectation is taken over input tokens $\x_i$ and the query subset $y$. For clarity, we write the expectation as $\E_{\X,y}[\cdot]$.

For the transformer with stochastic positional encoding, the training objective becomes:
\begin{equation}
\label{eqn: training objective for stochastic pe}
\mathcal{L}_{T,q}^{(s)}(\bm{\theta}) = \frac{1}{2}\E_{\X, y}\qty[\|\qsa(\X,y)-f^{(s)}_{\bm{\theta}} (\X,y)\|_2^2]. 
\end{equation}
When it is clear from context, we denote $\mathcal{L}_{T,q}$ and $\mathcal{L}_{T,q}^{(s)}$ as $\mathcal{L}$ and $\mathcal{L}^{(s)}$, respectively. For the learning objective in \Cref{eqn: training objective for qsa,eqn: training objective for stochastic pe}, we use gradient descent (GD) to train the neural network.
\begin{equation}
    \bm{\theta}(t+1) = \bm{\theta}(t)-\eta \nabla_{\bm{\theta}} \mathcal{L}(\bm{\theta}(t))
\end{equation}

\section{Main results}\label{sec: main results}
In this section, we present our main results. First, we characterize the convergence of GD on the one-layer architecture introduced in \Cref{subsec: reparam}, using one-hot positional encoding as a warm-up example to provide proof intuition (\Cref{sec: onehot position encoding}). Then, we consider the architecture with stochastic positional encoding in \Cref{subsec: stochastic positional encoding} and present our main theorem with \textit{exponential} separation (\Cref{sec: stochastic positional encoding}). As a corollary, the transformer with stochastic positional encoding trained with GD provably exhibits \textit{length generalization} capability (\Cref{subsec: length generalization}). We conclude this section with the expected loss lower bound on FCNs and rigorously establish the claimed exponential separation in expressive power. 

\subsection{Warm-up: one-hot positional encoding}
\label{sec: onehot position encoding}
In this subsection, we study gradient descent convergence on the $\qsa$ with one-hot positional encoding, building up proof intuition for the stochastic positional encoding case. We consider the joint training regime, where GD updates $\V$ and $\W$ simultaneously with the same learning rate $\eta$.
With one-hot positional encoding, a one-layer transformer with width $O(T)$ is sufficient, which is a factor of $d$ narrower than the lower bound of $\Omega(Td)$ for FCNs. 

We first prove global convergence when $\V$ and $\W$ are jointly trained with gradient descent. The following theorem characterizes the convergence time of GD. For simplicity, in this subsection, we assume $\x_{\text{query}}=\0_d$. Due to space limits, the proof of \Cref{main thm: joint training one hot} is deferred to \Cref {appen sec: proof detail for joint onehot}.

\begin{theorem}[Joint training with one-hot positional encoding]    \label{main thm: joint training one hot}
    For any $2\leq q<T/4,$ $\epsilon\in (0,\frac{dT}{100(T-q)q}),$
    $ \eta\leq \frac{1}{20d^2}, \x_{\text{query}}=\0_d$, 
    if we run gradient descent on the population loss in \Cref{eqn: training objective for qsa} with zero initialization $\W(0)=\textbf{0}_{(d+T)\times (d+T)},\V(0)=\textbf{0}_{d\times (d+T)}$, then after time $t \geq \tilde{O}(\frac{T^2 d}{\eta\epsilon})$, we have
    $\mathcal{L}(\bm{\theta}(t)) \leq \epsilon.$
\end{theorem}
We briefly sketch our proof techniques. We start from deriving the key lemma of the proof, \Cref{main Lemma: symmetry in different q-sparse subsets joint}. It shows the evolution trajectories of $\V(t)$ and $\W(t)$ are always along some direction for all time $t\geq 0$. 

\begin{lemma}[\Cref{Lemma: symmetry in different q-sparse subsets joint}, informal]    \label{main Lemma: symmetry in different q-sparse subsets joint}Along the gradient descent trajectory, for all $t\geq 0$, there exist some time-dependent scalars $C(t),\alpha(t)$ s.t.:
$$\W(t)= C(t) \begin{bmatrix}
    \0_{d\times d} &\0_{d\times T}\\
    \0_{T\times d}&\qty(\bI_T - \frac{1}{T}\mathbf{1}_T\mathbf{1}_T^\top)
\end{bmatrix}, $$
$$\V(t) = \alpha(t) \begin{bmatrix}
    \bI_d &\0_{d\times T}
\end{bmatrix}.$$
\end{lemma}
The nice property along the gradient descent trajectory is attributed to the symmetry and orthogonality of the one-hot positional encoding. Since $y$ follows uniform distribution over all possible $q$-subsets of $[T]$ and $\X$ follows the Gaussian distribution, we can directly calculate the gradient in each of the entry $\W_{ij}(t),\V_{ij}(t)$. It can be shown that along the training trajectory on the population objective, except for the position-position block, all other attention blocks are always $\0$. Moreover, the gradient of the position-position block of $\W$ always aligns with $\qty(\bI_T - \frac{1}{T}\mathbf{1}_T\mathbf{1}_T^\top)$, while the gradient of the token block of $\V$ always aligns with $\bI_d$. Using induction beginning from zero initialization of $\W(0)=\0,\V(t)=\0$, the lemma holds for all $t\geq 0$.

With \Cref{main Lemma: symmetry in different q-sparse subsets joint}, the GD dynamics can be approximately reduced to the following ODE on $\alpha(t)$ and $C(t)$:
\begin{equation}\label{eq: two-variable dynamics in joint case}
    \begin{split}
    \dot{\alpha} &=\eta s_+\qty(1-\frac{\alpha(Tqs_+^2-2qs_++1)}{(T-q)s_+})\\
    \dot{C} &=\frac{\eta \alpha d}{T-1}s_+ (1-qs_+)\qty(1 + \frac{q\alpha}{T-q}(1-Ts_+))
\end{split}
\end{equation}  
where $s_+:=\sm(\Z^\top\W(t)\z_{\text{query}})_i$. By \Cref{main Lemma: symmetry in different q-sparse subsets joint} it is equal to $ \sm(C\bE^\top \e_y)=\frac{1}{q+(T-q)e^{-C}}$ is the attention score of correct positions $i\in y$ for any $y$. It thus remains to analyze this non-linear ODE. 

We first characterize the dynamics for $\alpha(t)$. 
When $s_+$ is fixed with a stationary $\alpha^* = \frac{(T-q)s_+}{Tqs_+^2-2qs_++1}$, note that $\alpha(t)$ can be seen as a linear ODE. However, the evolution of $C(t)$ leads to a monotonic increment of $s_+$, making the trajectory of $\alpha(t)$ follow a non-monotonic pattern. Fortunately, we can inductively prove that first, $\alpha$ rapidly grows to a constant $c_1$ near 1, and subsequently, $\alpha$ can stay above $c_1$ and below a threshold $\qty(1+c_2\frac{(1-qs_+)}{qs_+(Ts_+-1)})\alpha^*$ with $c_2<1$. By the dynamics of $C(t)$, this threshold prevents $C(t)$ from decreasing, and the induction hypothesis enables us to prove that $C(t)$ is monotonically increasing. Note that when $s_+$ converges to $\frac{1}{q}$ as desired, the threshold naturally converges to 1 and the $\V(t)$ converges to the ground-truth $\qty[\bI_d\ \ \0_{d\times T}]$.

With $\alpha(t)$ bounded, we can analyze the movement of $C(t)$ by lower bounding its increment each iteration. Like in \citet{huang2023context}, the increment pattern exhibits a two-stage pattern. In the first stage, $C(t)$ rapidly increases at a linear rate until $s_+$ reaches $\Theta(\frac{1}{q})$; in the second stage, the growth rate of $C(t)$ gradually decreases while being lower bounded by $\text{poly}(T,d)\cdot \epsilon$ when $\mathcal{L}(t)\geq \epsilon$. Therefore, $s_+$ converges to $\frac{1}{q}$ at a rate of $O(1/t)$, leading to the final convergence of the model. 

\subsection{Transformer with stochastic positional encoding}
\label{sec: stochastic positional encoding}
In this subsection, we study the gradient descent convergence on the $\qsa$ problem with stochastic positional encoding. We also consider the joint training case, updating $\V(t)$ and $\W(t)$ simultaneously with learning rate $\eta$. 

Though stochastic PE was introduced as an alternative solution due to the failure of fixed absolute PE in length generalization \cite{shen2023positional}, it also turns out to help the analysis of the GD training trajectory. The symmetry of the randomized architecture makes it possible to use our intuition from the one-hot encoding case while reducing the transformer width to $\Theta(d+q\log T)$.
We now present our main theorem with stochastic PE.
\begin{theorem}[Joint training with stochastic positional encoding]    \label{main thm: joint training stochastic pe}
    Suppose $q=\Theta(1)$, $\delta < 1/10 $, $d_e=\Theta(q\log T/\delta^2)$, $2\leq q<T/4$, $\epsilon\in (0,\frac{dT}{100(T-q)q}), \eta\leq \frac{d_e^2}{40d^2T}$.
    If we run GD on the population loss in \Cref{eqn: training objective for stochastic pe} with zero initialization $\W(0)=\textbf{0}_{(d+d_e)\times (d+d_e)},\V(0)=\textbf{0}_{d\times (d+d_e)}$, then after time $t \geq \tilde{O}(\frac{T^{\frac{2-2\delta}{1-3\delta}}}{\eta}+\frac{T^2 d}{\eta\epsilon})$, we have
    $\mathcal{L}^{(s)}(\bm{\theta}(t)) \leq \epsilon.$
\end{theorem}
The full proof appears in \Cref{appen sec: proof detail for joint stochastic pe}. The proof idea for this theorem is similar to that of \Cref{main thm: joint training one hot}. We first use induction to prove a key lemma similar to \Cref{main Lemma: symmetry in different q-sparse subsets joint}, though the convergence direction is different.
\begin{lemma}[Consequence of \Cref{Lemma: induction hypothesis for joint training stochastic PE}, informal] 
\label{main Lemma: induction hypothesis for joint training stochastic PE}
Along the gradient descent trajectory, for all $t\geq 0$, there exist some time-dependent scalars $C(t),\alpha(t)$ s.t. 
$$\W(t)= C(t) \begin{bmatrix}
    \0_{d\times d} &\0_{d\times d_e}\\
    \0_{d_e\times d}&\bI_{d_e}
\end{bmatrix},$$$$ \V(t) = \alpha(t) \begin{bmatrix}
    \bI_d &\0_{d\times d_e}
\end{bmatrix}.$$
\end{lemma}

With \Cref{main Lemma: induction hypothesis for joint training stochastic PE}, we can use the $(q,\delta)$-RIP of $\bE$ to estimate the expectation in the dynamics, and moreover derive an approximate ODE similar to \Cref{eq: two-variable dynamics in joint case} that roughly tracks the trajectory. Although we cannot calculate the exact ODE, the upper and lower bounds enable us to formulate an inductive argument, similar to the proof of \Cref{main thm: joint training one hot}. Eventually, the analysis of our controlled dynamics leads to global convergence.

Note that the model trained by GD with stochastic PE converges exactly to a solution equivalent to the constructed one-layer transformer in \citet{sanford2023representational} with the same width $m\sim \Theta(d+q\log T)$. This implies GD learns $\qsa$ using a transformer of near-optimal width, where the theoretical lower bound $\Theta(d+q)$ is only smaller by a logarithmic factor.\footnote{Remarkably, this resembles the intuition from \citet{zhou2023algorithms}: since we learn the smallest transformer in theory, it can length-generalize. This corresponds to our length generalization guarantee in \Cref{subsec: length generalization}.}

Moreover, the learned transformer can perfectly solve $\qsa$ tasks given any fixed positional $\bE$ satisfying $(q,\delta)$-RIP, and we can replace the stochastic architecture with any valid $\bE$ as a fixed architecture. In practice, this means that once the transformer is successfully trained, we no longer need to sample new positional encodings for evaluation. For more discussion, see \Cref{appen sec: from stochastic to fix}.

\begin{figure}
    \centering
    \hspace{-0.4cm}\includegraphics[width=0.8\textwidth]{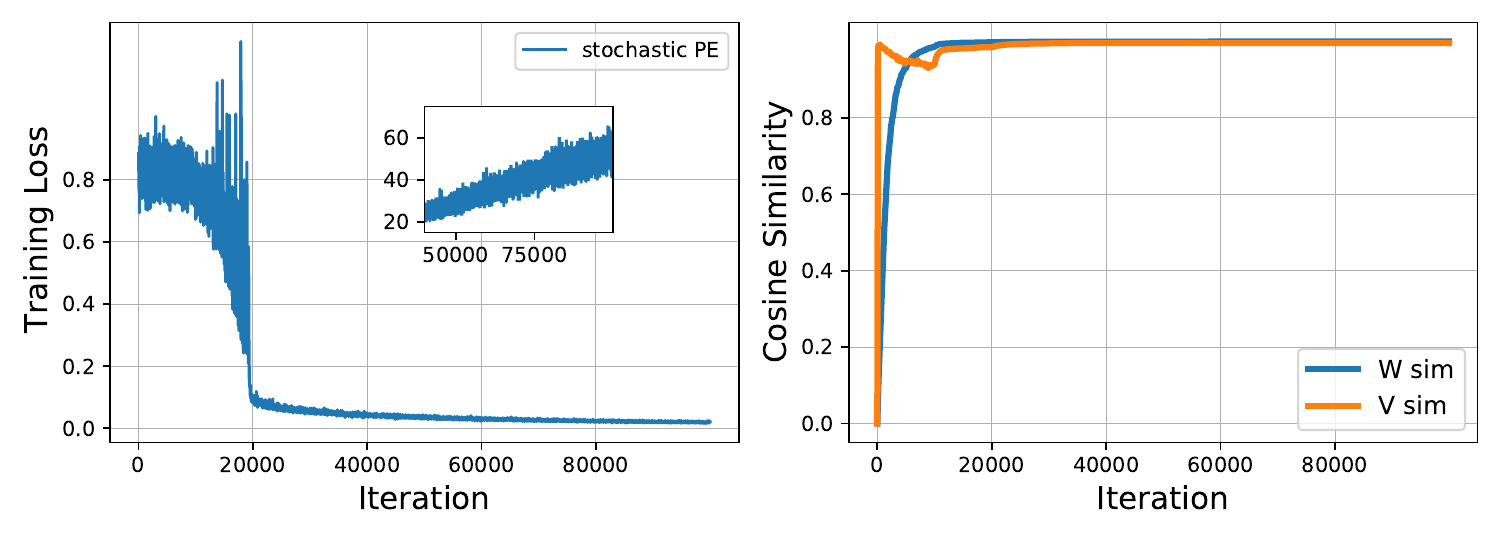}    
    \caption{The above figures describe the training trajectory of the one-layer transformer model with attention layer $\W$ attending to the full matrix $[\Z,\z_{\text{query}}]$. \textbf{Left (Global convergence of the transformer)}: we plot the training loss for the one-layer transformer with stochastic PE, complementing it with the inverse loss plot. The training loss converges to the global minimum of 0, and the inverse loss increases linearly, indicating the $O(1/t)$ convergence rate. \textbf{Right (Cosine similarity with the ground-truth)}: we plot the evolution of the cosine similarity between $\W$ and $\W^\star$ and between $\V$ and $\V^\star$ throughout training, where $\W^\star$ and $\V^\star$ are the ground-truth matrices. The two cosine similarity curves gradually converge to 1, indicating the one-layer transformer eventually converges to the desired direction.}
    \label{fig:similarity plot}
\end{figure}

\subsection{Expressive power separation}
We now complement our positive results with a width lower bound for FCNs on the population loss: any FCN without the first layer width $\Omega(Td)$ cannot approximate the $\qsa$ with respect to the expected square loss. In comparison to the results of \citet{sanford2023representational}, our results lower bound the approximation error on average over the data distribution instead of considering a single worst-case data point\footnote{The lower bound of average-case loss with respect to some data distribution implies the existence of a lower bound for worst-case loss in \cite{sanford2023representational}.}.

In this section, we consider FCN of the following form:
$$f(\x)=\W_L\sigma(\W_{L-1}\cdots\sigma(\W_1\x))$$
In fact, the results can be generalized to any network of the form $f(\x)=g(\W_1 \x)$ for some first layer weight matrix $\W_1\in \R^{m\times (Td+q)}$ and arbitrary function $g:\R^m\to \R^{Td}$. For the FCN, the input format should be flattened as a vector  $[\x_{\text{vec}},y] = \qty(\x_1^\top,\x_2^\top, ..., \x^\top_T, y_1,y_2,...,y_q)^\top\in \R^{dT+q}$. Here, $\x_i \in \R^d$ for $i\in[T]$.

We now consider the expected loss for the $\qsa$ with input tokens $\x_i\sim \mathcal{N}(0,\bI_d)$ and the query subset $y$ sampled uniformly from the set of $q$-subsets of $[T]$. Our result shows any FCN that can approximate $\qsa$ requires its first hidden layer to have width at least $\Omega(Td)$. 
\begin{theorem}\label{main thm: fcn expected case lower bound}
    Let $\mathcal{M}:\R^{dT+q}\to\R^d$ be any FCN employing any activation having first layer width at most $Td-1$. Then: 
    \begin{equation*}
        \E_{\X,y}\qty[\|\mathcal{M}([\x_{\text{vec}},y])-\qsa(\X,y)\|_2^2]\geq \frac{T-q}{Tq(T-1)}
    \end{equation*}
\end{theorem}
The proof (see \Cref{appen sec: the average case lower bound for fcns}) relies on the observation that when the first hidden layer has fewer than $Td$ neurons, the linear transformation $\W_1$ has a non-trivial kernel. This implies there exists some 1-dimensional subspace where the FCN outputs a constant. However, the $\qsa(\cdot)$ function depends on the total input vector $\x_{\text{vec}}\in \R^{Td}$ in all directions. This leads to a non-trivial approximation error.

Combining these results and \Cref{sec: stochastic positional encoding}, we rigorously establish the expressive power separation between transformers and FCNs on this $\qsa$ task in the intrinsic complexity of width: a simple transformer with width $\Theta(d+q\log T)$ can efficiently \textit{learn} $\qsa$, while any FCNs must have  $\Omega(Td)$ width to even \textit{approximate} this task. 

\subsection{Length generalization} \label{subsec: length generalization}
Here, we present our length generalization result. One key motivation for stochastic PE is that it can improve the length generalization of the trained model, which is beyond the range of the training distribution. This inspires us to investigate theoretical guarantees on out-of-distribution data of longer sequence length on the $\qsa$ task.

\begin{figure*}[t]
    \centering
    \includegraphics[width=\textwidth]{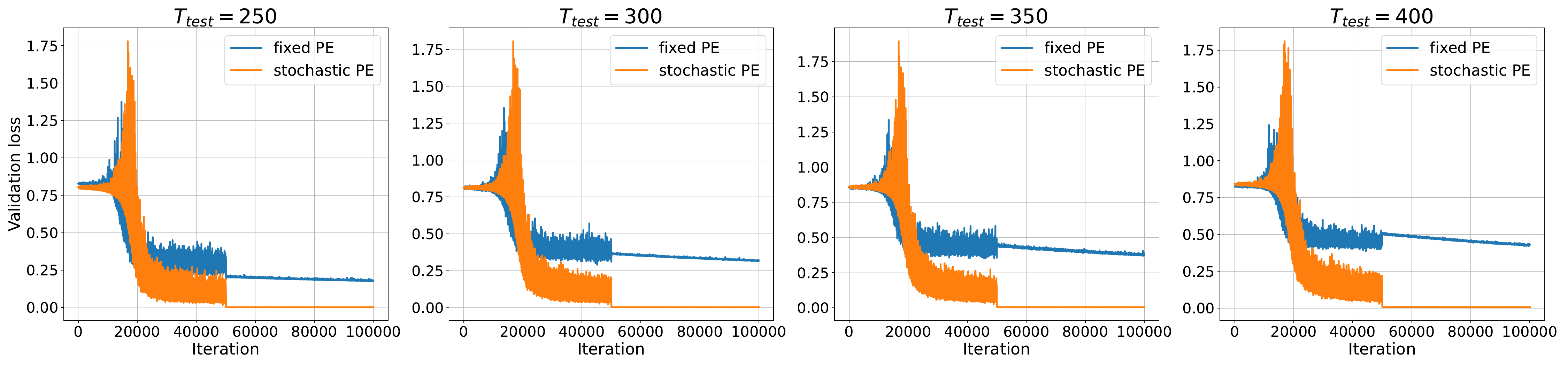}
    \caption{\textbf{Length generalization superiority of stochastic PE:} We plot the out-of-distribution error throughout training on each of the four \textit{length generalization} tasks ($T_{\text{test}}=250, 300, 350, 400$) when $T=200$ and $q=3$. Observe that the one-layer transformer with stochastic positional encoding has\textbf{ a clear advantage} over the fixed positional encoding architecture in all four tasks: stochastic architecture converges after 10k steps, while the length generalization error of the fixed architecture does not go below some constant.
    }
    \label{fig:length generalization}
\end{figure*}

Suppose our training objective is based on the distribution $\mathcal{D}_{T_1,q}$ with sequence length $T_1$ and subset size $q$. Then, the OOD loss with sequence length $T_2 \geq T_1$ becomes:
\begin{equation}\label{main eqn: ood objective for qsa}
    \mathcal{L}^{(s)}_{T_2,q}(\bm{\theta}) = \frac{1}{2}\E_{\X,y\sim\mathcal{D}_{T_2,q}}\qty[\|\qsa(\X,y)-f^{(s)}_{\bm{\theta}}(\X,y)\|_2^2]
\end{equation}
We can then derive the following corollary of \Cref{main thm: joint training stochastic pe}, proving the length generalization error goes to 0 as the training error tends to 0, as long as the embedding dimension is large enough for $\bE$ to satisfy RIP. 
\begin{corollary}[Informal\footnote{It is a direct corollary of \Cref{appen coro: ood guarantee}, which can also generalized to OOD data with unseen query subset size $q'\not= q$.}]\label{sec: ood theory}
Suppose $q=\Theta(1),d_e=\Theta(q\log T_{\max}/\delta^2)$. If we apply gradient descent with zero initialization with $T_1<T_{\max}$ to train the model under same condition in \Cref{main thm: joint training stochastic pe}, then when the training loss $\mathcal{L}_{T_1,q}^{(s)}(\bm{\theta})\leq \epsilon$, it holds that for any $T_2\in(T_1,T_{\max}]$:
$$\mathcal{L}^{(s)}_{T_2,q}(\bm{\theta})\leq O\qty(\frac{T_2^2\epsilon}{T_1^2})$$
\end{corollary}
The proof, given in \Cref{appen sec: ood guarantee}, is based on the exact global minimizer GD finds after training. With stochastic positional encoding, the GD solution can naturally generalize on longer input sequences, leading to this length generalization corollary.

\section{Experiments}\label{sec: experiments}

In this section, we describe our experimental setup on synthetic data, which numerically justifies our theoretical guarantees for convergence. In addition, we devise several length generalization tasks for our model, in which we are able to highlight the benefits of our stochastic architecture. 

\textbf{Synthetic experimental setup.} 
In our experiment, we use the one-layer transformer architecture defined in \Cref{subsec: reparam}.
Our synthetic data follows the distribution $\mathcal{D}_{T,q}$ in our theoretical analysis: $\X$ is sampled from the standard Gaussian, and $y$ is uniformly sampled from all possible $q$-subsets of $[T]$. In particular, we choose $T=200$ for our sequence length, $q=3$, $d=5$, and $d_e = 170.$ In addition, to simulate the population loss training, we train using online stochastic gradient descent (SGD) by resampling a fresh batch of $n=256$ datapoints $(\X, y)$ at each iteration to use for our gradient estimate. When training with a fixed positional encoding, we sample and fix the encoding samples at the start of training; when we simulate the training of the stochastic architecture, we sample a single positional encoding $\bE$ at each iteration. 

\textbf{Training loss convergence.} Our experiments show the convergence of the in-distribution training/population loss (estimated using the fresh batch of data at each iteration) for our stochastic architecture (\Cref{fig:similarity plot}). The convergence process can be verified by observing the cosine similarity of the weight matrices to their ground truth directions in \Cref{fig:similarity plot}. We also experiment on a smaller transformer with $d=20, d_e=20$ and plot the heat map of the parameters $\V$ and $\W$. As we can see in \Cref{fig:heat map}, the heatmaps of $\W$ and $\V$ coincide with the ground-truth
\begin{align*}
    \W^* = \begin{bmatrix}
    \0&\0\\
    \0&\alpha\bI_{d_e}
\end{bmatrix}\quad
\V^*= \begin{bmatrix}
    \bI_d&\0
\end{bmatrix}.
\end{align*}
This justifies our reparameterization, confirming that the simplified transformer still captures the essential positional information of the $\qsa$ task.

\begin{figure}[t]
    \centering
    \includegraphics[width=0.36\textwidth]{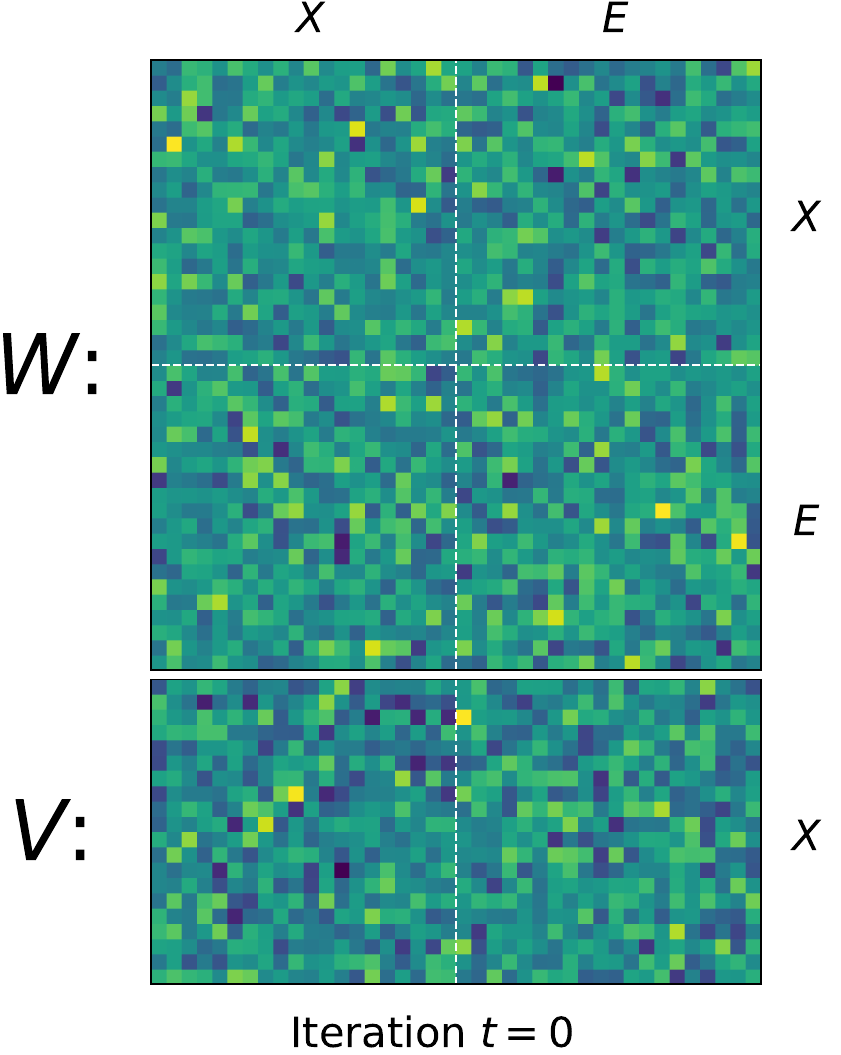}
    \includegraphics[width=0.35625\textwidth]{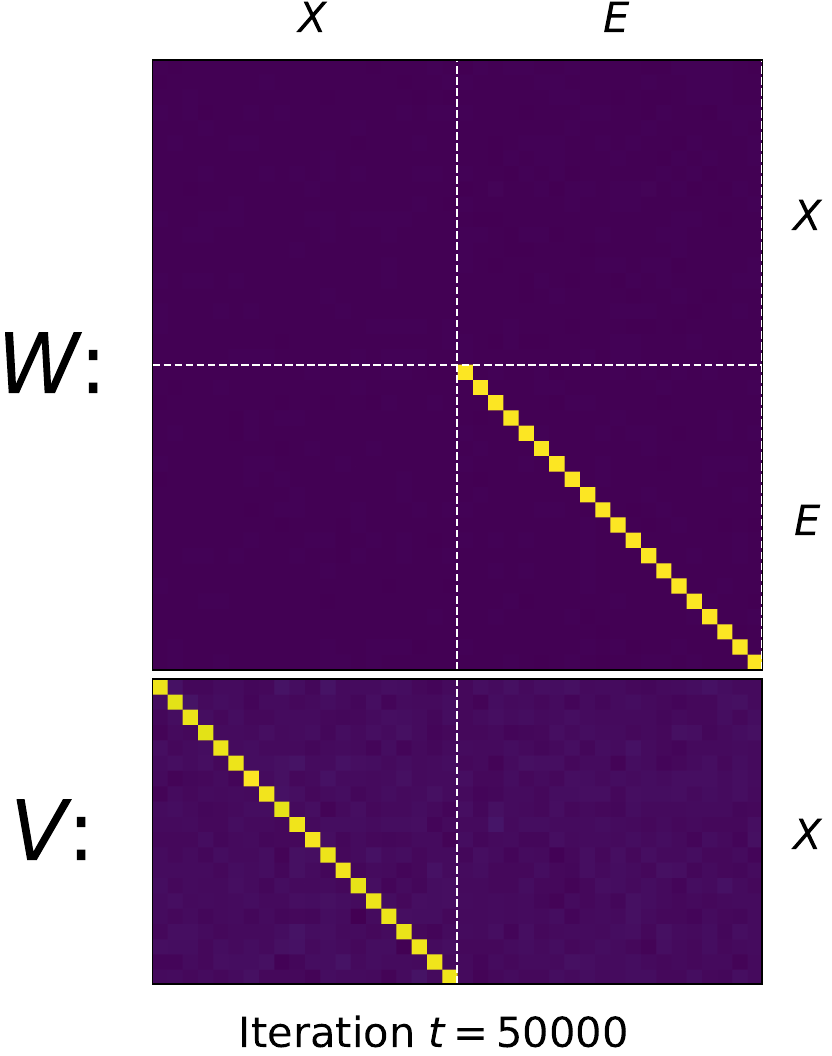}
    \caption{\textbf{Interpretable training:} For the full model \cref{main eqn: practical tf}, we present a heat map of the self-attention layer $\W$ and the value matrix $\V$ at initialization and after convergence. We initialize $\W,\V$ randomly at $t=0$. After training, observe that only the sub-block of $\W$ that attends to the positional encodings $\bE$ converges to the identity direction, while all other entries converge to 0; in $\V$, only the sub-block that attends to the input tokens $\X$ converges to identity direction with all other entries converging to 0.}
    \label{fig:heat map}
\end{figure}

\textbf{Length generalization.} We propose the following out-of-distribution length generalization tasks for our models. For each task, we fix before training a validation set of $n_{\text{test}}=128$ out-of-distribution datapoints $(\X, y)$ from the corresponding task distribution.\footnote{For detailed settings, see \Cref{sec: experiment details appn}.} \Cref{fig:length generalization} confirms that the fixed architecture has drastically worse length generalization compared to the stochastic architecture: transformers with fixed PE cannot even converge to zero OOD validation loss while the stochastic model extrapolates to data of unseen length. Our empirical findings reveal the benefits of using a stochastic positional encoding architecture over a fixed positional encoding architecture, thereby justifying our theoretical setup and the results in \Cref{sec: ood theory}.


\section{Conclusion}
In this paper, we put forward a comprehensive theoretical analysis of gradient descent on the \textit{sparse token selection task} $\qsa$. We characterize the joint training dynamics of a one-layer transformer with stochastic positional encoding and demonstrate the width separation between transformers and fully-connected networks on $\qsa$ task. The stochasticity of these positional encodings also provably leads to length generalization capabilities beyond what is seen in the input data, highlighting a benefit of our architecture.

There are still many open questions. For instance, can we move beyond population loss and show a sample complexity guarantee? Can we extend the benefits of randomized positional encodings to other tasks? Can we analyze the length generalization of transformers in other practical settings?

\bibliography{main}

\newpage
\appendix
\onecolumn
\section{Backgrounds and Preliminaries}
\label{appen sec: backgrounds and prelim}
\subsection{Restricted isometry and orthogonality property}
We replicate the definition of \textit{restricted isometry and orthogonality} from \citet{sanford2023representational} in this section. For $\v\in \mathbb{R}^T$, denote $\text{supp}(\v) = \{i \in [T] : \v_i\neq 0\}$.
\begin{definition}
    We say a matrix $\bE\in \mathbb{R}^{d_e\times T}$ satisfies the $(q,\delta)$-restricted isometry and orthogonality property if \begin{align*}
        \|\bE \v\|_2^2 \in [(1-\delta)\|\v\|_2^2, (1+\delta)\|\v\|_2^2] \quad \text{and} \quad |\langle \bE \v, \bE \v^\prime \rangle| \leq \delta \|\v\|_2 \|\v^\prime\|
    \end{align*}
    for all vectors $\v,\v^\prime\in \mathbb{R}^T$ with $|\text{supp}(\v)| \leq q$, $|\text{supp}(\v^\prime)|\leq 2q$, and $|\text{supp}(\v) \cap \text{supp}(\v^\prime)| = 0$.
\end{definition}
Here we restate the two lemmas in \citet{sanford2023representational} for the construction of the transformer approximating $\qsa$. The first lemma is that the existence of a Rademacher random matrix can satisfy the restricted isometry and orthogonality properties. For simplicity, we just call it RIP. 
\begin{lemma}[Lemma 12 from \citet{sanford2023representational}]
    \label{Lemma: whp random matrices are RIP}
    There is an absolute constant $C>0$ such that the following holds. Fix $\delta \in (0,1/2)$ and $q\in\mathbb{N}$. Let $\bE$ denote an $d_e\times T$ matrix of independent Rademacher random variables scaled by $\frac{1}{\sqrt{d_e}}$. If $d_e\geq C(q\log T)/\delta^2$, then with positive probability, $\bE$ satisfies the $(q,\delta)$-restricted isometry and orthogonality property.
\end{lemma}

\begin{lemma}[Lemma 13 from \citet{sanford2023representational} and consequence of Lemma 2.1 from \citet{candes2005decoding}]
    \label{Lemma: RIP matrices has dual certificate}
    Fix $\delta\in(0,1/2)$ and $q\in\mathbb{N}$. Let $\bE  =[\e_1,\dots, \e_T] \in \mathbb{R}^{d_e\times T}$ satisfy the $(q,\delta)$-restricted isometry orthogonality property. For every vector $\v\in\{0,1\}^T$ with $\supp(\v)\leq q$, suppose the indices $i_1,i_2,\dots, i_{\supp(\v)} \in [T]$ in $\v$ have entry 1. Then, the vector $\e_{\v} = \bE_{\v}(\bE_{\v}^\top \bE_{\v})^{-1} \1_{\supp(\v)} \in \mathbb{R}^{d_e}$ satisfies the following: \begin{align*}
        \|\e_{\v}\|_2 &\leq \sqrt{q} / (1-2\delta) \\
        \langle \e_i, \e_{\v}\rangle &= 1 \quad \text{if } v_i=1 \\
        |\langle \e_i, \e_{\v} \rangle| &\leq \delta/(1-2\delta) \quad \text{if } v_i=0
    \end{align*}
    where we define $\bE_{\v} = [\e_{i_1},\dots, \e_{i_{\supp(\v)}}] \in \mathbb{R}^{d_e \times \supp(v)}$.
\end{lemma}

\textbf{Remark.} Compared with \citet{sanford2023representational}, we add the expression of $\e_{\v}=\bE_{\v}(\bE_{\v}^\top \bE_{\v})^{-1}\1_{\supp(\v)}$, which we use as the encoding for the query subset $y$ in our one-layer transformer. It is used in Lemma 2.1 of \citet{candes2005decoding} as the dual certificate for the specific dual problem of a primal $\ell_1$-minimization problem. We borrow the exact form of the dual certificate here as the encoding, replacing the blackbox MLP layer $\phi(\X)$ used in \citet{sanford2023representational} (which can be poly$(T,d)$-wide, far larger than the width of the transformer). We posit that this positional encoding for $y$ can also be expressed by a multi-layer transformer with width $\Theta(d+q\log T)$, depth $\Theta(\log\log\log T + \log\log \frac{q}{\epsilon})$. It can be constructed using the transformer construction on matrix inverse in \citet{giannou2023looped}.

\subsection{The equivalence between simplified $\qsa$ and $q$SA in \citet{sanford2023representational}}\label{appen sec: equivalence between qsa}

We restate our $\qsa$ task and compare it with the original $q$SA task used in \citet{sanford2023representational}.
\begin{definition}\label{appen def: sparse selection}
For sparsity $q$, token dimension $d$, and input dimension $dT+q$, consider the input $(\X,y)=(\x_1,\x_2,...,\x_T;y_1, y_2,..., y_q)\in\mathbb{R}^{dT+q}$ where $\x_i\in \R^d$ and $y\in\binom{[T]}{q}$ is a 
$q$-element subset of $[T]$. Define the $q$-sparse token selection $\qsa(\cdot)$ as
\begin{align*}
    \qsa(\X, y)=\frac{1}{q}\sum_{j=1}^q\x_{y_{j}}
\end{align*}
\end{definition}
This arithmetic task is based on the $q$SA task, while reduce the number of query subsets $y_i$ (which is $T$ in \citet{sanford2023representational}) to 1. As a regression task, we believe it is a natural simplification only to consider one query in one data point. 

Moreover, when considering the population $\ell_2$ loss \Cref{eqn: training objective for qsa} we are using (instead of the $\ell_\infty$ loss used in \citet{sanford2023representational}), the two tasks are equivalent. For $\qsa$, the population loss is:
$$\mathcal{L}_{T,q}(\bm{\theta}) = \frac{1}{2}\E_{\X,y\sim\mathcal{D}_{T,q}}\qty[\|\qsa(\X,y)-f_{\bm{\theta}}(\X,y)\|_2^2]=\frac{1}{2}\E_{\X,y\sim\mathcal{D}_{T,q}}\qty[\left\|\frac{1}{q}\sum_{j=1}^q\x_{y_j}-f_{\bm{\theta}}(\X,y)\right\|_2^2]$$
For the original $q\textbf{SA}$, suppose all $y_i$ follow the uniform distribution, then we have the population loss:
\begin{align*}
\mathcal{L}_{T,q}(\bm{\theta}) &= \frac{1}{2}\E_{\X,y_i\sim\mathcal{D}_{T,q}}\qty[\sum_{i=1}^T\|q\text{SA}(\X,y_i)-f_{\bm{\theta}}(\X,y_i)\|_2^2]\\&=\frac{T}{2}\E_{\X,y\sim\mathcal{D}_{T,q}}\qty[\left\|\frac{1}{q}\sum_{j=1}^q\x_{y_{i,j}}-f_{\bm{\theta}}(\X,y_i)\right\|_2^2]    
\end{align*}

They only have a $T$-factor difference. It also indicates that the training trajectory is also different in scale with an appropriate learning rate. Therefore, we believe it's more reasonable to consider the optimization problem on our simplified $\qsa$ problem.

\section{Limitation and discussion}
Beyond the scope of this work, there still exist some limitations and future open problems.

First of all, our analysis is based on population loss rather than empirical risk minimization (ERM) on a finite dataset using (stochastic) gradient descent. The population loss is equivalent to the empirical loss induced by the limit of infinite training samples, which largely simplifies the dynamics. In our case, the stochasticity in the positional encodings also takes advantage of this population loss. It enables us to focus on analyzing the stochastic architecture $\E_{E}[f]$, where $f$ is the transformer with $E$ as its positional encoding. This population objective is used in almost all recent works analyzing the full-training dynamics on linear/softmax transformer architectures \citep{huang2023context, zhang2023trained, nichani2024transformers, chen2024training, tian2023scan, kim2024transformers}. It would be an interesting open problem to analyze the SGD dynamics and sample complexity on any of the existing tasks in the literature.

Another limitation of our submitted version of the paper relates to the simplified architecture of our setting. Specifically, we consolidate the query and key matrices $\Q, \K$ into $\W$. Though it may lead to different landscape properties on the training trajectory, it does not inherently change the expressive power of the transformer. In particular, most of the recent works have also adopted (1) to simplify the dynamics \citep{huang2023context, zhang2023trained, nichani2024transformers, tian2023scan, kim2024transformers}. 

Finally, we would like to remark that we use stochastic positional encoding in our one-layer transformer instead of a fixed set of near-orthogonal positional encoding; the additional stochasticity enables us to analyze the GD dynamics theoretically. Moreover, it also helps a lot in practice with regard to out-of-distribution length generalization. Empirically, both choices work for in-distribution convergence, but only stochastic positional encoding can also achieve out-of-distribution length generalization. The length generalization superiority of randomized PE is also justified in several recent works with extensive experiments \citep{shen2023positional, ruoss2023randomized, zhou2024transformers}. Nevertheless, analyzing the dynamics with a set of fixed positional encodings on the in-distribution loss can be an interesting open problem. 


\section{Approximation results on $\qsa$}
\label{appen sec: B expressivity}
As \citet{sanford2023representational} proved the worst-case width lower bound for FCNs to approximate the original $q$SA problem, we can also prove similar results for the new formulation of $\qsa$. In this paper, we go beyond the worst-case analysis and prove the width lower bound in expectation under certain data distribution. Also for completeness, we present the approximation results with a one-layer transformer, showing the capability of the transformer to represent $\qsa$.
\subsection{The average-case lower bounds for FCNs on $\qsa$}\label{appen sec: the average case lower bound for fcns}
We first present our negative results FCNs for $\qsa$ in the expected loss. For an FCN, we vectorize the input to $\qty[\x_{\text{vec}},y] = \qty(\x_1^\top,\x_2^\top, ..., \x^\top_T, y_1,y_2,...,y_q)^\top\in \R^{dT+q}$. Here $\x_i \in \R^d$ for $i=1,2,...,T$. With this theorem, we rigorously establish the \textit{exponential} expressivity separation between FCNs and one-layer transformers in the complexity metric of width/embedding dimension.

\begin{theorem}
    Let $\mathcal{M}:\R^{dT+q}\to\R^d$ be any FCN employing any activation having first layer width at most $Td-1$, then 
    \begin{equation*}
        \E_{\x_i\sim \mathcal{N}(0,\bI_{d}),y\sim\unif\qty(\binom{[T]}{q})}\qty[\|\mathcal{M}(\x_{\text{vec}},y)-\qsa(\X,y)\|_2^2]\geq \frac{T-q}{Tq(T-1)}
    \end{equation*}
\end{theorem}
\begin{proof}
    Consider the first layer weight matrix $\W\in \R^{k\times (dT+q)}$. Since $k\leq Td-1$, $\text{rank}(\W)\leq Td-1$. Therefore, we have the submatrix of the first $Td$ rows $W_s = \W[:,Td]\in\R^{k\times Td}$ has its nullspace $\ker(\W_s)$. Denote the vector $\v\in\ker(\W_s)$ s.t. $\|\v\|=1$ and w.l.o.g. $\|\v\|_\infty=v_1$, which is the first entry of the vector. We then find the set of orthonormal basis \{$\v_1,...,\v_{Td}$\} in $\R^{Td}$ with $\v_1:=\v$ as the first basis vector.

    Now consider the decomposition of $\x_{\text{vec}}$ in the basis \{$\v_1,...,\v_{Td}$\}. Since it's sampled from normal distribution, $\x_{\text{vec}}$ can be rewritten to the reparametrized vector:
    $$\x_{\text{vec}} = \epsilon_1\v_1+ \epsilon_2\v_2+\cdots +\epsilon_{Td}\v_{Td}, \ \epsilon_i\sim \mathcal{N}(0,1),\ i=1,2,...,Td$$
    Denote $\beps:=(\epsilon_1,...,\epsilon_{Td})$. Then we can calculate the expected MSE loss given some sparse set $y$:
    \begin{align*}
        &\E_{\x_i\sim \mathcal{N}(0,\bI_{d})}\qty[\|\mathcal{M}(\x_{\text{vec}},y)-\qsa(\X,y)\|_2^2]\\ =& \int_{\x_{\text{vec}}}\|\mathcal{M}(\x_{\text{vec}},y)-\qsa(\x_{\text{vec}},y)\|^2 \mathrm{d}p(\x)\\
        =&\int_{\beps}\left\|\mathcal{M}\qty(\sum_{i=1}^{Td}\epsilon_i \v_i, y_1,...,y_q)-\qsa(\x_{\text{vec}},y)\right\|^2 \mathrm{d}p(\beps)
    \end{align*}
    Note that $\qsa(\x_{\text{vec}},y)=\qsa\qty(\sum_{i=1}^{Td}\epsilon_i \v_i, y_1,...,y_q) = \sum_{i=1}^{Td}\epsilon_i\qsa\qty( \v_i, y_1,...,y_q)$. And since $\v_1=\v$ is in the kernel of the first layer, 
    $$\mathcal{M}\qty(\sum_{i=1}^{Td}\epsilon_i \v_i, y_1,...,y_q) = \mathcal{M}\qty(\sum_{i=2}^{Td}\epsilon_i \v_i, y_1,...,y_q)$$
    which is constant with respect to $\epsilon_1$.
    Therefore we have
    \begin{align*}
        &\int_{\R}\left\|\mathcal{M}\qty(\sum_{i=1}^{Td}\epsilon_i \v_i, y_1,...,y_q)-\qsa(\x_{\text{vec}},y)\right\|^2 p(\epsilon_1)\mathrm{d}\epsilon_1\\
        = &\int_{\R}\left\|\mathcal{M}\qty(\sum_{i=2}^{Td}\epsilon_i \v_i, y_1,...,y_q)-\sum_{i=1}^{Td}\epsilon_i\qsa\qty( \v_i, y_1,...,y_q))\right\|^2 p(\epsilon_1)\mathrm{d}\epsilon_1\\
        = & \int_{\R}\left\|\mathcal{M}\qty(\sum_{i=2}^{Td}\epsilon_i \v_i, y_1,...,y_q)-\epsilon_1\qsa\qty( \v_1, y_1,...,y_q))\right.\\
        &\left.-\sum_{i=2}^{Td}\epsilon_i\qsa\qty( \v_i, y_1,...,y_q))\right\|^2 p(\epsilon_1)\mathrm{d}\epsilon_1\\
        \geq &\int_{\R}\left\|\qsa\qty( \v_1, y_1,...,y_q))\right\|^2 \epsilon_1^2p(\epsilon_1)\mathrm{d}\epsilon_1 = \left\|\qsa\qty( \v_1, y_1,...,y_q))\right\|^2
    \end{align*}
    The last inequality is because after expanding the squared norm we have (1) $\epsilon_1^0$-terms are always non-negative and (2) the integral for $\epsilon_1$-terms is 0.

    Now we have the lower bound for the expected loss:
    \begin{align*}
        &\E_{\x_i\sim \mathcal{N}(0,\bI_{d}),y\sim\unif\qty(\binom{[T]}{q})}\qty[\|\mathcal{M}(\x_{\text{vec}},y)-\qsa(\X,y)\|_2^2]\\ =&\E_{y\sim\unif\qty(\binom{[T]}{q})}\int_{\beps}\left\|\mathcal{M}\qty(\sum_{i=1}^{Td}\epsilon_i \v_i, y_1,...,y_q)-\qsa(\x_{\text{vec}},y)\right\|^2 \mathrm{d}p(\beps)\\
        \geq &\E_{y\sim\unif\qty(\binom{[T]}{q})}\int_{\beps_{2:Td}} \left\|\qsa\qty( \v_1, y_1,...,y_q))\right\|^2\mathrm{d}p(\beps_{2:Td})\\
        = &\E_{y\sim\unif\qty(\binom{[T]}{q})}\left\|\qsa\qty( \v, y_1,...,y_q))\right\|^2\\
        =&\E_{y\sim\unif\qty(\binom{[T]}{q})}\left\|\sum_{i\in y}\v^{(i)}\right\|^2/q^2\\
        =&\E_{y\sim\unif\qty(\binom{[T]}{q})}\qty[\sum_{i\in y} \|\v^{(i)}\|^2/q^2] +\E_{y\sim\unif\qty(\binom{[T]}{q})}\qty[\frac{1}{q^2}\sum_{i,j\in y,i\neq j} {\v^{(i)}}^\top\v^{(j)}]\\
        =&\frac{q}{T}\sum_{i=1}^T\frac{1}{q^2}\|\v^{(i)}\|^2 +\frac{q(q-1)}{T(T-1)} \sum_{i\neq j}\frac{1}{q^2}{\v^{(i)}}^\top \v^{(j)}\\
        \geq{}&\frac{T-q}{Tq(T-1)}\|\v\|^2 +\frac{(q-1)}{Tq(T-1)}\left\|\sum_{i=1}^T \v^{(i)}\right\|^2\geq \frac{T-q}{Tq(T-1)}.
    \end{align*}
\end{proof}

\subsection{The worst-case lower bounds for FCNs on $\qsa$}
For completeness, we present the negative results FCNs to approximate the newly formulated $\qsa$ problem. It is using exactly the same technique in Theorem 10 of \citet{sanford2023representational}.
\begin{theorem}
    Let $\mathcal{M}:\R^{dT+q}\to\R^d$ be any FCN employing any activation having first layer width at most $(T-q+1)d-1$, then there exists some input $(\x_{\text{vec}},y)\in \R^{dT+q}$ s.t.
    \begin{equation*}
        \|\mathcal{M}(\x_{\text{vec}},y)-\qsa(\X,y)\|_2^2 \geq \frac{1}{2q}
    \end{equation*}
\end{theorem}
\begin{proof} 
Let $\mathcal{M}(\x) = f(\W \x)$, where $\W$ is the first layer matrix satisfying $\W \in \R^{m \times (dT + q)}, m \leq (T-q+1)d - 1$, and $f: \R^m \rightarrow \R^d$ is an arbitrary function representing the subsequent layers of the FCN. $\W$ can be partitioned as 
\begin{equation*}
    \left[\V_1 ; ...; \V_T; \Tilde{\W}\right],
\end{equation*}
where $\V_1, ..., \V_T \in \R^{m \times d}, \tilde{\W} \in \R^{m \times q}$. Due to our restriction on the width of the first layer, we have
\begin{equation*}
    \rank\left(\left[\V_{q} ; \cdots; \V_T\right]\right) \leq m \leq (T-q+1)d - 1  < (T-q+1)d.
\end{equation*}
This implies that $\left[\V_{q} ; \cdots; \V_T\right]$ has a nontrivial null space containing a nonzero vector $\u = (\u_{q}, ..., \u_T) \in \R^{(T-q+1)d}$. Let 
\begin{equation*}
    \bXi = \frac{1}{2\max_{j \in \{q, ..., T\}}\left\Vert \u_j\right\Vert_2}\u,
\end{equation*}
and (for simplicity, we concatenate $\x_{\text{vec}}$ and y and call it $\x_{\text{vec}}$ in the following proof.)
\begin{align*}
    \x_{\text{vec}} &= (\underbrace{\mbf{0}, ..., \mbf{0}}_{q-1}, \bXi_{q}, ..., \bXi_{T}, y_1, y_2, ..., y_q) \\
    \x'_{\text{vec}} &= (\underbrace{\mbf{0}, ..., \mbf{0}}_{q-1}, -\bXi_{q}, ..., -\bXi_{T}, y_1, y_2, ..., y_q).
\end{align*}
Then 
\begin{equation*}
    \W \x_{\text{vec}} = \V_1\mbf{0}+ ... + \V_{q-1} \mbf{0} + \left[\V_{q} ; \cdots; \V_T\right] \bXi + \tilde{\W} \y = \W \x'_{\text{vec}}, 
\end{equation*}
and for some $j^* \in \{q, ..., T\}$, $\left\Vert\bXi_{j^*}\right\Vert = 1/2$. Consider 
\begin{align*}
    \y &= (1, ..., q-1, j) \\
    \x_{\text{vec}, j} &= (\mbf{0}, ..., \mbf{0}, \bXi_{q}, ..., \bXi_{T}, \y) \\
    \x'_{\text{vec}, j} &= (\mbf{0}, ..., \mbf{0}, -\bXi_{q}, ..., -\bXi_{T}, \y) 
\end{align*}
for each $j \in \{q, ..., T\}$. We observe that 
\begin{equation*}
    \qsa (\x_{\text{vec}, j}) = \frac{1}{q}\bXi_{j} \quad\text{and}\quad \qsa (\x'_{vec, j}) = -\frac{1}{q}\bXi_{j}.
\end{equation*}
It follows that 
\begin{equation*}
    \left\Vert\qsa (\x_{\text{vec}, j^*}) - \qsa (\x'_{\text{vec}, j^*}) \right\Vert = \frac{1}{q}.
\end{equation*}
Because we have shown that $\mathcal{M}(\x_{\text{vec}, j}) = f(\W \x_{\text{vec}, j})= f(\W \x'_{\text{vec}, j}) = \mathcal{M}(\x'_{\text{vec}, j})$, 
\begin{equation*}
    \max \left(\left\Vert\mathcal{M}(\x_{\text{vec}, j}) - \qsa(\x_{\text{vec}, j})\right\Vert, \left\Vert\mathcal{M}(\x'_{\text{vec}, j}) - \qsa(\x'_{\text{vec}, j})\right\Vert\right) \geq \frac{1}{2q}.
\end{equation*}
So in the worst case, $\mathcal{M}$ can approximate $\qsa$ with a loss no better than $1 / 2q$.
\end{proof}

\subsection{Self-attention can approximate $\qsa$}
\label{appen sec: A.3 tf expressivity}
We exhibit the expressivity result for the $q$-sparse token selection task for completeness. Due to the equivalence between our $\qsa$ and the original $q$SA, the proof uses the same method in the proof in \citet{sanford2023representational}. In this method, it is required to reform the input matrix \Cref{eqn: reformed input in thm 6}. In expressivity, it is equivalent to our model.
\begin{theorem}[Consequence of Theorem 2 in \citet{sanford2023representational}]
    For any $T$, $\epsilon$, any $m:= d +2d_e\geq \Omega(d+q\log T)$, there exists some near-orthogonal positional encoding $\bE  \in \R^{d_e\times T}$ and the corresponding $q$-sparse subset encoding $\e_y\in \mathbb{R}^{d_e}$ for $y\in \binom{[T]}{q}$, s.t. there exists some 1-layer self-attention unit with width $m$ that $\epsilon$-approximates $\qsa$.
    \label{thm: 1-layer self-attention}
\end{theorem}
\begin{proof}
First, we choose the near-orthogonal positional encodings to enable the efficient representation of set $y$ and position $i$. We consider some $d_e = C\frac{q\log T}{\delta^2}$ for some constant $C$ (We can pick $\delta=1/4$), and generate Rademacher random vectors. According to Lemma 12 and 13 in \citet{sanford2023representational}, there exist $T$ positional encoding vectors $\e_1, \e_2, ..., \e_T\in\mathbb{R}^{d_e}$ satisfying
\begin{align*}
\langle \e_i, \e_i\rangle &=1\\
|\langle \e_i, \e_j\rangle| &\leq \delta, \qquad i\not=j
\end{align*}

Denote $\e_{y} = h(y)= \bE_y (\bE_y^\top \bE_y)^{-1}\mathbf{1}_q$, then we will have (this is the explicit form of \citet{sanford2023representational} Lemma 13) 
\begin{align*}
\langle \e_i, \e_y\rangle &=1 \text{  for all }i\in y\\
|\langle \e_i, \e_y\rangle| &\leq  \frac{\delta}{1-2\delta}\text{  for all }i\not\in y
\end{align*}

Then we describe the transformer weights. Now we have the input in the following form:
\begin{equation}\label{eqn: reformed input in thm 6}
    \X_{\text{input}} = \begin{bmatrix}
        \x_1 &\x_2 &\cdots&\x_{T-1}&\x_T\\
        \textbf{0}&\textbf{0}& \cdots&\textbf{0}&\e_y\\
        \e_{1}& \e_{2}& \cdots&\e_{T-1}&\e_{T}
    \end{bmatrix}
\end{equation}
With positional encodings, we can construct sparse linear operators $\Q=[\textbf{0}_{d\times d},\bI_{d_e},\textbf{0}_{d_e\times d_e}],\K=[\textbf{0}_{d\times d},\textbf{0}_{d_e\times d_e},\bI_{d_e}],\V$ to represent the $v_{y_i}, e_i$, and $\x_i$ with $\Theta(d + qd_e)$ width:
\begin{align*}
\Q\X &= \qty(0, 0, ...0,,\alpha \e_{y})\in \mathbb{R}^{ d_e\times T}\\
\K\X &= \qty(\e_1, \e_{2}, ..., \e_{T})\in \mathbb{R}^{ d_e\times T}\\
\V\X &= \qty(\x_1, \x_2, ...,\x_T)\in \mathbb{R}^{d \times T}
\end{align*}
We pick $\alpha :=\lceil 2\log(2T)/\epsilon \rceil$ to help the softmax layer to do the average.

Pass the input through the transformer $f$, and we only take the output for the last token with $\e_y$, we have
\begin{align*}
    f(\X)_T &= \V\X\text{softmax}(\X^\top \K^\top \Q \X)[:,T]\\
    & = \sum_{i=1}^T \text{softmax}(\X^\top\K^\top \Q\X)[i,T]\x_{i}
\end{align*}

Analyze the output of the softmax. If $i\in y$, we first consider the upper bound:
\begin{align*}
 \text{softmax}(\X^\top \K^\top \Q \X)[i,T]\leq \frac{e^{\alpha }}{qe^{\alpha}} = \frac{1}{q}
\end{align*}
And then the lower bound
\begin{align*}
 \text{softmax}(\X^\top \K^\top \Q \X)[i,T]\geq \frac{e^{\alpha }}{qe^{\alpha }+Ne^{\frac{\alpha}{2}}} \geq \frac{1}{q}- \frac{T}{q^2}e^{-\alpha}\geq \frac{1}{q} - \frac{\epsilon}{2q}.
\end{align*}
That means 
$$\text{softmax}(\X^\top \K^\top \Q \X)[i,T]\in\qty[\frac{1}{q},\frac{1}{q}+\frac{\epsilon}{2q}]\text{ for all }i\in y_i.$$
If $i\not \in y_i$, then the upper bound of the softmax is
\begin{align*}
\text{softmax}(\X^\top \K^\top \Q \X)[i,T]\leq\frac{e^{\alpha /2}}{qe^{\alpha}}\leq \frac{\epsilon}{2T}.
\end{align*}
 Finally, we conclude the above bounds:
 \begin{align*}
     &\left\|f(\X,y)_T-\qsa(\X,y)_T\right\|_2\\
     \leq{}&\left\|\sum_{i\in y}\qty(\frac{1}{q}-\text{softmax}(\X^\top \K^\top \Q \X)[i,T])\x_{i}-\sum_{i\not\in y}\qty(\text{softmax}(\X^\top \K^\top \Q \X)[i,T])\x_{i}\right\|_2\\
     \leq{}& q\cdot \frac{\epsilon}{2q} + (T-q)\cdot\frac{\epsilon}{2T} \leq \epsilon
 \end{align*}
So we construct a 1-layer transformer that can $\epsilon$-approximate $\qsa$.
\end{proof}

\section{Proof details in \Cref{sec: onehot position encoding}}
\label{appen sec: proof detail for onehot}
In this subsection, we study the gradient descent convergence on the $q$-sparse token selection problem with one-hot positional encoding. We consider updating $\V$ and $\W$ simultaneously with the same learning rate $\eta$. For gradient descent, the update dynamics for $\W(t)$ and $\V(t)$ should be
\begin{align*}
        {\W}(t+1) &= \W(t)-\eta\nabla_{\W}\mathcal{L}(t)\\
    {\V}(t+1) &= \V(t)-\eta\nabla_{\V}\mathcal{L}(t)
\end{align*}

\subsection{GD dynamics and preliminaries}

Based on the reparameterization and the objective in \Cref{eqn: training objective for qsa}, the following lemma shows the gradients of $\W$ and $\V$. Recall that the input matrix is in the following form \begin{equation}
    [\Z,\z_{\text{query}}] := \begin{bmatrix}
        \x_1 &\x_2 &\cdots &\x_{T-1}&\x_T&\x_{\text{query}}\\
        \e_{1}& \e_{2}& \cdots&\e_{T-1}&\e_{T}&\e_{y}
    \end{bmatrix} \in \mathbb{R}^{(d+T) \times (T+1)}.
\end{equation}
where we separate the input tokens $\Z$ and the query token $\z_{\text{query}}$:
\begin{equation}
    \Z := \begin{bmatrix}
        \X\\
        \bE
    \end{bmatrix}  =\begin{bmatrix}
        \x_1 &\x_2 &\cdots &\x_{T-1}&\x_T\\
        \e_{1}& \e_{2}& \cdots&\e_{T-1}&\e_{T}
    \end{bmatrix} \in \mathbb{R}^{(d+T) \times T}, \z_{\text{query}} = \begin{bmatrix}
        \x_{\text{query}}\\\e_{y}
    \end{bmatrix}
\end{equation}

\begin{lemma}\label{Lemma: basic dynamics for 1-layer tf, qsa}
    Denote $\S_y:=\sm(\Z^\top \W \z_{\text{query}})\in \R^{T}$ at time $t$ for certain $q$-sparse set $y$. Also, we define the $q$-hot vector $\Y=(\mathds{1}\{1\in y\}, \mathds{1}\{2\in y\}, ..., \mathds{1}\{T\in y\})$ for the subset $y\in\binom{[T]}{q}$. The gradient dynamics of $\W$ with input token matrix $\X$ is:
\begin{align*}
    \nabla_{\W}\mathcal{L} &=  -\E_{\X,y}\qty(\Z(\diag(\S_y)-\S_y\S_y^\top)\Z^\top \V^\top (\frac{1}{q}\X\Y-\V \Z\S_y)\zq^\top)\\
     \nabla_{\V}\mathcal{L} &= -\E_{\X,y}\qty( \qty(\frac{1}{q}\X\Y-\V \Z\S_y)(\Z\S_y)^\top)
\end{align*}
\end{lemma}
\begin{proof}
    The loss function is as follows according to \Cref{eqn: training objective for qsa} 
$$\mathcal{L}(\bm{\theta}(t))= \frac{1}{2}\E_{\X,y}\qty[\|\qsa(\X,y)-f_{\bm{\theta}}(\X,y)\|_2^2].$$
    Take matrix differentials and we have
    \begin{align*}
        \mathrm{d}\mathcal{L} = &\E_{\X,y} \qty[(f(\X)-\qsa(\X))^\top \V\X\mathrm{d}(\sm(\Z^\top \W \zq))]\\ + &\E_{\X,y}\qty[ (f(\X)-\qsa(\X))^\top\mathrm{d}\V (\Z\S_y)]
    \end{align*}
    To the softmax function, we have $\mathrm{d}\sm(\v)=(\diag(\v)-\v\v^\top)\mathrm{d}\v$. Therefore we have
    \begin{align*}
        \mathrm{d}\mathcal{L}& = \E_{\X,y}\qty(f(\X)-\qsa(\X))^\top \V\Z\mathrm{d}(\sm(\Z^\top \W \zq))\\
        &\quad + \E_{\X,y}\qty[ (f(\X)-\qsa(\X))^\top\mathrm{d}\V (\Z\S_y)]\\
        &=-\E_{\X,y}\qty(\frac{1}{q}\X\Y-\V\Z\S_y)^\top \V\Z\mathrm{d}(\sm(\Z^\top \W \zq))\\
        &\quad - \E_{\X,y}\qty[ \qty(\frac{1}{q}\X\Y-\V\Z\S_y)^\top\mathrm{d}\V (\Z\S_y)]\\
        &=-\E_{\X,y}\qty(\frac{1}{q}\X\Y-\V\Z\S_y)^\top \V\Z(\diag(\S_y)-\S_y\S_y^\top)\Z^\top \mathrm d \W \zq \\
        &\quad - \E_{\X,y}\qty[ \qty(\frac{1}{q}\X\Y-\V\Z\S_y)^\top\mathrm{d}\V (\Z\S_y)]
    \end{align*}
    We have the gradients for the two parameters:
    \begin{align*}
        \nabla_{\W}\mathcal{L} &=  -\E_{\X,y}\qty(\Z(\diag(\S_y)-\S_y\S_y^\top)\Z^\top \V^\top (\frac{1}{q}\X\Y-\V \Z\S_y)\zq^\top)\\
     \nabla_{\V}\mathcal{L} &= -\E_{\X,y}\qty( \qty(\frac{1}{q}\X\Y-\V \Z\S_y)(\Z\S_y)^\top)
    \end{align*}
    Thus we complete the proof.
\end{proof}

Along the gradient trajectory, if the value matrix can be aligned with the ground-truth $\begin{bmatrix}
        \bI_d&\0_{d\times T}
    \end{bmatrix}$, we can have the following nice form for the loss function.
\begin{lemma}    \label{Lemma: population loss expression onehot}
    Denote $\S_y:=\sm(\Z^\top \W \zq)\in \R^{T}$ for certain $q$-sparse set $y$. Also, we define the $q$-hot vector $\Y=(\mathds{1}\{1\in y\}, \mathds{1}\{2\in y\}, ..., \mathds{1}\{T \in y\})^\top$ for the subset $y\in\binom{[T]}{q}$. If $\V(t)=\alpha(t) \begin{bmatrix}
        \bI_d&\0_{d\times T}
    \end{bmatrix}$, the loss function can be represented as the following form:
    $$\mathcal{L}(\bm{\theta}(t))  = \frac{d}{2}\E_y\qty[\left\|\frac{1}{q}\Y-\alpha(t)\S_y\right\|^2]$$
\end{lemma}
\begin{proof}
    We have the architecture $f_{\bm{\theta}}=\V\Z\sm(\Z^\top \W \zq)$ and the loss function in \Cref{eqn: training objective for qsa}:
$$\mathcal{L}(\bm{\theta}(t))= \frac{1}{2}\E_{\X,y}\qty[\|\qsa(\X,y)-f_{\bm{\theta}}(\X,y)\|_2^2].$$ 
    Then we can have the following expression
    \begin{align*}
        \mathcal{L}(\bm{\theta}(t)) &=\frac{1}{2}\E_{\X,y}\qty[\left\|\frac{1}{q}\X\Y-\V(t)\Z\S_y\right\|^2]\\
        &=\frac{1}{2}\E_{\X,y}\qty[\left(\frac{1}{q}\X\Y-\alpha(t)\X\S_y\right)^\top \left(\frac{1}{q}\X\Y-\alpha(t)\X\S_y\right)]\\
        &=\frac{d}{2}\E_{y}\qty[\left\|\frac{1}{q}\Y-\alpha(t)\S_y\right\|^2]
    \end{align*}
    The last identity is due because the expectation of the covariance matrix $\X^\top\X$ is $d\bI_T$.
\end{proof}

In the rest of this section, we characterize the convergence result training with GD for the settings above. Specifically, we consider the one-hot positional encoding. Here we define the encoding for all positions $i\in[T]$ and the encoding for any $q$-subset $y$.
\begin{definition}[One-hot Positional encoding] The one-hot positional encoding is an orthogonal matrix $\bE = \bI_T$, with each column unit one-hot vector $\e_i$ as the positional encoding of the $i$-th token. Thus, for any set $y\in\binom{[T]}{q}$, it holds that $\e_y = \sum_{j\in y}\e_j$. 
\end{definition}

\textbf{Remark.} With one-hot positional encoding, the one-layer transformer can achieve $O(T)$ width cost, which is polynomially smaller than the memory lower bound $\Omega(Td)$ for FCNs. However, it cannot achieve \textit{exponential} separation due to the inefficiency of one-hot encodings.

\subsection{Joint training of value and attention matrix}\label{appen sec: proof detail for joint onehot}
Now, we analyze the dynamics of training the value matrix $\V$ and attention matrix $\W$ simultaneously with the same learning rate $\eta$. Instead of continuous gradient descent training, we directly train the transformer with gradient descent on the population loss.
The following theorem characterizes the convergence of GD when training both layers simultaneously.

\begin{theorem}[Joint training with one-hot positional encoding]    \label{thm: joint training one hot}

    Suppose $2\leq q<T/4$. For any $\epsilon\in (0,\frac{dT}{100(T-q)q}), \eta\leq \frac{1}{20d^2},\x_{\text{query}}=\0_d,$ 
    if we apply gradient descent on the population loss in \Cref{eqn: training objective for qsa} with zero initialization $\W(0)=\textbf{0}_{(d+T)\times (d+T)},\V(0)=\textbf{0}_{d\times (d+T)}$, then after time $t \geq \tilde{O}(\frac{T^2 d}{\eta\epsilon})$, we have
    $$\mathcal{L}(\bm{\theta}(t))= \frac{1}{2}\E_{\X,y}\qty[\|\qsa(\X,y)-f_{\bm{\theta}}(\X,y)\|_2^2] \leq \epsilon.$$ 
\end{theorem}

We have the following lemma that will hold throughout the gradient descent training trajectory. We inductively prove that $\V$ and $\W$ are always along the ground-truth direction, respectively. For convenience, we consider all functions of $\W,\V$ including $\mathcal{L},\S_y$ as a function of $t$. 

\begin{lemma}[Induction Hypothesis]    \label{Lemma: symmetry in different q-sparse subsets joint}
    Given the initialization of $\W(0)=\0$ and $\V(t) = \0$, $\x_{\text{query}} = \0$, then along the gradient descent trajectory, for all $t\geq 0$, we have $\forall y,y'\in\binom{[T]}{q}$:
    \begin{enumerate}
        \item For all $t\geq 0$, there exists some time-dependent scalar $C(t)$ s.t.
       $$\W(t)= C(t) \begin{bmatrix}
    \0_{d\times d} &\0_{d\times T}\\
    \0_{T\times d}&\qty(\bI_T - \frac{1}{T}\mathbf{1}_T\mathbf{1}_T^\top)
\end{bmatrix}.$$
        \item For all $t\geq 0$, there exists some time-dependent scalar $\alpha(t)$ s.t. $$\V(t) = \alpha(t) \begin{bmatrix}
    \bI_d &\0_{d\times T}
\end{bmatrix}.$$
    \end{enumerate}
\end{lemma}
\begin{proof}
    First, observe that the base case when $t=0$ is clearly true. Thus, given that these statements hold up to time $t$, it suffices to prove these for the next iteration $t+1$, and the result will follow by induction.
    
    Note that since the position-position block (right-bottom block) of $\W$ is always in the direction of $\bI_T - \frac{1}{T}\mathbf{1}_T\mathbf{1}_T^\top$ and other entries are 0, so we know
    \begin{align*}
        \S_{y}^{(t)}(i)=s_+(t)>s_-(t)=\S_y^{(t)}(j), \forall i\in y, j\not\in y.
    \end{align*}
    Once again, since we are only considering the dynamics at an arbitrary time $t$ throughout the proof, we will abuse the notation as $s_+:=s_+(t)$, $s_-:=s_-(t)$, $\alpha:= \alpha(t)$, and $\S_y(i):=\S_y^{(t)}(i)$ which is the $i$-th entry of the softmax output. 

    First, suppose that the properties 1 and 2 hold. Recall that \begin{align*}
        \nabla_{\V}\mathcal{L}&= \E_{\X,y}\qty( \qty(\frac{1}{q}\X\Y-\V \Z\S_y)(\Z\S_y)^\top)
    \end{align*}
    Now we consider the token-token block in $\V$. We have that the $(i,j)$-entry ($i\in [d], j\in [d]$) of $\nabla_{\V}\mathcal{L}(t)$ is 
    \begin{align*}
        \e_i^\top \nabla_{\V}\mathcal{L}(t) \e_j &= \E_{\X, y} \qty[\e_i^\top \qty(\frac{1}{q}\X\Y-\V \Z\S_y)(\Z\S_y)^\top \e_j ]\tag{$\V(t) = \alpha(t) \begin{bmatrix}
    \bI_d &\0_{d\times T}
\end{bmatrix}$} \\
        &= \E_{\X,y} \qty[\e_i^\top \X \qty(\frac{1}{q}\Y - \alpha \S_y) (\X \S_y)^\top \e_j] \\
        &= \E_{\X, y}\qty[ \qty(\X \qty(\frac{1}{q}\Y - \alpha \S_y))_i \qty(\X \S_y)_j ] \\ 
        &= \E_{\X, y} \qty[ \qty(\sum_{k=1}^T \X_{i,k} \qty( \frac{\mathds{1}\{k\in y \}}{q} - \alpha \S_y(k) )) \qty(\sum_{k=1}^T \X_{j,k} \S_y(k)) ]
    \end{align*}

    We will case these entries based on whether they are on or off-diagonal.
    
    \textbf{Case I. Diagonal entries} ($i=j$)
    \begin{align*}
        \e_i^\top \nabla_{\V}\mathcal{L}(t) \e_j &= \E_{\X, y} \qty[ \qty(\sum_{k=1}^T \X_{i,k} \qty( \frac{\mathds{1}\{k\in y \}}{q} - \alpha \S_y(k) )) \qty(\sum_{k=1}^T \X_{j,k} \S_y(k)) ] \\
        &= \E_{\X, y} \qty[ \sum_{k=1}^T \X_{i,k}^2 \qty( \frac{\mathds{1}\{k\in y \}}{q} - \alpha \S_y(k) ) \S_y(k)  ] \\
        &= \E_y \qty[\sum_{k=1}^T \qty( \frac{\mathds{1}\{k\in y \}}{q} - \alpha \S_y(k) ) \S_y(k)] \\
        &= \E_y \qty [q\qty(\frac{1}{q}- \alpha s_+)s_+ - (T-q)\alpha s_-^2] \\
        &= (1-\alpha qs_+)s_+ - (T-q) \alpha s_-^2
    \end{align*}

    \textbf{Case II. Off-diagonal entries} ($i\neq j$)
    \begin{align*}
        \e_i^\top \nabla_{\V}\mathcal{L}(t) \e_j &= \E_{\X, y} \qty[ \qty(\sum_{k=1}^T \X_{i,k} \qty( \frac{\mathds{1}\{k\in y \}}{q} - \alpha \S_y(k) )) \qty(\sum_{k=1}^T \X_{j,k} \S_y(k)) ] = 0
    \end{align*}
    This follows from the fact that each cross term of this sum is a product of two mean zero independent Gaussians, and thus the entire expression is 0.
    That means the gradient of $\V$ has its token-token block aligned with identity.

    Then, we consider the position part. We have that the $(i,d+j)$-entry ($i\in [d], j\in [T]$) of $\nabla_{\V}\mathcal{L}(t)$ is 
    \begin{align*}
        \e_i^\top \nabla_{\V}\mathcal{L}(t) \e_j &= \E_{\X, y} \qty[\e_i^\top \qty(\frac{1}{q}\X\Y-\V \Z\S_y)(\Z\S_y)^\top \e_{j+d} ]\tag{$\V(t) = \alpha(t) \begin{bmatrix}
        \bI_d &\0_{d\times T}
        \end{bmatrix}$} \\
        &= \E_{\X,y} \qty[\e_i^\top \X \qty(\frac{1}{q}\Y - \alpha \S_y) (\Z \S_y)^\top \e_{j+d}] \\
        &= \E_{\X, y}\qty[ \qty(\X \qty(\frac{1}{q}\Y - \alpha \S_y))_i \qty( \S_y)_j ] \\ 
        &= \E_{\X, y} \qty[ \qty(\sum_{k=1}^T \X_{i,k} \qty( \frac{\mathds{1}\{k\in y \}}{q} - \alpha \S_y(k) ))  \S_y(j)]
    \end{align*}
    Using induction hypothesis, since 
    $$\W(t)= C(t) \begin{bmatrix}
    \0_{d\times d} &\0_{d\times T}\\
    \0_{T\times d}&\qty(\bI_T - \frac{1}{T}\mathbf{1}_T\mathbf{1}_T^\top)
    \end{bmatrix}.$$
    $\S_y(j)$ is independent of $\X$, so this expectation of the token-position block is all zero:
    \begin{align*}
        \e_i^\top \nabla_{\V}\mathcal{L}(t) \e_j 
        &= \E_{\X, y} \qty[ \qty(\sum_{k=1}^T \X_{i,k} \qty( \frac{\mathds{1}\{k\in y \}}{q} - \alpha \S_y(k) ))  \S_y(j)]\\
        &= \E_{ y} \qty[ \qty(\sum_{k=1}^T \E_{\X}\qty[\X_{i,k}] \qty( \frac{\mathds{1}\{k\in y \}}{q} - \alpha \S_y(k) ))  \S_y(j)]=0
    \end{align*}
    
    Combining the two blocks, we have that $\nabla_{\V}\mathcal{L}(t) = \qty((1-\alpha qs_+)s_+ - (T-q) \alpha s_-^2) \begin{bmatrix}
    \bI_d &\0_{d\times T}
\end{bmatrix}$, which implies that for time $t+1$, $\V(t)$ will also be in the direction of $\begin{bmatrix}
    \bI_d &\0_{d\times T}
\end{bmatrix}$, hence proving the fourth property.

    We now proceed to prove the third property with a similar analysis. We also just need to prove that the gradient can be expressed as:
    $$\nabla_{\W}\mathcal{L}(t)= C_1(t) \onehotWdirection$$
    for some scalar $C_1(t)$, and then it indicates the next iteration $\W(t+1)$ will keep in the same direction.
    Recall that \begin{align*}
        \nabla_{\W}\mathcal{L}(t) &= \E_{\X,y} \qty[\Z\qty(\diag(\S_y) - \S_y\S_y^\top) \Z^\top \V^\top \qty(\frac{1}{q}\X\Y - \V\Z\S_y) \zq^\top]
    \end{align*}
    Since $\V(t)=\alpha(t)\onehotVdirection$ by the induction hypothesis, we have that this is equal to \begin{align*}
        \nabla_{\W}\mathcal{L}(t) &= \alpha  \E_{\X,y} \qty [ \Z\qty(\diag(\S_y) - \S_y\S_y^\top) \X^\top \X\qty(\frac{1}{q}\Y - \alpha \S_y) \zq^\top ]
    \end{align*}
    
    We first consider the first $d$ columns of the gradient, which is 
    \begin{align*}
        \nabla_{\W}\mathcal{L}(t)_{:,\leq d} &= \alpha \E_{\X, y} \qty [ \Z\qty(\diag(\S_y) - \S_y\S_y^\top) \X^\top \X \qty(\frac{1}{q}\Y - \alpha \S_y) {\x_{\text{query}}}^\top ]\\
        & = \0_{(d+T)\times d}.
    \end{align*}
    Because $\x_{\text{query}}=\0_d$.
    
    Next, we consider the top-right block of the gradient $\nabla_{\W}\mathcal{L}(t)_{\leq d,d:T+d}$. 
    By induction hypothesis, $\S_y=\sm(\Z^\top \W(t)\zq) = \sm(C(t)(\bI-\frac{1}{T}\mathbf{1}\mathbf{1}^\top)\e_y)$ at time $t$, which is independent of $\X$. Therefore, by symmetry the gradient of this block is
    \begin{align*}
        \nabla_{\W}\mathcal{L}(t)_{\le d, d:T+d} &= \alpha \E_{\X, y} \qty [ \X\qty(\diag(\S_y) - \S_y\S_y^\top) \X^\top \X \qty(\frac{1}{q}\Y - \alpha \S_y) {\e_y}^\top ]\\
        & = \frac{1}{2}\alpha \E_{\X, y} \qty [ \X\qty(\diag(\S_y) - \S_y\S_y^\top) \X^\top \X \qty(\frac{1}{q}\Y - \alpha \S_y) {\e_y}^\top ]\\
        &-\frac{1}{2}\alpha \E_{\X, y} \qty [ \X\qty(\diag(\S_y) - \S_y\S_y^\top) \X^\top \X \qty(\frac{1}{q}\Y - \alpha \S_y) {\e_y}^\top ]=\0
    \end{align*}
    
    Finally, we consider the position-position block. In particular, since $\E\X^\top \X=d\bI_d$ we have that for $i,j\in[T]$, \begin{align*}
        \nabla_{\W}\mathcal{L}(t)_{(i+d),(j+d)} &= \alpha  \E_{\X,y} \qty[\e_i^\top \qty(\diag(\S_y) - \S_y\S_y^\top) \X^\top \X\qty(\frac{1}{q}\Y - \alpha \S_y) \e_y^\top \e_j] \tag{Here $\e_i$ are one-hot vectors in $\R^T$.}\\
        &=\alpha d \E_y \qty[ \qty(\S_y(i) \e_i^\top - \S_y(i)\S_y^\top) \qty(\frac{1}{q}\Y - \alpha \S_y) \mathds{1}\{j\in y\}] \\
        &= \frac{\alpha d}{\binom{T}{q}} \sum_{y: j\in y} \qty(\S_y(i) \e_i^\top - \S_y(i)\S_y^\top) \qty(\frac{1}{q}\Y - \alpha \S_y) \\ 
        &= \frac{\alpha d}{\binom{T}{q}} \sum_{y:j\in y} \qty( \S_y(i)\mathds{1}\{i \in y\}/q  - \alpha \S_y(i)^2 - \S_y(i) s_+ + \alpha \S_y(i) \|\S_y\|^2)
    \end{align*}
    We now case this between on and off diagonals again.
    
    \textbf{Case I. Diagonal entries} $(i=j)$
    \begin{align*}
        \e_i^\top \nabla_{\W}\mathcal{L}(t) \e_j &= \frac{\alpha d}{\binom{T}{q}} \sum_{y: j\in y} \qty( \S_y(i)\mathds{1}\{i \in y\}/q  - \alpha \S_y(i)^2 - \S_y(i) s_+ + \alpha \S_y(i) \|\S_y\|^2) \\
        &= \frac{\alpha d \binom{T-1}{q-1}}{\binom{T}{q}} s_+ \qty(\frac{1}{q} - \alpha s_+ - s_+ + \alpha \|\S_y\|^2) \\
        &= \frac{\alpha d q}{T} s_+ \qty(\frac{1}{q} - \alpha s_+ - s_+ + \alpha qs_+^2 + \frac{\alpha (1-qs_+)^2}{T-q}) \\
        &= \frac{\alpha d}{T} s_+ \qty (1 - \alpha qs_+ - qs_+ + \alpha q^2 s_+^2 + \frac{\alpha q (1-qs_+)^2}{T-q}) \\
        &= \frac{\alpha d}{T}s_+ \qty((1-\alpha)(1-qs_+) + \frac{T\alpha}{T-q}(1-qs_+)^2)
    \end{align*}

    \textbf{Case II. Off-diagonal entries} $(i\neq j)$
    \begin{align*}
        \e_i^\top \nabla_{\W}\mathcal{L}(t) \e_j &= \frac{\alpha d}{\binom{T}{q}} \sum_{y: j\in y} \qty( \S_y(i) \mathds{1}\{i \in y\}/q  - \alpha \S_y(i)^2 - \S_y(i) s_+ + \alpha \S_y(i) \|\S_y\|^2) \\
        &= \frac{\alpha d}{\binom{T}{q}} \sum_{y: j\in y, i\in y} \qty( \S_y(i) \mathds{1}\{i \in y\}/q  - \alpha \S_y(i)^2 - \S_y(i) s_+ + \alpha \S_y(i) \|\S_y\|^2) \\
        &+ \frac{\alpha d}{\binom{T}{q}}\sum_{y: j\in y, i\notin y} \qty( \S_y(i) \mathds{1}\{i \in y\}/q  - \alpha \S_y(i)^2 - \S_y(i) s_+ + \alpha \S_y(i) \|\S_y\|^2)
    \end{align*}

    The first of these terms is similar in structure to the diagonal entry case:
    \begin{align*}
        & \frac{\alpha d}{\binom{T}{q}} \sum_{y: j\in y, i\in y} \qty( \S_y(i) \mathds{1}\{i \in y\}/q  - \alpha \S_y(i)^2 - \S_y(i) s_+ + \alpha \S_y(i) \|\S_y\|^2)  \\
        ={}& \frac{\alpha d \binom{T-2}{q-2}}{\binom{T}{q}} s_+ \qty (\frac{1}{q} - \alpha s_+ - s_+ + \alpha \|\S_y\|^2) \\
        ={}&\frac{\alpha d (q-1) q}{T(T-1)} s_+ \qty(\frac{1}{q} - \alpha s_+ - s_+ + \alpha \|\S_y\|^2) \\ 
        ={}&\frac{\alpha d (q-1)}{T(T-1)} s_+ \qty((1-\alpha)(1-qs_+) + \frac{T\alpha}{T-q}(1-qs_+)^2)
    \end{align*}
    where the last line follows from the calculation in the diagonal entry case.

    Looking at the second term, we have:
    \begin{align*}
        &\frac{\alpha d}{\binom{T}{q}}\sum_{y: j\in y, i\notin y} \qty( \S_y(i) \mathds{1}\{i \in y\}/q  - \alpha \S_y(i)^2 - \S_y(i) s_+ + \alpha \S_y(i) \|\S_y\|^2) \\
        ={}& \frac{\alpha d \binom{T-2}{q-1}}{\binom{T}{q}} s_- \qty(-\alpha s_- - s_+ + \alpha \|\S_y\|^2) \\
        ={}& \frac{\alpha d q (T-q)}{T(T-1)} \frac{1-qs_+}{T-q} \qty (-\frac{\alpha(1-qs_+)}{T-q} - s_+ + \alpha qs_+^2 + \frac{\alpha(1-qs_+)^2}{T-q}) \\
        ={}& \frac{\alpha dq}{T(T-1)} (1-qs_+) \qty (-s_+ + \alpha qs_+^2 -\frac{qs_+ \alpha (1-qs_+)}{T-q})
    \end{align*}

    From this, we get that $\e_i^\top \nabla_{\W}\mathcal{L}(t) \e_j$ is the following expression multiplied by $\frac{\alpha d}{T(T-1)}$\begin{align*}
         &(q-1)s_+ \qty((1-\alpha)(1-qs_+) + \frac{T\alpha}{T-q}(1-qs_+)^2) + q(1-qs_+) \qty (-s_+ + \alpha qs_+^2 -\frac{qs_+ \alpha (1-qs_+)}{T-q}) \\
         ={}& (q-1) \qty(s_+ - s_+\alpha - qs_+^2 + \alpha qs_+^2 + \frac{T\alpha s_+}{T-q}(1-qs_+)^2) \\
         &+ q \qty(-s_+ + qs_+^2 + \alpha qs_+^2 - \alpha q^2 s_+^3 - \frac{qs_+ \alpha (1-qs_+)^2}{T-q}) \\
         ={}& -\qty(s_+ - s_+\alpha - qs_+^2 + \alpha qs_+^2 + \frac{T\alpha s_+}{T-q}(1-qs_+)^2) + q\qty(-s_+\alpha + 2\alpha qs_+^2 - \alpha q^2s_+^3 + \alpha s_+(1-qs_+)^2) \\
         ={}& -s_+ + s_+\alpha + qs_+^2 - \alpha qs_+^2 - \frac{T\alpha s_+}{T-q}(1-qs_+)^2 - qs_+\alpha + 2\alpha q^2s_+^2 -\alpha q^3 s_+^3 + \alpha qs_+ (1-qs_+)^2 \\
         ={}& -s_+ \qty((1-\alpha)(1-qs_+) + \frac{T\alpha}{T-q}(1-qs_+)^2)
    \end{align*}

    Thus, we have that \begin{align*}
        \e_i^\top \nabla_{\W}\mathcal{L}(t) \e_j = \frac{\alpha d}{T(T-1)} \qty((1-\alpha)(1-qs_+) + \frac{T\alpha}{T-q}(1-qs_+)^2) (-s_+)
    \end{align*}
    for off-diagonal entries.

    Combining our on and off-diagonal entry calculations, we obtain that \begin{align*}
        \nabla_{\W}\mathcal{L}(t) = \qty(\frac{\alpha d}{T-1}s_+ \qty((1-\alpha)(1-qs_+) + \frac{T\alpha}{T-q}(1-qs_+)^2)) \qty( \bI_T  -\frac{1}{T}\1_T \1_T^\top)
    \end{align*}

    Combine all the four blocks, we know the gradient of $\W$ is along the $\onehotWdirection$. This proves property 1 for iteration $t+1$ as desired.
\end{proof}

\paragraph{Remark.} After proving the induction lemma about the evolving direction of $\W(t)$ and $\V(t)$, the optimization problem can be reduced to analyzing the two variable dynamics of $C(t)$ and $\alpha(t)$. 

We can now proceed to prove the main theorem for the joint training algorithm by analyzing the $C(t)$ and $\alpha(t)$ dynamics. One can refer to the main paper for proof ideas.
\begin{proof}[Proof of \Cref{thm: joint training one hot}]
    After \Cref{Lemma: symmetry in different q-sparse subsets joint} shows that $\V$ and $\W$ are always along the ground-truth direction: $\V(t) = \alpha(t)\bI_d,\W(t)=C(t)(\bI_T-\frac{1}{T}\mathbf{1}_T\mathbf{1}_T^\top)$, the dynamics of the parameter matrices then can be characterized by two scalar variable $\alpha(t)$ and $C(t)$. Our update rules become 
\begin{align*}
        \alpha(t+1) &= \alpha(t) + \eta\qty((1-\alpha qs_+)s_+ - \frac{\alpha(1-qs_+)^2}{T-q}) \\
        &=\alpha(t) + \eta s_+(1-\frac{\alpha(t)(Tqs_+^2-2qs_++1)}{(T-q)s_+})\\
        C(t+1) &= C(t) + \eta \frac{\alpha d}{T-1}s_+ \qty((1-\alpha)(1-qs_+) + \frac{T\alpha}{T-q}(1-qs_+)^2)\\
        &=C(t) + \eta \frac{\alpha d}{T-1}s_+ (1-qs_+)\qty(1 + \frac{q\alpha}{T-q}(1-Ts_+))
    \end{align*}
    Here, $s_+(t):=\S_{y}^{(t)}(i), i\in y$ is the correct position softmax probability value at time $t$. We omit the $t$ here for clarity since at each iteration only $s_+(t)$ is related. Along the trajectory, $s_+\leq \frac{1}{q}$ by its definition: since $\W$ is along $(\bI_T-\frac{1}{T}\mathbf{1}_T\mathbf{1}_T^\top)$, all $i\in y$ has the same softmax probability $\S_y^{(t)}(i)$, and thus they cannot exceed the upper bound $1/q$. 
    
    Now we prove that these update rules take $\alpha\to 1, C(t)\to +\infty$ when $t\to+\infty$. Note that when $s_+$ is fixed, $\alpha(t)$ has a stationary point $\alpha^*(t) = \frac{(T-q)s_+}{Tqs_+^2-2qs_++1}$ . That means we can write the $\alpha$ dynamics into:
    $$\alpha(t+1) = \alpha(t) +\eta s_+\qty(1-\frac{\alpha(t)}{\alpha^*(t)}).$$
    One can easily check that when $s_+\in(\frac{1}{T},\frac{1}{q}),\alpha^*(t)\geq 1,$ and it achieves maximum at $s_+=\frac{1}{\sqrt{Tq}}$.
    
    To characterize this limit above when $t\to \infty$, we need to prove the following two arguments:
    \begin{enumerate}
        \item $C(t)$ is non-decreasing for all $t\geq0$.
        \item $\alpha(t)$ goes through 2 phases: 
        
        \textbf{Phase I.} $\alpha$ monotonically grows to $1-0.1\sqrt{\frac{q(T-q)\epsilon}{dT}}$ at some time $t_1$.
        
        \textbf{Phase II.} $\alpha$ stays within an interval whose upper bound is close to $\alpha^*$ after time $t_1$: 
        \begin{equation*}
            \alpha(t)\in\qty[1-0.1\sqrt{\frac{q(T-q)\epsilon}{dT}}, \qty(1+\frac{(8d-1)(1-qs_+)}{8dqs_+(Ts_+-1)})\alpha^*(t)]\tag{IH1}
            \label{IH: alpha stays near alpha star}
        \end{equation*}
    \end{enumerate}

    \textbf{Phase I.}
     We inductively prove that both $\alpha(t), C(t)$ are non-decreasing. For $t=0$, by zero initialization, $C(0)=C(1)=0,\alpha(1)=\frac{\eta}{T}>0$ so it holds for $t=0$. Suppose it holds for some $t<t_1$ before $\alpha$ hit $1-0.1\sqrt{\frac{q(T-q)\epsilon}{dT}}$. Then we know for $\alpha(t+1)$, we have the update rule:
     \begin{align*}
         \alpha(t+1)&=\alpha(t) +\eta s_+\qty(1-\frac{\alpha(t)}{\alpha^*(t)})\geq \alpha(t) + \eta \frac{1}{T}(1-\frac{\alpha(t)}{\alpha^*(t)})\tag{Induction Hyp. $s_+\geq \frac{1}{T}$}\\
        &\geq \alpha(t) + \eta \frac{1}{T}(1-{\alpha(t)})\tag{$\alpha^*\geq 1$}\\
        &\geq \alpha(t)+0.1\eta\sqrt{\frac{q(T-q)\epsilon}{dT^3}}.\tag{$\alpha\leq 1-0.1\sqrt{\frac{q(T-q)\epsilon}{dT}}$}
     \end{align*}
     So $\alpha$ is non-decreasing. Meanwhile, for $C(t+1)$:
     \begin{align*}
         C(t+1) &=C(t) + \eta \frac{\alpha d}{T-1}s_+ (1-qs_+)\qty(1 + \frac{q\alpha}{T-q}(1-Ts_+))\\
         &\geq C(t) + \eta \frac{\alpha d}{T-1}s_+ (1-qs_+)\qty(1 + \frac{q}{T-q}(1-Ts_+))\tag{$\alpha<1,(1-Ts_+)<0$}\\
         &=C(t) + \eta \frac{\alpha d}{T-1}s_+ (1-qs_+)\frac{T-q+q-Tqs_+}{T-q}\\
         &=C(t)+ \eta \frac{\alpha dT}{(T-1)(T-q)}s_+ (1-qs_+)^2\geq C(t)
     \end{align*}
     so they are both non-decreasing. Then we need to upper bound the time $t_1$ for $\alpha(t)$ to reach $1-0.1\sqrt{\frac{q(T-q)\epsilon}{dT}}$: by the update above we have
     \begin{align*}
         \alpha(t+1)&=\alpha(t) +\eta s_+\qty(1-\frac{\alpha(t)}{\alpha^*(t)})\geq \alpha(t) + \eta \frac{1}{T}(1-\frac{\alpha(t)}{\alpha^*(t)})\tag{Induction Hyp. $s_+\geq \frac{1}{T}$}\\
        &\geq \alpha(t) + \eta \frac{1}{T}(1-{\alpha(t)})\tag{$\alpha^*\geq 1$}\\
        \Rightarrow 1-\alpha(t+1)&\leq (1-\eta/T)(1-\alpha(t))\leq ...\leq (1-\eta/T)^{t}(1-\alpha(0)).
     \end{align*}
     Thus for $\alpha(t+1)\geq 1-0.1\sqrt{\frac{q(T-q)\epsilon}{dT}}$, it takes at most $O(\frac{T\log \frac{d}{\epsilon}}{\eta})$ iterations.

     \textbf{Phase II.} In this phase, we first consider $s_+(t)< \frac{1}{\sqrt{Tq}}$. In this case, $\alpha^*(t+1)>\alpha^*(t)$, and inductively
     $$\alpha^*(t)-\alpha(t+1)=\alpha^*(t)-\alpha(t)-\frac{\eta s_+(t)}{\alpha^*(t)}(\alpha^*(t)-\alpha(t)) = (1-\frac{\eta s_+(t)}{\alpha^*(t)})(\alpha^*(t)-\alpha(t)) > 0,$$
     $\alpha(t+1)<\alpha^*(t+1)$ always holds, and the induction hypothesis holds for $t+1$.
     
     All arguments below is based on $s_+(t)\geq \frac{1}{\sqrt{Tq}}$. We first verify that within the induction hypothesis range, $C(t+1)\geq C(t)$.
        \begin{align*}
         C(t+1) &=C(t) + \eta \frac{\alpha d}{T-1}s_+ (1-qs_+)\qty(1 + \frac{q\alpha}{T-q}(1-Ts_+))\\
         &\geq C(t) + \eta \frac{\alpha d}{T-1}s_+ (1-qs_+)\qty(1 + \qty(1+\frac{(d-1)(1-qs_+)}{dqs_+(Ts_+-1)})\alpha^*(t)\frac{q}{T-q}(1-Ts_+))\tag{$\alpha\leq \qty(1+\frac{(d-1)(1-qs_+)}{dqs_+(Ts_+-1)})\alpha^*(t),(1-Ts_+)<0$}\\
         &=C(t) + \eta \frac{\alpha d}{T-1}s_+ (1-qs_+)\cdot\frac{1}{d}\frac{1-qs_+}{Tqs_+^2-2qs_++1}\\
         &=C(t)+ \eta \frac{\alpha }{(T-1)(Tqs_+^2-2qs_++1)}s_+ (1-qs_+)^2\geq C(t)
     \end{align*}
     Next, we first divide $\alpha(t)$'s possible range into two parts: $\alpha(t)\leq \qty(1+\frac{(4d-1)(1-qs_+(t))}{4dqs_+(t)(Ts_+(t)-1)})\alpha^*(t)$ and $\alpha(t)\in \qty[\qty(1+\frac{(4d-1)(1-qs_+(t))}{4dqs_+(t)(Ts_+(t)-1)})\alpha^*(t), \qty(1+\frac{(8d-1)(1-qs_+(t))}{8dqs_+(t)(Ts_+(t)-1)})\alpha^*(t)].$ 
     
     For the first part, we prove the following statement \textbf{(S1)}:
     \begin{center}
         If $\alpha(t)\leq \qty(1+\frac{(4d-1)(1-qs_+(t))}{4dqs_+(t)(Ts_+(t)-1)})\alpha^*(t)$, the next step $$\alpha(t+1)\leq \qty(1+\frac{(8d-1)(1-qs_+(t+1))}{8dqs_+(t+1)(Ts_+(t+1)-1)})\alpha^*(t+1).$$
     \end{center}
    If \textbf{(S1)} is true, then we know once $\alpha(t)\leq \qty(1+\frac{(4d-1)(1-qs_+(t))}{4dqs_+(t)(Ts_+(t)-1)})\alpha^*(t)$, $\alpha(t+1)$ satisfy the induction hypothesis. After proving \textbf{(S1)}, the only part left is when $$\alpha(t)\in \left( \qty(1+\frac{(4d-1)(1-qs_+(t))}{4dqs_+(t)(Ts_+(t)-1)})\alpha^*(t), \left( 1+\frac{(8d-1)(1-qs_+(t))}{8dqs_+(t)(Ts_+(t)-1)} \right) \alpha^*(t)\right].$$
     
     We prove statement \textbf{(S1)} by proving 
     \begin{equation}
         \qty(1+\frac{(8d-1)(1-qs_+(t+1))}{8dqs_+(t+1)(Ts_+(t+1)-1)})\alpha^*(t+1)\geq \qty(1+\frac{(4d-1)(1-qs_+(t))}{4dqs_+(t)(Ts_+(t)-1)})\alpha^*(t)\tag{\textbf{S2}}
     \end{equation}
     When the inequality \textbf{(S2)} above is proved, then \textbf{(S1)} is proved. This is because: if $\alpha(t)<\alpha^*(t)$, then by update rule we have $$\alpha^*(t)-\alpha(t+1)=\alpha^*(t)-\alpha(t)-\frac{\eta s_+(t)}{\alpha^*(t)}(\alpha^*(t)-\alpha(t)) = (1-\frac{\eta s_+(t)}{\alpha^*(t)})(\alpha^*(t)-\alpha(t)) > 0,$$
     \begin{align*}
         \alpha(t+1) < \alpha^*(t) &\leq \qty(1+\frac{(4d-1)(1-qs_+(t))}{4dqs_+(t)(Ts_+(t)-1)})\alpha^*(t) \\&\leq \qty(1+\frac{(8d-1)(1-qs_+(t+1))}{8dqs_+(t+1)(Ts_+(t+1)-1)})\alpha^*(t+1)
     \end{align*}
     If $\alpha(t)\geq \alpha^*(t),$ then $\alpha(t+1)\leq \alpha(t)\leq \qty(1+\frac{(4d-1)(1-qs_+(t))}{4dqs_+(t)(Ts_+(t)-1)})\alpha^*(t)$, and therefore smaller than $\qty(1+\frac{(8d-1)(1-qs_+(t+1))}{8dqs_+(t+1)(Ts_+(t+1)-1)})\alpha^*(t+1)$. 
     
     Now we prove \textbf{(S2)} by expanding the $s_+(t+1)$ using the update rule of $C(t)$. Denote $\Delta C(t):=C(t+1)-C(t).$ Since $\eta\leq \frac{1}{20d^2}$, $\Delta C(t)<\frac{1}{5}$. Then we have
     \begin{align*}
         s_+(t+1)&=\frac{1}{q+(T-q)e^{-C(t)-\Delta C(t)}}\leq \frac{1}{q+(T-q)e^{-C(t)}(1-\Delta C(t))}\\
         &=\frac{1}{q+(T-q)e^{-C(t)}-(T-q)e^{-C(t)}\Delta C(t)}\\
         &\leq \frac{1}{q+(T-q)e^{-C(t)}} + \frac{5}{4}\qty(\frac{1}{q+(T-q)e^{-C(t)}})^2(T-q)e^{-C(t)}\Delta C(t)
     \end{align*}
     The last inequality is due to $\Delta C(t)<\frac{1}{5}$. Then we have
     \begin{align*}
         s_+(t+1) &\leq \frac{1}{q+(T-q)e^{-C(t)}} + \frac{5}{4}\qty(\frac{1}{q+(T-q)e^{-C(t)}})^2(T-q)e^{-C(t)}\Delta C(t)\\
         &=s_+(t) + \frac{5}{4}s_+^2(t)\qty(\frac{1}{s_+}-q)\Delta C(t)\\
         &=s_+(t) +\frac{5}{4}s_+(1-qs_+)\Delta C(t)
     \end{align*}
     Then we consider the decrement of $\alpha^*(t+1)$ and $\frac{(1-qs_+(t+1))}{qs_+(t+1)(Ts_+(t+1)-1)}$.
     \begin{align*}
         \alpha^*(t+1)&=\frac{(T-q)s_+(t+1)}{Tqs_+(t+1)^2-2qs_+(t+1)+1}=\frac{(T-q)}{Tqs_+(t+1)-2q+1/s_+(t+1)}\\
         &\geq \frac{(T-q)}{Tqs_+(t)-2q+1/s_+(t) + \frac{5}{4}Tqs_+(t)(1-qs_+(t))\Delta C(t)}\\
         &\geq \frac{(T-q)}{Tqs_+(t)-2q+1/s_+(t)} - \frac{\frac{5}{4}(T-q)Tqs_+(t)(1-qs_+(t))\Delta C(t)}{(Tqs_+(t)-2q+1/s_+(t))^2}\\
         &=\alpha^*(t)-\frac{5(T-q)Tqs_+^3(t)(1-qs_+(t))\Delta C(t)}{4(Tqs_+^2(t)-2qs_+(t)+1)^2}\\
         &=\alpha^*(t)-\frac{5Tqs_+^2(t)(1-qs_+(t))\Delta C(t)}{4(Tqs_+^2(t)-2qs_+(t)+1)}\alpha^*(t)
     \end{align*}
     \begin{align*}
         &\frac{(1-qs_+(t+1))}{qs_+(t+1)(Ts_+(t+1)-1)}\\={}&\frac{1}{qs_+(t+1)(Ts_+(t+1)-1)} - \frac{1}{(Ts_+(t+1)-1)}\\
         \geq{}& \frac{1}{qs_+(t+1)(Ts_+(t+1)-1)} - \frac{1}{(Ts_+(t)-1)}\tag{$s_+(t+1)\geq s_+(t)$}\\
         ={}&\frac{1}{qT}\cdot \frac{1}{s_+(t+1)}\cdot\frac{1}{s_+(t+1)-\frac{1}{T}}- \frac{1}{(Ts_+(t)-1)}\\
         \geq{}& \frac{1}{qT}\qty(\frac{1}{s_+(t)}-\frac{\frac{5}{4}s_+(1-qs_+)\Delta C(t)}{s_+^2(t)})\qty(\frac{1}{s_+(t)-\frac{1}{T}}-\frac{\frac{5}{4}s_+(1-qs_+)\Delta C(t)}{(s_+-\frac{1}{T})^2})- \frac{1}{(Ts_+(t)-1)}\\
         \geq{}&\frac{(1-qs_+(t))}{qs_+(t)(Ts_+(t)-1)}-\frac{\frac{5}{4}s_+(1-qs_+)\Delta C(t)}{qs_+(t)^2(Ts_+-1)}-\frac{\frac{5}{4}s_+(1-qs_+)T\Delta C(t)}{qs_+(t)(Ts_+-1)^2}\\
         ={}&\frac{(1-qs_+(t))}{qs_+(t)(Ts_+(t)-1)}-\frac{\frac{5}{4}s_+(1-qs_+)\Delta C(t)(2Ts_+-1)}{qs_+(t)^2(Ts_+-1)^2}
     \end{align*}
     Then plug in the original term, we have the lower bound for 
     \begin{align*}
         &\qty(1+\frac{(8d-1)(1-qs_+(t+1))}{8dqs_+(t+1)(Ts_+(t+1)-1)})\alpha^*(t+1)\\
         \geq{}&\qty(1+\frac{(4d-1)(1-qs_+(t))}{4dqs_+(t)(Ts_+(t)-1)})\alpha^*(t)+\frac{(1-qs_+(t))}{8dqs_+(t)(Ts_+(t)-1)}\alpha^*(t)\\
         &-\frac{5s_+(1-qs_+)\Delta C(t)(2Ts_+-1)}{4qs_+(t)^2(Ts_+-1)^2}\alpha^*(t)\\
         &-\qty(1+\frac{(8d-1)(1-qs_+(t+1))}{8dqs_+(t+1)(Ts_+(t+1)-1)})\frac{5Tqs_+^2(t)(1-qs_+(t))\Delta C(t)}{4(Tqs_+^2(t)-2qs_+(t)+1)}\alpha^*(t)
     \end{align*}
     Since $\frac{(8d-1)(1-qs_+(t+1))}{8dqs_+(t+1)(Ts_+(t+1)-1)}\leq \frac{8d-1}{8d}\leq 1$ when $s_+\geq \frac{1}{\sqrt{Tq}}$, we have
    \begin{align*}
         &\qty(1+\frac{(8d-1)(1-qs_+(t+1))}{8dqs_+(t+1)(Ts_+(t+1)-1)})\alpha^*(t+1)\\
         \geq{}&\qty(1+\frac{(4d-1)(1-qs_+(t))}{4dqs_+(t)(Ts_+(t)-1)})\alpha^*(t)+\frac{(1-qs_+(t))}{8dqs_+(t)(Ts_+(t)-1)}\alpha^*(t)\\
         &-2\cdot \frac{5Tqs_+^2(t)(1-qs_+(t))\Delta C(t)}{4(Tqs_+^2(t)-2qs_+(t)+1)}\alpha^*(t)-\frac{\frac{5}{4}s_+(1-qs_+)\Delta C(t)(2Ts_+-1)}{qs_+(t)^2(Ts_+-1)^2}\alpha^*(t)\tag{\textbf{*}}
         \label{ineq: upper bound dynamics}
     \end{align*}
     Then we need to prove that (here $s_+:= s_+(t)$) to show \textbf{(S2)}.
     \begin{align*}
         \frac{(1-qs_+)}{8dqs_+(Ts_+-1)}&\geq \frac{5Tqs_+^2(1-qs_+)\Delta C(t)}{2(Tqs_+^2-2qs_++1)}+\frac{\frac{5}{4}s_+(1-qs_+)\Delta C(t)(2Ts_+-1)}{qs_+^2(Ts_+-1)^2}
     \end{align*}
     We have that the right hand side has the following upper bound ($T\geq 4q$):
     \begin{align*}
         &\frac{5Tqs_+^2(1-qs_+)\Delta C(t)}{2(Tqs_+^2-2qs_++1)}+\frac{\frac{5}{4}s_+(1-qs_+)\Delta C(t)(2Ts_+-1)}{qs_+^2(Ts_+-1)^2}\\
         \leq{}&\frac{5Tqs_+^2(1-qs_+)\Delta C(t)}{2qs_+(Ts_+-1)}+\frac{\frac{5}{4}s_+(1-qs_+)\Delta C(t)(2Ts_+-1)}{qs_+^2(Ts_+-1)^2}\tag{$s_+\leq 1/q.$}\\
         \leq{}&\frac{5Tqs_+^2(1-qs_+)\Delta C(t)}{2qs_+(Ts_+-1)}+\frac{\frac{15}{4}s_+(1-qs_+)\Delta C(t)}{qs_+^2(Ts_+-1)}\tag{\textbf{**}}
     \end{align*}
     Let $\alpha=(1+\gamma \frac{1-qs_+}{qs_+(Ts_+-1)})\alpha^*,$ then we have the upper bound for the update
     \begin{align*}
         \Delta C(t) &= \eta (1-\gamma)(1+\gamma \frac{1-qs_+}{qs_+(Ts_+-1)})\alpha^* \frac{d}{(T-1)(Tqs_+^2-2qs_++1)}s_+(1-qs_+)^2\\
         &\leq \eta(1-\gamma^2)\frac{d(T-q)s_+^2(1-qs_+)^2}{(T-1)(Tqs_+^2-2qs_++1)^2}\\
         &\leq \eta(1-\gamma^2)\frac{d(T-q)(1-qs_+)^2}{(T-1)(Tqs_+-2q+1/s_+)^2}\\
         &\leq \eta(1-\gamma^2)\frac{d(T-q)(1-\frac{\sqrt q}{\sqrt T})^2}{(T-1)(2\sqrt{Tq}-2q)^2}\leq \frac{\eta d}{4qT}.\tag{$s_+\geq 1/\sqrt{Tq}$.}
     \end{align*}
     Since $\eta\leq \frac{1}{20d^2}$, plug the upper bound for $\Delta C(t)$ back to the two terms in \textbf{(**)} respectively and we proved the inequality. And therefore, \textbf{(S2)} is proved, which also leads to \textbf{(S1)}. 

     Finally, we consider $\alpha(t)\in \qty[\qty(1+\frac{(4d-1)(1-qs_+(t))}{4dqs_+(t)(Ts_+(t)-1)})\alpha^*(t), \qty(1+\frac{(8d-1)(1-qs_+(t))}{8dqs_+(t)(Ts_+(t)-1)})\alpha^*(t)]$.
     Now, since by update rule,$$\alpha(t+1) = \alpha(t) + \eta s_+\qty(1-\frac{\alpha(t)}{\alpha^*(t)})\leq \qty(1+\frac{(8d-1)(1-qs_+(t))}{8dqs_+(t)(Ts_+(t)-1)})\alpha^*(t) - \frac{(4d-1)\eta(1-qs_+(t))}{4dq(Ts_+(t)-1)}$$
     We just need to prove that 
     \begin{align*}
         &\qty(1+\frac{(8d-1)(1-qs_+(t+1))}{8dqs_+(t+1)(Ts_+(t+1)-1)})\alpha^*(t+1)\\\geq{}& \qty(1+\frac{(8d-1)(1-qs_+(t))}{8dqs_+(t)(Ts_+(t)-1)})\alpha^*(t) - \frac{(4d-1)\eta(1-qs_+(t))}{4dq(Ts_+(t)-1)}
     \end{align*}
     Note \Cref{ineq: upper bound dynamics} gives the lower bound for the left hand side:
         \begin{align*}
         &\qty(1+\frac{(8d-1)(1-qs_+(t+1))}{8dqs_+(t+1)(Ts_+(t+1)-1)})\alpha^*(t+1)\\
         \geq{}&\qty(1+\frac{(8d-1)(1-qs_+(t))}{8dqs_+(t)(Ts_+(t)-1)})\alpha^*(t)- \frac{5Tqs_+^2(t)(1-qs_+(t))\Delta C(t)}{2(Tqs_+^2(t)-2qs_+(t)+1)}\alpha^*(t)\\
         &-\frac{\frac{5}{4}s_+(1-qs_+)\Delta C(t)(2Ts_+-1)}{qs_+(t)^2(Ts_+-1)^2}\alpha^*(t)\tag{\textbf{\#}}
         \label{ineq: upper bound @}
     \end{align*}
     Yet when $\alpha(t)\geq \qty(1+\frac{(4d-1)(1-qs_+(t))}{4dqs_+(t)(Ts_+(t)-1)})\alpha^*(t)$, we have a better upper bound for $\Delta C(t)$:
     \begin{align*}
         \Delta C(t) &= \eta (1-\gamma)(1+\gamma \frac{1-qs_+}{qs_+(Ts_+-1)})\alpha^* \frac{d}{(T-1)(Tqs_+^2-2qs_++1)}s_+(1-qs_+)^2\\
         &\leq \eta(1-\gamma^2)\frac{d(T-q)s_+^2(1-qs_+)^2}{(T-1)(Tqs_+^2-2qs_++1)^2}\\
         &\leq \eta\qty(1-\qty(\frac{4d-1}{4d})^2)\frac{d(T-q)s_+^2(1-qs_+)^2}{(T-1)(Tqs_+^2-2qs_++1)^2}\\
         &\leq \frac{\eta}{2}\frac{(T-q)s_+^2(1-qs_+)^2}{(T-1)(Tqs_+^2-2qs_++1)^2}
     \end{align*}
     Then we need to bound both terms in \Cref{ineq: upper bound @} (for simplicity denote $s_+$ as $s_+(t)$):
     \begin{align*}
         &\frac{5Tqs_+^2(t)(1-qs_+(t))\Delta C(t)}{2(Tqs_+^2(t)-2qs_+(t)+1)}\alpha^*(t)
         \leq\frac{5\eta Tqs_+^5(1-qs_+)^3(T-q)^2}{4(T-1)(Tqs_+^2-2qs_++1)^4}\\
         ={}&\frac{5\eta}{4}\cdot \frac{T(T-q)^2}{T-1}\cdot\frac{(1-qs_+)}{Tqs_+^2-2qs_++1}\cdot\frac{qs_+^5(1-qs_+)^2}{(Tqs_+^2-2qs_++1)^3}\\
         \leq{}&\frac{5\eta}{4}\cdot \frac{T(T-q)^2}{T-1}\cdot\frac{(1-qs_+)}{qs_+(Ts_+-1)}\cdot\frac{qs_+^5(1-qs_+)^2}{(Tqs_+^2-2qs_++1)^3}\tag{$s_+\leq \frac{1}{q}$}\\
         ={}&\frac{5\eta}{4}\cdot \frac{T(T-q)^2}{T-1}\cdot\frac{(1-qs_+)}{q(Ts_+-1)}\cdot\frac{q(1-qs_+)^2}{(Tqs_+^2-2qs_++1)(Tq-2q/s_++1/s_+^2)^2}\\
         \leq{}&\frac{5\eta}{4}\cdot \frac{T(T-q)^2}{T-1}\cdot\frac{(1-qs_+)}{q(Ts_+-1)}\cdot\frac{q}{2(Tq-q^2)^2}\tag{$s_+\geq \frac{1}{\sqrt{Tq}}.$}\\
         ={}&\frac{5\eta}{8q}\cdot \frac{T}{T-1}\cdot\frac{(1-qs_+)}{q(Ts_+-1)}\leq \frac{5\eta}{14}\frac{(1-qs_+)}{q(Ts_+-1)}\tag{$q\geq 2,T\geq 4q$}
     \end{align*}
     \begin{align*}
         &\frac{5s_+(1-qs_+)\Delta C(t)(2Ts_+-1)}{4qs_+(t)^2(Ts_+-1)^2}\alpha^*(t)\\
         \leq{}&\frac{5(1-qs_+)(2Ts_+-1)}{4qs_+(Ts_+-1)^2}\cdot\frac{\eta(T-q)^2s_+^3(1-qs_+)^2}{2(T-1)(Tqs_+^2-2qs_++1)^3}\tag{Plug in $\Delta C(t)$}\\
         ={}&\frac{5\eta(1-qs_+)(2Ts_+-1)(T-q)^2s_+^2(1-qs_+)^2}{8q(Ts_+-1)^2(T-1)(Tqs_+^2-2qs_++1)^3}\\
         \leq{}& \frac{5\eta(1-qs_+)}{8q(Ts_+-1)}\cdot \frac{2Ts_+-1}{Ts_+-1}\cdot \frac{(T-q)^2}{T-1}\cdot\frac{1}{Tq-2q/s_++1/s_+^2}\cdot\frac{(1-qs_+)^2}{(Tqs_+^2-2qs_++1)^2}\\
         \leq{}&\frac{5\eta(1-qs_+)}{8q(Ts_+-1)}\cdot 3\cdot \frac{(T-q)^2}{T-1}\frac{1}{Tq-q^2}\frac{1}{4}
         \leq{}\frac{15\eta(1-qs_+)}{32q^2(Ts_+-1)}\leq \frac{15\eta(1-qs_+)}{64q(Ts_+-1)}.
     \end{align*}     
     Therefore, by \Cref{ineq: upper bound @} we have
     \begin{align*}
         &\qty(1+\frac{(8d-1)(1-qs_+(t+1))}{8dqs_+(t+1)(Ts_+(t+1)-1)})\alpha^*(t+1)\\
         \geq{}&\qty(1+\frac{(8d-1)(1-qs_+(t))}{8dqs_+(t)(Ts_+(t)-1)})\alpha^*(t)
         - \frac{5Tqs_+^2(t)(1-qs_+(t))\Delta C(t)}{2(Tqs_+^2(t)-2qs_+(t)+1)}\alpha^*(t)\\&-\frac{\frac{5}{4}s_+(1-qs_+)\Delta C(t)(2Ts_+-1)}{qs_+(t)^2(Ts_+-1)^2}\alpha^*(t)\\
         \geq{}&\qty(1+\frac{(8d-1)(1-qs_+(t))}{8dqs_+(t)(Ts_+(t)-1)})\alpha^*(t)-\qty(\frac{5}{14}+\frac{15}{64})\frac{\eta(1-qs_+)}{q(Ts_+-1)}\\
         \geq{}&\qty(1+\frac{(8d-1)(1-qs_+(t))}{8dqs_+(t)(Ts_+(t)-1)})\alpha^*(t)-\frac{4d-1}{4d}\frac{\eta(1-qs_+)}{q(Ts_+-1)}\tag{$d\geq 1$.}\\
         \geq{}&\alpha(t+1).
     \end{align*}
     Therefore, we finish the induction for (\ref{IH: alpha stays near alpha star}). 

     With the induction hypothesis, we can analyze the upper bound of convergence time. For \textbf{Phase I}, we have the lower bound for $\Delta C(t)$ (note that $\alpha < 1$, $\gamma < 0$):
     \begin{align*}
         \Delta C(t) &\geq \eta (1-\gamma)(1-0.1\sqrt{\frac{q(T-q)\epsilon}{dT}}) \frac{d}{(T-1)(Tqs_+^2-2qs_++1)}s_+(1-qs_+)^2\\
         &\geq \frac{\eta ds_+(1-qs_+)^2}{2(T-1)(Tqs_+^2-2qs_++1)}
     \end{align*}
     And for \textbf{Phase II}, we have the lower bound since $\gamma\leq \frac{8d-1}{8d}$:
     \begin{align*}
         \Delta C(t) &= \eta (1-\gamma)(1+\gamma \frac{1-qs_+}{qs_+(Ts_+-1)})\alpha^* \frac{d}{(T-1)(Tqs_+^2-2qs_++1)}s_+(1-qs_+)^2\\
         &\geq \frac{\eta}{8}\frac{(T-q)s_+^2(1-qs_+)^2}{(T-1)(Tqs_+^2-2qs_++1)^2}\geq \frac{\eta}{8}\frac{(1-qs_+)^2}{q(T-1)(Tqs_+^2-2qs_++1)}
     \end{align*}
    Then $\Delta C(t)\geq \frac{\eta}{8}\frac{(1-qs_+)^2}{q(T-1)(Tqs_+^2-2qs_++1)}$ for all $t$. Then we divide the training trajectory into two stages as in \citet{huang2023context}: in the first stage, $s_+$ grows to $1/2q$. In the second one, $C(t)$ grows large enough s.t. $s_+\geq \frac{1}{q}-\frac{1}{q}\sqrt{\frac{q(T-q)\epsilon}{dT}}$.

    For the first stage, it's necessary that $C(t)\geq \log\frac{T-q}{q}$. While since $s_+<\frac{1}{2q}$, we have $$\Delta C(t)\geq \frac{\eta \cdot \frac{1}{4}}{8q(T-1)(Tqs_+^2-2qs_++1)}\geq \frac{\eta }{8T^2}$$ It takes at most $O(\frac{T^2\log T}{\eta})$ iteration for $C(t)$ to reach this value.

    For the second stage, we need $C(t)\geq \log\qty(\frac{1}{q}\sqrt{\frac{dT}{q(T-q)\epsilon}})$. Since $s_+\leq \frac{1}{q}-\frac{1}{q}\sqrt{\frac{q(T-q)\epsilon}{dT}}$ during this period, we can lower bound the increment:$$\Delta C(t)\geq\frac{\eta}{8}\frac{(1-qs_+)^2}{q(T-1)(Tqs_+^2-2qs_++1)}\geq \frac{\eta q\epsilon}{8d(T-1)T}.$$
    So it takes at most $O(\frac{T^2d\log(\frac{dT}{\epsilon})}{\epsilon\eta})$ iterations.

    Finally, we check that if $s_+\geq\frac{1}{q}-\frac{1}{q}\sqrt{\frac{q(T-q)\epsilon}{dT}}$, then the loss is smaller than $\epsilon$. 
    \begin{align*}
        \mathcal{L}(\bm{\theta}(t)) &=\frac{1}{2}\E\qty[\left\|\frac{1}{q}\X\Y-\V(t)\Z\S_y\right\|^2]\\
        &=\frac{1}{2}\E\qty[\left(\frac{1}{q}\X\Y-\alpha(t)\X\S_y\right)^\top \left(\frac{1}{q}\X\Y-\alpha(t)\X\S_y\right)]\\
        &=\frac{1}{2}\E\qty[d\left\|\frac{1}{q}\Y-\alpha(t)\S_y\right\|^2]\\
        &=\frac{d}{2(T-q)}\qty((T-q)q\qty(\alpha(t)s_+ - \frac{1}{q})^2 + {\alpha(t)^2(1-qs_+)^2})\\
    \end{align*}
    While $\alpha(t)\in\qty[1-0.1\sqrt{\frac{q(T-q)\epsilon}{dT}}, \qty(1+\frac{(8d-1)(1-qs_+)}{8dqs_+(Ts_+-1)})\alpha^*(t)]:=[\alpha_1,\alpha_2]$, the loss value is upper bounded by $$\max_{j\in\{1,2\}}\frac{d}{2(T-q)}\qty((T-q)q\qty(\alpha_j s_+ - \frac{1}{q})^2 + {\alpha_j^2(1-qs_+)^2})$$

    For $\alpha_1 = 1-0.1\sqrt{\frac{q(T-q)\epsilon}{dT}}$, we have 
    \begin{align*}
        \mathcal{L}(\bm{\theta}(t))={}&\frac{d}{2(T-q)}\qty((T-q)q\qty(\alpha(t)s_+ - \frac{1}{q})^2 + {\alpha(t)^2(1-qs_+)^2})\\
        \leq{}&\frac{d}{2(T-q)}\qty(\frac{T-q}{q}\qty(q\alpha(t)s_+ - 1)^2 + {(1-qs_+)^2})\\
        \leq{}&\frac{d}{2(T-q)}\qty(\frac{T-q}{q}\qty(qs_+ - 1+0.1qs_+\sqrt{\frac{q(T-q)\epsilon}{dT}})^2 + {(1-qs_+)^2})\\
        \leq{}&\frac{d}{2(T-q)}\qty(\frac{T-q}{q}\cdot \qty((1-qs_+)^2 +0.2\sqrt{\frac{q(T-q)\epsilon}{dT}}(1-qs_+)+0.01\cdot\frac{q(T-q)\epsilon}{dT})+ (1-qs_+)^2)\\
        \leq{}&\frac{d}{2(T-q)}\qty(\frac{T-q}{q}\cdot \qty((1-qs_+)^2 +0.2\sqrt{\frac{q(T-q)\epsilon}{dT}}\cdot\sqrt{\frac{q(T-q)\epsilon}{dT}}+0.01\cdot\frac{q(T-q)\epsilon}{dT}))\\&+ \frac{d}{2(T-q)}(1-qs_+)^2\\
        \leq{}&\frac{d}{2(T-q)}\frac{T}{q} (1-qs_+)^2 + 0.21\epsilon\leq \epsilon.
    \end{align*}
    For $\alpha_2 = \qty(1+\frac{(8d-1)(1-qs_+)}{8dqs_+(Ts_+-1)})\alpha^*(t)$, denote $\Delta \alpha=\alpha_2-\alpha^*$. Then $\Delta \alpha s_+\leq \frac{2\sqrt{\frac{q(T-q)\epsilon}{dT}}}{T}\alpha^*$. Also, for $\alpha^*$ we have a upper bound (using $\epsilon\leq \frac{dT}{100(T-q)q})$):
    \begin{align*}
        \alpha^*=\frac{T-q}{Tqs_+-2q+\frac{1}{s_+}}\leq \frac{1}{qs_+}\leq \frac{1}{1-\sqrt{\frac{q(T-q)\epsilon}{dT}}}\leq 1+\frac{6}{5}\sqrt{\frac{q(T-q)\epsilon}{dT}}\leq \frac{28}{25}.
    \end{align*}
    Then the loss can be upper bounded:
    \begin{align*}
        \mathcal{L}(\bm{\theta}(t))={}&\frac{d}{2(T-q)}\qty((T-q)q\qty(\alpha_2s_+ - \frac{1}{q})^2 + {\alpha_2^2(1-qs_+)^2})\\
        ={}&\frac{d}{2(T-q)}\qty((T-q)q\qty(\alpha^* s_+ - \frac{1}{q})^2 + {\alpha^*}^2(1-qs_+)^2)\\
        +&\frac{d}{2(T-q)}\qty(-(T-q)\frac{4\sqrt{\frac{q(T-q)\epsilon}{dT}}}{T}\alpha^*(1-qs_+\alpha^*) + \frac{4q^2(T-q)^2\epsilon}{dT^3}{\alpha^*}^2)\\
        +&\frac{d}{2(T-q)} \qty(\frac{4\sqrt{\frac{q(T-q)\epsilon}{dT}}}{T}{\alpha^*}^2(1-qs_+)^2+\frac{4q(T-q)\epsilon}{dT^3}{\alpha^*}^2(1-qs_+)^2)\\
        ={}&\frac{d}{2(T-q)}\qty((T-q)q\qty(\alpha^* s_+ - \frac{1}{q})^2 + {\alpha^*}^2(1-qs_+)^2)\\
        +&\frac{d}{2(T-q)}\qty(-(T-q)\frac{4\sqrt{\frac{q(T-q)\epsilon}{dT}}}{T(Tqs_+^2-2qs_++1)}\alpha^*(1-qs_+)^2 + \frac{4q^2(T-q)^2\epsilon}{dT^3}{\alpha^*}^2)\\
        +&\frac{d}{2(T-q)} \qty(\frac{4\sqrt{\frac{q(T-q)\epsilon}{dT}}}{T}{\alpha^*}^2(1-qs_+)^2+\frac{4q(T-q)\epsilon}{dT^3}{\alpha^*}^2(1-qs_+)^2)\\
        ={}&\frac{dT}{2q(T-q)}(qs_+(t)-1)^2
        +\frac{d}{2(T-q)}\qty(-\frac{4\sqrt{\frac{q(T-q)\epsilon}{dT}}}{Ts_+}{\alpha^*}^2(1-qs_+)^2 + \frac{4q^2(T-q)^2\epsilon}{dT^3}{\alpha^*}^2)\\
        +&\frac{d}{2(T-q)} \qty(\frac{4\sqrt{\frac{q(T-q)\epsilon}{dT}}}{T}{\alpha^*}^2(1-qs_+)^2+\frac{4q(T-q)\epsilon}{dT^3}{\alpha^*}^2(1-qs_+)^2)\\
        \leq{} &\frac{\epsilon}{2} + \frac{d}{2(T-q)}\qty(\frac{4q^2(T-q)^2\epsilon}{dT^3}{\alpha^*}^2+\frac{4q(T-q)\epsilon}{dT^3}{\alpha^*}^2(1-qs_+)^2)\leq \epsilon.
    \end{align*}
    when $T\geq 4q$. 

    In conclusion, after \textbf{Phase I} (at most $O(\frac{T\log \frac{d}{\epsilon}}{\eta})$ iterations) and \textbf{Phase II} (takes at most $O(\frac{T^2\log T}{\eta}+\frac{T^2d\log(\frac{dT}{\epsilon})}{\eta \epsilon})$ iterations), the population loss $\mathcal{L}(\bm{\theta}(t))\leq \epsilon$
    after time $\tilde{O}(\frac{T^2 d}{\eta\epsilon})$.
\end{proof}

\section{Proof details in \Cref{sec: stochastic positional encoding}}
\label{appen sec: proof detail for stochastic pe}
Though \citet{sanford2023representational} demonstrated the success of the representational power of one-layer transformers on the original $q$SA task, it does not guarantee gradient descent can converge to the constructed solution for $\W$. For our simplified task $\qsa$, the problem remains. Experiments (see \Cref{sec: experiments}) show that if the positional encoding is fixed during training, even though the trained model can express $\qsa$ when it has the original positional encoding, the performance can be drastically bad after switching to another set of valid positional encoding. This motivates us to consider the population loss with the resampling of positional encoding. 

To be specific, we further consider resampling from all possible, yet valid positional encodings to get a transformer with a stochastic positional encoding module. In the construction of \ref{appen sec: A.3 tf expressivity}, as long as the positional encoding matrix $\bE$ satisfy the $(q,\delta)$-restricted isometry property (RIP) and subset encoding $\e_y=\bE_y(\bE_y^\top \bE_y)^{-1}\mathbf{1}_q$ is used, $\W=\alpha \bI_d$ with infinite large $\alpha$ can approximate $\qsa$. Therefore, we condition on that the random Rademacher matrix has the $(q,\delta)$-restricted isometry and orthogonality property for some $\delta$, and take the expectation of the model output as a transformer with a stochastic positional encoding. After adding the randomized architecture, the model becomes:
\begin{definition}[Transformer with Stochastic Positional Encoding] Define a reparameterized 1-layer self-attention layer with stochastic positional encoding as the following model with trainable parameter matrix $\V, \W$ where $\W\in\R^{d_e\times d_e},\V\in\R^{d\times d}$:
\begin{align*}
    f^{(s)}_{\bm{\theta}} (\X,y)&=\E_{\bE}\qty[\V\Z\sm(\Z^\top \W \zq)\Big|\bE\text{ satisfies $(q,\delta)$-RIP}]
\end{align*}
\end{definition}
Recall the definition of $\Z$ and $\zq$:
\begin{equation}
    \Z := \begin{bmatrix}
        \X\\
        \bE
    \end{bmatrix}  =\begin{bmatrix}
        \x_1 &\x_2 &\cdots &\x_{T-1}&\x_T\\
        \e_{1}& \e_{2}& \cdots&\e_{T-1}&\e_{T}
    \end{bmatrix} \in \mathbb{R}^{(d+T) \times T}, \z_{\text{query}} = \begin{bmatrix}
        \x_{\text{query}}\\\e_{y}
    \end{bmatrix}
\end{equation}

For simplicity, we denote the conditional expectation of a random variable $\bm{x}$ as:
$$    \E_{\bE}^{(R)}\qty[\bm{x}] = \E_{\bE}\qty[\bm{x}\Big|\bE\text{ satisfies $(q,\delta)$-RIP}]
$$
The training objective becomes:
\begin{equation}
\label{appen eqn: training objective for stochastic pe}
\mathcal{L}(\bm{\theta}) = \frac{1}{2}\E_{\X, y}\qty[\|\qsa(\X,y)-f^{(s)}_{\bm{\theta}} (\X,y)\|_2^2]. 
\end{equation}

\subsection{Notations} 
\label{appen subsec: notations} 
In this section, we introduce the notations used in the proof of \Cref{thm: joint training stochastic pe} for simplicity. One can refer to this section once some undefined notations are found. 

We apply the same notation $\S_y$ for simplicity: denote the attention score $$\S_y^{(t)}:=\sm(\Z^\top \W(t) \zq)\in \R^{T}$$ for certain $q$-sparse set $y\in\binom{[T]}{q}$. We denote the $i$-th entry for the attention score as
$$\S_y^{(t)}(i):=\sm(\Z^\top \W(t) \zq)_i\in \R$$

For clarity, we ignore the timestamp and overload the notation as $\S_y$ for all time $t$ during gradient calculation if all variables that occur in the expression are at time $t$. For GD updates, we will specify the difference between $\S_y^{(t)}$ and $\S_y^{(t+1)}$. 

For $\W(t)=C(t)\randomWdirection$, we know before softmax layer the pre-attention score $C(t)$ is the same for any $i,j\in y$ since $\e_i^\top \e_y = \e_j^\top \e_y = 1.$ That means $\S_y(i)=\S_y(j)$. During the proof, we denote $s_+ = \S_y(i)$ with $i\in y$ when $\W(t)=C(t)\randomWdirection$ for some scalar (which is always true in the proof of the two theorems).

After introducing the stochastic architecture, we have $\E_{\bE}^{(R)}[\S_y]$ in our dynamics. For simplicity, we define $\bar{\S_y} = \E_{\bE}^{(R)}[\S_y]$. By \Cref{Lemma: Conditional expectation of the softmax vector}, we know for $\W(t)=C(t)\randomWdirection$ (which is always true in the two theorems), there exists $\bar{s_+}(t), \bar{s_-}(t)$ s.t.
    $$\bar{\S_y} = \bar{s_+}\Y + \bar{s_-}(\mathbf{1}_T-\Y).$$
And here $\bar{s_+}:= \E_{\bE}^{(R)}[\S_y(i)], i\in y, \bar{s_-}:= \E_{\bE}^{(R)}[\S_y(i)], i\not\in y.$

In the following subsections, we will separate the gradient matrices into blocks. For matrices in the same shape of $\W\in \R^{(d+d_e)\times (d+d_e)}$, we denote $\W_{:,\leq d}\in \R^{(d+d_e)\times d}$ as the submatrix formed by the first $d$ columns. Then, we define $\W_{:,>d}$ as the submatrix formed by the last $d_e$ columns. Similarly, we denote $\W_{\leq d, \leq d}, \W_{> d, \leq d}, \W_{\leq d, > d}, \W_{> d, >d}$ as the four different block submatrices, respectively. Similar notations also apply in the case of $\V$.
\subsection{GD dynamics for stochastic PE}
Before introducing the main theorem in this section, we first provide the expression of the gradient dynamics. Recall that the input matrix is in the following form \begin{equation}
    [\Z,\z_{\text{query}}] := \begin{bmatrix}
        \x_1 &\x_2 &\cdots &\x_{T-1}&\x_T&\x_{\text{query}}\\
        \e_{1}& \e_{2}& \cdots&\e_{T-1}&\e_{T}&\e_{y}
    \end{bmatrix} \in \mathbb{R}^{(d+d_e) \times (T+1)}.
\end{equation}
where we separate the input tokens $\Z$ and the query token $\z_{\text{query}}$:
\begin{equation}
    \Z := \begin{bmatrix}
        \X\\
        \bE
    \end{bmatrix}  =\begin{bmatrix}
        \x_1 &\x_2 &\cdots &\x_{T-1}&\x_T\\
        \e_{1}& \e_{2}& \cdots&\e_{T-1}&\e_{T}
    \end{bmatrix} \in \mathbb{R}^{(d+d_e) \times T}, \z_{\text{query}} = \begin{bmatrix}
        \x_{\text{query}}\\\e_{y}
    \end{bmatrix}
\end{equation}

With the stochastic positional encoding introduced, we have the following lemma that shows the dynamics of $\W,\V$: 

\begin{lemma}\label{Lemma: basic dynamics for 1-layer tf with stochastic pe, qsa}
    Denote $\S_y:=\sm(\Z^\top \W \zq)\in \R^{T}$ for certain $q$-sparse set $y$. Also, we define the $q$-hot vector $\Y=(\mathds{1}\{1\in y\}, \mathds{1}\{2\in y\}, ..., \mathds{1}\{T\in y\})$ for the subset $y\in\binom{[T]}{q}$. The gradient dynamics of $\W$ with input $\X$ is:
\begin{align*}
    \nabla_{\W}\mathcal{L}&= -\E_{\X,y}\E_{\bE}^{(R)}\qty(\Z(\diag(\S_y)-\S_y\S_y^\top)\Z^\top \V^\top (\frac{1}{q}\X\Y-\V \Z\E_{\bE}^{(R)}[\S_y])\zq^\top)\\
    \nabla_{\V}\mathcal{L}&= -\E_{\X,y}\qty( \qty(\frac{1}{q}\X\Y-\V \Z\E_{\bE}^{(R)}[\S_y])(\Z\E_{\bE}^{(R)}[\S_y])^\top)
\end{align*}
\end{lemma}

\begin{proof}
    The loss function is $$\mathcal{L}(\bm{\theta}) = \frac{1}{2}\E_{\X, y}\qty[\|\qsa(\X,y)-f^{(s)}_{\bm{\theta}} (\X,y)\|_2^2].$$

    Take matrix differentials and we have
    \begin{align*}
        \mathrm{d}\mathcal{L} = &\E_{\X,y} \qty[(f^{(s)}_{\bm{\theta}} (\X,y)-\qsa(\X,y))^\top \V\Z\mathrm{d}(\E^{(R)}_{\bE}\qty[\sm(\Z^\top \W \zq)])]\\ + &\E_{\X,y}\qty[ (f^{(s)}_{\bm{\theta}} (\X,y)-\qsa(\X,y))^\top\mathrm{d}\V \E_{\bE}^{(R)}(\Z\S_y)]
    \end{align*}

    To the softmax function, we have $\mathrm{d}\sm(\v)=(\diag(\v)-\v\v^\top)\mathrm{d}\v$. Therefore we have
    \begin{align*}
        \mathrm{d}\mathcal{L}&=\E_{\X,y} \qty[(f^{(s)}_{\bm{\theta}} (\X,y)-\qsa(\X,y))^\top \V\Z\mathrm{d}(\E^{(R)}_{\bE}\qty[\sm(\Z^\top \W \zq)])]\\  &\quad+\E_{\X,y}\qty[ (f^{(s)}_{\bm{\theta}} (\X,y)-\qsa(\X,y))^\top\mathrm{d}\V \E_{\bE}^{(R)}(\Z\S_y)]\\
        &=-\E_{\X,y}\qty(\frac{1}{q}\X\Y-\V\Z\E_{\bE}^{(R)}\qty[\S_y])^\top \V\Z\E_{\bE}^{(R)}\mathrm{d}(\sm(\Z^\top \W \zq))\\
         &\quad - \E_{\X,y}\qty[ \qty(\frac{1}{q}\X\Y-\V\Z\E_{\bE}^{(R)}[\S_y])^\top\mathrm{d}\V \E_{\bE}^{(R)}(\Z\S_y)]\\
        &=-\E_{\X,y}\qty(\frac{1}{q}\X\Y-\V\Z\E_{\bE}^{(R)}[\S_y])^\top \V\Z\E_{\bE}^{(R)}\qty[(\diag(\S_y)-\S_y\S_y^\top)\Z^\top \mathrm d \W \zq]\\
        &\quad -\E_{\X,y}\qty[ \qty(\frac{1}{q}\X\Y-\V\Z\E_{\bE}^{(R)}[\S_y])^\top\mathrm{d}\V (\Z\E_{\bE}^{(R)}[\S_y])]
    \end{align*}
    Since $\frac{\mathrm d \W}{\mathrm d t} = -\frac{\partial \mathcal{L}}{\partial \W}$, we have
    $$\nabla_{\W}\mathcal{L}= -\E_{\X,y}\E_{\bE}^{(R)}\qty(\Z(\diag(\S_y)-\S_y\S_y^\top)\Z^\top \V^\top (\frac{1}{q}\X\Y-\V \X\E_{\bE}^{(R)}[\S_y])\zq^\top)$$
    $$\nabla_{\V}\mathcal{L}= -\E_{\X,y}\qty( \qty(\frac{1}{q}\X\Y-\V \Z\E_{\bE}^{(R)}[\S_y])(\Z\E_{\bE}^{(R)}[\S_y])^\top)$$
\end{proof}

Along the gradient trajectory, if the value matrix $\V$ can be aligned with the ground-truth $\randomVdirection$ and $\W=C\cdot \randomWdirection$, we can have the following nice form for the loss function.
\begin{lemma}\label{Lemma: stochastic pe loss expression}
    Denote $\S_y:=\sm(\Z^\top \W \zq)\in \R^{T}$ for certain $q$-sparse set $y$. Also, we define the $q$-hot vector $\Y=(\mathds{1}\{1\in y\}, \mathds{1}\{2\in y\}, ..., \mathds{1}\{T \in y\})^\top$ for the subset $y\in\binom{[T]}{q}$. If $\V(t)=\alpha(t) \randomVdirection$ and $\W(t)=C(t)\randomWdirection$, the loss function can be represented as the following form:
    $$\mathcal{L}(\bm{\theta}(t))  = \frac{d}{2}\E_y\qty[\left\|\frac{1}{q}\Y-\alpha(t)\E_{\bE}^{(R)}[\S_y]\right\|^2]$$
\end{lemma}
\begin{proof}
    We have the loss function
    $$\mathcal{L}(\bm{\theta}) = \frac{1}{2}\E_{\X}\qty[\|\qsa(\X,y)-f^{(s)}_{\bm{\theta}} (\X,y)\|_2^2]$$
    and the architecture $f_{\bm{\theta}}=\E_{\bE}^{(R)}[\V\Z\sm(\Z^\top \W \zq)]$.
    
    Then we have
    \begin{align*}
        \mathcal{L}(\bm{\theta}(t)) &=\frac{1}{2}\E_{\X,y}\qty[\left\|\frac{1}{q}\X\Y-\V(t)\Z\E_{\bE}^{(R)}[\S_y]\right\|^2]\\
        &=\frac{1}{2}\E_{\X,y}\qty[\left(\frac{1}{q}\X\Y-\alpha(t)\X\E_{\bE}^{(R)}[\S_y]\right)^\top \left(\frac{1}{q}\X\Y-\alpha(t)\X\E_{\bE}^{(R)}[\S_y]\right)]\\
        &=\frac{d}{2}\E_{y}\qty[\left\|\frac{1}{q}\Y-\alpha(t)\E_{\bE}^{(R)}[\S_y]\right\|^2]
    \end{align*}
    The last identity is due to that the expectation of the covariance matrix $\X^\top\X$ is $d\bI_T$, and $\S_y$ is independent of $\X$. The independence is because $\S_y=\sm(\Z^\top\W\zq) = \sm(C(t)\bE^\top \e_y)$ according to the condition that $\W(t)=C(t)\randomWdirection$.
\end{proof}

Then, we will prove that during joint training dynamics, $\V$ can always evolve along the direction of $\randomVdirection$. 
\begin{lemma}
    \label{Lemma: joint training V dynamics}
    If $\V(t) = \alpha(t)\bI_d$, $\W(t)=C(t)\randomWdirection$ for some scalar $\alpha(t), C(t)$, then the gradient for $\V(t)$ is 
    $$\nabla_{\V}\mathcal{L}(t) = \qty(- \E_y\qty[\sum_{i\in y} \E_{\bE}^{(R)}[\S_y(i)] /q] + \alpha\E_y\qty[\left\|\E_{\bE}^{(R)}[\S_y]\right\|^2]) \randomVdirection$$
\end{lemma}

\begin{proof}
    Consider the gradient of $\V$:
    \begin{align*}
        \nabla_{\V}\mathcal{L}= -\E_{\X,y}\qty( \qty(\frac{1}{q}\X\Y-\V \Z\E_{\bE}^{(R)}[\S_y])(\Z\E_{\bE}^{(R)}[\S_y])^\top)
    \end{align*}
    Since $\W(t)=C(t)\randomWdirection$, we know $\S_y$ is independent of the randomness of $\X$.
    
    We first consider the token block in $\V$, which is the first $d$ columns.
    \begin{align*}
        \nabla_{\V}\mathcal{L}_{:,\leq d}&= -\E_{\X,y}\qty( \qty(\frac{1}{q}\X\Y-\V \Z\E_{\bE}^{(R)}[\S_y])(\Z\E_{\bE}^{(R)}[\S_y])_{\leq d}^\top)\\
        &= - \E\qty[\frac{1}{q}\X\Y\qty(\X\E_{\bE}^{(R)}[\S_y])^\top] + \E_{\X,y}\qty( \V \Z\E_{\bE}^{(R)}[\S_y](\X\E_{\bE}^{(R)}[\S_y])^\top)
    \end{align*}            
    We consider the two matrices separately: 
    \begin{align*}
        \E\qty[\frac{1}{q}\X\Y\qty(\X\E_{\bE}^{(R)}[\S_y])^\top]\text{ and }\E_{\X,y}\qty( \V \Z\E_{\bE}^{(R)}[\S_y](\X\E_{\bE}^{(R)}[\S_y])^\top)
    \end{align*}

    For the first matrix, consider the $(n,m)$-entry of the matrix:
     \begin{align*}
         \e_n^\top\E\qty[\frac{1}{q}\X\Y\E_{\bE}^{(R)}[\S_y]^\top \X^\top ]\e_m&= \frac{1}{q}\E \qty[\X_{n:}\Y \E_{\bE}^{(R)}[\S_y]^\top \X_{m:}^\top]\\
        &= \frac{1}{q}\E \qty[\sum_{i\in y} x_{ni} \sum_{i=1}^T \E_{\bE}^{(R)}[\S_y(i)]x_{mi}]
    \end{align*}
    When $m\not=n$, the expectation should be 0 since $x_{ni}$ is sampled from standard Gaussian distribution. When $m=n$, then the expectation is $\E_y\qty[\sum_{i\in y} \E_{\bE}^{(R)}[\S_y(i)] /q]$. Therefore, we have the whole matrix
    $$\E\qty[\frac{1}{q}\X\Y\E_{\bE}^{(R)}[\S_y]^\top \X^\top ] = \E_y\qty[\sum_{i\in y} \E_{\bE}^{(R)}[\S_y(i)] /q]\bI_d$$
    \begin{align*}
        \e_n^\top \E\qty[\V\Z\E_{\bE}^{(R)}[\S_y](\X\E_{\bE}^{(R)}[\S_y])^\top]\e_m &= \alpha\E\qty[\e_n^\top\X\E_{\bE}^{(R)}[\S_y]\E_{\bE}^{(R)}[\S_y]^\top \X^\top \e_m]\\
        &= \alpha\E \qty[\X_{n:}\E_{\bE}^{(R)}[\S_y]\E_{\bE}^{(R)}[\S_y]^\top \X_{m:}^\top]\\
        &= \alpha\delta_{nm}\E \qty[\qty(\sum_{i=1}^T \E_{\bE}^{(R)}[\S_y(i)]x_{mi})^2]\\
        &= \alpha\delta_{nm} \E_y\qty[\left\|\E_{\bE}^{(R)}[\S_y]\right\|^2]
    \end{align*}
    Thus the matrix should be in the following form:
    $$ \E\qty[\V\X\E_{\bE}^{(R)}[\S_y](\X\E_{\bE}^{(R)}[\S_y])^\top] = \alpha\E_y\qty[\left\|\E_{\bE}^{(R)}[\S_y]\right\|^2] \bI_d$$
    So we get the token block of the gradient
    \begin{align*}
        \nabla_{\V}\mathcal{L}_{:,\leq d}&= - \E\qty[\frac{1}{q}\X\Y\qty(\X\E_{\bE}^{(R)}[\S_y])^\top] + \E_{\X,y}\qty( \V \X\E_{\bE}^{(R)}[\S_y](\X\E_{\bE}^{(R)}[\S_y])^\top)\\
        &= \qty(- \E_y\qty[\sum_{i\in y} \E_{\bE}^{(R)}[\S_y(i)] /q] + \alpha\E_y\qty[\left\|\E_{\bE}^{(R)}[\S_y]\right\|^2]) \bI_d
    \end{align*}

    Then, we consider the position block of the gradient:
    \begin{align*}
        \nabla_{\V}\mathcal{L}_{:,d+1:d+d_e}&= -\E_{\X,y}\qty( \qty(\frac{1}{q}\X\Y-\V \Z\E_{\bE}^{(R)}[\S_y])(\Z\E_{\bE}^{(R)}[\S_y])_{d+1:d+d_e}^\top)\\
        &= -\E_{\X,y}\qty( \qty(\frac{1}{q}\X\Y-\alpha \X\E_{\bE}^{(R)}[\S_y])(\bE\E_{\bE}^{(R)}[\S_y])^\top)=\0_{d\times d_e}
    \end{align*}
    Since $\X\sim\mathcal{N}(0,\bI_d)$ and $\S_y$ is independent of $\X$. Combine the two parts, and the proof is completed.
\end{proof}
Finally, along the gradient descent trajectory, if $\V=\alpha(t)\randomVdirection$ and $\W(t) = C(t)\randomWdirection$ for all $t$, we can further simplify the gradient expression. 
\begin{lemma}
\label{Lemma: joint training W dynamics}
    If $\V=\alpha(t)\randomVdirection$ and $\W(t) = C(t)\randomWdirection$ for all $t$ along the training trajectory and $\bE$ is stochastic positional encoding, 
    the position-position block matrix of the gradient of $\W$, i.e. $\nabla_{\W}\mathcal{L}_{>d,>d}$, with input tokens $\X$ is:
\begin{align*}
    \nabla_{\W}\mathcal{L}_{>d,>d}&= -d\alpha(t)\E_{y}\E_{\bE}^{(R)}\qty(\bE(\diag(\S_y)-\S_y\S_y^\top)\qty(\frac{1}{q}\Y-\alpha(t)\E_{\bE}^{(R)}[\S_y])\e_y^\top)\\
\end{align*}
\end{lemma}
\begin{proof}
    By Lemma~\ref{Lemma: basic dynamics for 1-layer tf with stochastic pe, qsa}, we have the gradient dynamics of $\W$ with input token matrix $\X$ is:
\begin{align*}
    \nabla_{\W}\mathcal{L}_{>d,>d}&= -\E_{\X,y}\E_{\bE}^{(R)}\qty(\bE(\diag(\S_y)-\S_y\S_y^\top)\Z^\top \V^\top \qty(\frac{1}{q}\X\Y-\V \Z\E_{\bE}^{(R)}[\S_y])\e_y^\top)\\
    & = -\alpha(t)\E_{\X,y}\E_{\bE}^{(R)}\qty(\bE(\diag(\S_y)-\S_y\S_y^\top)\X^\top \X\qty(\frac{1}{q}\Y-\alpha(t)\E_{\bE}^{(R)}[\S_y])\e_y^\top)\tag{Using $\V = \randomVdirection$}\\
    & = -\alpha(t)\E_{y}\E_{\bE}^{(R)}\qty(\bE(\diag(\S_y)-\S_y\S_y^\top)\E_{x_{ij}\sim\mathcal{N}(0,1)}[\X^\top \X]\qty(\frac{1}{q}\Y-\alpha(t)\E_{\bE}^{(R)}[\S_y])\e_y^\top)
    \tag{Using $\W = \randomWdirection\Rightarrow \S_y$ is independent of $\X$}\\
    & = -d \alpha(t)\E_{y}\E_{\bE}^{(R)}\qty(\bE(\diag(\S_y)-\S_y\S_y^\top)\qty(\frac{1}{q}\Y-\alpha(t)\E_{\bE}^{(R)}[\S_y])\e_y^\top)
    \tag{$\E\qty[\X^\top \X] = d\bI$}
\end{align*}
\end{proof}
In this section, we consider the gradient descent dynamics. The update rules are in the following form
\begin{align*}
    {\W}(t+1) &= \W(t)-\eta\nabla_{\W}\mathcal{L}(t)\\
    {\V}(t+1) &= \V(t)-\eta\nabla_{\V}\mathcal{L}(t)
\end{align*}

\subsection{Joint Training}\label{appen sec: proof detail for joint stochastic pe}
Now, we analyze the dynamics of training the value matrix $\V$ and attention matrix $\W$ simultaneously with the same learning rate $\eta$ with the stochastic positional encoding. 
The following theorem characterizes the convergence of GD when training both layers simultaneously with stochastic positional encoding.

\begin{theorem}[Joint training with stochastic positional encoding]    \label{thm: joint training stochastic pe}
    Suppose $2\leq q<T/4, q,\delta=\Theta(1),d_e=\Theta(q\log T/\delta^2), \delta < 1/10$. For any $\epsilon\in (0,\frac{dT}{100(T-q)q}), \eta\leq \frac{d_e^2}{40d^2T}$, 
    if we apply gradient descent on the population loss in \Cref{eqn: training objective for stochastic pe} with zero initialization $\W(0)=\textbf{0}_{d_e\times d_e},\V(0)=\textbf{0}_{d\times d}$, then after time $t \geq \tilde{O}(\frac{T^{\frac{2-2\delta}{1-3\delta}}}{\eta}+\frac{T^2 d}{\eta\epsilon})$, we have
    $$\mathcal{L}(\bm{\theta}(t))= \frac{1}{2}\E_{\X,y}\qty[\|\qsa(\X,y)-f^{(s)}_{\bm{\theta}}(\X,y)\|_2^2] \leq \epsilon.$$
\end{theorem}

The proof idea is similar to the joint training scenario with one-hot PE. We can still simplify the dynamics of $\W$ and $\V$ using symmetry, proving convergence along the global optimal direction. Then the two variable dynamics are considered inductively so that $\W$ and $\V$ can converge to the global minimum.

We have the following lemma that characterizes the evolution speed along the converging direction of $\W$ and $\V$: they always pointing to the ground-truth direction that $\W(t)=C(t)\randomWdirection,$ $ \V(t)=\alpha(t)\randomVdirection.$
\begin{lemma}[Induction Hypothesis for Stochastic PE, Joint Training] 
\label{Lemma: induction hypothesis for joint training stochastic PE}
Suppose all conditions in \Cref{thm: joint training stochastic pe} holds, then along the gradient descent trajectory, for all $t\geq 0$, there exist some time-dependent scalars $C(t),\alpha(t)$ s.t. 
$\W(t)=C(t)\randomWdirection,$$ \V(t)=\alpha(t)\randomVdirection$, and $C(t),\alpha(t)$ satisfies:
\begin{align*}
    C(t+1)&\geq C(t) + \eta \frac{1-3\delta}{1-2\delta}\cdot\frac{d\alpha(t)}{d_e}\E_{y}\E_{\bE}^{(R)}(1-qs_+)\qty(1-\frac{q\alpha(t)(T\bar{s_+}-1)}{T-q})s_+\\
    C(t+1)&\leq C(t) + \eta \frac{1-\delta}{1-2\delta}\cdot\frac{d\alpha(t)}{d_e}\E_{y}\E_{\bE}^{(R)}(1-qs_+)\qty(1-\frac{q\alpha(t)(T\bar{s_+}-1)}{T-q})s_+\\
    \alpha(t+1)&=\alpha(t) + \eta\bar{s_+}\qty(1-q\alpha \bar{s_+}+\frac{\alpha(1-q\bar{s_+})^2}{(T-q)\bar{s_+}})
\end{align*}

\end{lemma}

\begin{proof}

    We prove $\W(t)=C(t)\randomWdirection,$$ \V(t)=\alpha(t)\randomVdirection.$ by induction. First, check the initialization $\W(t)=0=0\cdot\randomWdirection,\V(0)=0=0\cdot\randomVdirection$. Here the scalar $C(0)=\alpha(0)=0$. For GD dynamics, by \Cref{Lemma: joint training V dynamics} and \Cref{Lemma: joint training W dynamics} we have $C(1)=C(0),\alpha(1)=\alpha(0)+\eta /T$. Therefore, the induction hypothesis holds for $t=0$.

    Then we prove this argument inductively: if \Cref{Lemma: induction hypothesis for joint training stochastic PE} holds for iteration $t$, it is enough to prove that for iteration $t+1$, this argument still holds. Then by induction, we can conclude that it holds for all time $t\geq 0$.

    Note that since the position-position block (right-bottom block) of $\W$ is always in direction of $\randomWdirection$ and other entries are 0, $\S_y = \sm(\Z^\top \W \zq)=\sm(C(t)\E^\top \e_y)$ when $y$ is given. Therefore, the softmax vector is always independent of $\X$ for iterations $\leq t$.
    
    Now we suppose these two properties hold for iteration $t$. We first consider \textbf{the token-token, token-position and position token} submatrices of $\W(t)$. We prove that those gradient blocks should always be $\0$.

    For the first $d$ rows, we have (here within $\zq$, $\x_{\text{query}}$ can be any fixed token vector):
    \begin{align*}
        \nabla_{\W}\mathcal{L}(t)_{\leq d, :} &= \alpha \E_{\X, y}\E_{\bE}^{(R)} \qty [ \X\qty(\diag(\S_y) - \S_y\S_y^\top) \X^\top \X \qty(\frac{1}{q}\Y - \alpha \S_y) {\zq}^\top ]\\
        &= \frac{1}{2}\alpha \E_{\X, y}\E_{\bE}^{(R)} \qty [ \X\qty(\diag(\S_y) - \S_y\S_y^\top) \X^\top \X \qty(\frac{1}{q}\Y - \alpha \S_y) {\zq}^\top ]\\
        &+\frac{1}{2}\alpha \E_{\X, y}\E_{\bE}^{(R)} \qty [ -\X\qty(\diag(\S_y) - \S_y\S_y^\top) (-\X^\top) (-\X) \qty(\frac{1}{q}\Y - \alpha \S_y) {\zq}^\top ]\tag{By symmetry and independence between $\X$ and $\S_y$}\\
        &=\0_{d\times (d+d_e)}.
    \end{align*}
    Then we consider the position-token block $\nabla_{\W}\mathcal{L}(t)_{>d, \leq d}$:
    \begin{align*}
        \nabla_{\W}\mathcal{L}(t)_{> d, \leq d} &= \alpha \E_{\X, y}\E_{\bE}^{(R)} \qty [ \bE\qty(\diag(\S_y) - \S_y\S_y^\top) \X^\top \X \qty(\frac{1}{q}\Y - \alpha \S_y) {\x_{\text{query}}}^\top ]\\
        &=d\alpha \E_{\X, y}\E_{\bE}^{(R)} \qty [ \bE\qty(\diag(\S_y) - \S_y\S_y^\top) \qty(\frac{1}{q}\Y - \alpha \S_y) {\x_{\text{query}}}^\top ]\tag{Independence between $\X$ and $\S_y$}\\
        &=d\alpha \E_{\X, y}\E_{\bE}^{(R)} \qty [ \bE\qty(\diag(\S_y) - \S_y\S_y^\top) \qty(\frac{1}{q}\Y - \alpha \S_y) {\x_{\text{query}}}^\top ]\\&+d\alpha \E_{\X, y}\E_{\bE}^{(R)} \qty [ -\bE\qty(\diag(\S_y) - \S_y\S_y^\top) \qty(\frac{1}{q}\Y - \alpha \S_y) {\x_{\text{query}}}^\top ]\tag{For $\bE$ and $-\bE$, $\S_y$ are the same.}=\0_{d_e\times d}.
    \end{align*}
    
    Finally, we consider the update for \textbf{the position-position submatrix} $\W(t)_{>d,>d}$ in its $(k,j)$-entry $\W_{kj}(t)$, $k,j\in[d_e]$. With the update rule by \Cref{Lemma: joint training W dynamics}, 
    \begin{align*}
        \Delta{\W}(t)_{>d,>d} &= \eta d\alpha(t)\E_{y}\E_{\bE}^{(R)}\qty(\bE(\diag(\S_y)-\S_y\S_y^\top)\qty(\frac{1}{q}\Y-\alpha(t)\E_{\bE}^{(R)}[\S_y])\e_y^\top)
    \end{align*}
 
    Since $\W(t)_{>d,>d}=C(t)\bI_{d_e}$, we only need to prove that the gradient term is along the direction of $\bI_{d_e}$. Recall the notation $\bar{\S_y} = \E_{\bE}^{(R)}[\S_y]$. To prove this statement, we first expand the gradient term:
     \begin{align*}
        \Delta{\W}(t)_{>d,>d} &= \eta d\alpha(t)\E_{y}\E_{\bE}^{(R)}\qty(\bE(\diag(\S_y)-\S_y\S_y^\top)\qty(\frac{1}{q}\Y-\alpha(t)\E_{\bE}^{(R)}[\S_y])\e_y^\top)\\
        & = \eta d\alpha(t)\E_{y}\E_{\bE}^{(R)}\qty(\sum_{i=1}^T\e_i\qty[(\diag(\S_y)-\S_y\S_y^\top)\qty(\frac{1}{q}\Y-\alpha(t)\bar{\S_y})]_i\e_y^\top)\\
        & = \eta d\alpha(t)\E_{y}\E_{\bE}^{(R)}\qty(\sum_{i\in y}\e_i\qty[(\diag(\S_y)-\S_y\S_y^\top)\qty(\frac{1}{q}\Y-\alpha(t)\bar{\S_y})]_i\e_y^\top)\\
        &+\eta d\alpha(t)\E_{y}\E_{\bE}^{(R)}\qty(\sum_{i\not \in y}\e_i\qty[(\diag(\S_y)-\S_y\S_y^\top)\qty(\frac{1}{q}\Y-\alpha(t)\bar{\S_y})]_i\e_y^\top)\\
        & = \eta d\alpha(t)\E_{y}\E_{\bE}^{(R)}\qty(\sum_{i\in y}\e_i\qty[\S_y(i)\qty(\frac{1}{q}-\alpha(t)\bar{\S_y}(i))+\S_y(i)\qty[\alpha(t)\S_y^\top \bar{\S_y} - \sum_{i\in y}\S_y(i)/q]]\e_y^\top)\tag{Term 1'}\\
        &\ +\eta d\alpha(t)\E_{y}\E_{\bE}^{(R)}\qty(\sum_{i\not \in y}\e_i\qty[-\alpha(t)\S_y(i)\bar{\S_y}(i)+\alpha(t)\S_y(i)(\S_y^\top \bar{\S_y})-\frac{1}{q}\S_y(i)\sum_{j\in y}\S_y(j)]\e_y^\top)\tag{Term 2'}
    \end{align*}
    
    By \Cref{Lemma: Conditional expectation of the softmax vector}, we know for $\W(t)=C(t)\bI_{d_e}$, there exists $\bar{s_+}(t), \bar{s_-}(t)$ s.t.
    $$\bar{\S_y} = \bar{s_+}\Y + \bar{s_-}(\mathbf{1}_T-\Y).$$

    By \Cref{Lemma: Off-diagonal entries of the gradient have 0 expectation}, any off-diagonal entry of both terms in the gradient should be 0 according to symmetry. Now we only need to consider the diagonal entries. By \Cref{Lemma: Diagonal entries of the gradient are the same}, all the diagonal entries of each term in the gradient have the same value. Therefore, the gradient of the submatrix is aligned with $\bI_{d_e}$. Combine the four blocks, the whole gradient of $\W$ is aligned with $\randomWdirection$.

    For $\V(t)$, by \Cref{Lemma: joint training V dynamics} we know its update is always in the $\randomVdirection$ direction. Therefore, $$\V(t+1)=\V(t) + \qty( \E_y\qty[\sum_{i\in y} \E_{\bE}^{(R)}[\S_y(i)] /q] - \alpha\E_y\qty[\left\|\E_{\bE}^{(R)}[\S_y]\right\|^2]) \randomVdirection$$
    aligns with $\randomVdirection$. By induction, we complete the proof for the direction property.
    
    After proving the direction property, we calculate $C(t)$'s dynamics by considering two terms separately. Now we only consider the position-position blocks' dynamics. Since the symmetry among all the diagonal entries, the trace of each term is considered and each diagonal entry will be $1/d_e$ of the conditional expectation of the trace.
    \paragraph{Term 1:} Terms with $i \in y$.
    \begin{align*}
        &\eta d\alpha(t)\tr\qty(\E_{y}\E_{\bE}^{(R)}\qty(\sum_{i\in y}\e_i\qty[\S_y(i)\qty(\frac{1}{q}-\alpha(t)\bar{\S_y}(i))+\S_y(i)\qty[\alpha(t)\S_y^\top \bar{\S_y} - \sum_{i\in y}\S_y(i)/q]]\e_y^\top))\\
        ={}&\eta d\alpha(t)\E_{y}\E_{\bE}^{(R)}\qty(\sum_{i\in y}\qty[\S_y(i)\qty(\frac{1}{q}-\alpha(t)\bar{\S_y}(i))+\S_y(i)\qty[\alpha(t)\S_y^\top \bar{\S_y} - \sum_{i\in y}\S_y(i)/q]]\tr\qty(\e_i\e_y^\top))\tag{Linearity of trace and expectation}\\
        ={}&\eta d\alpha(t)\E_{y}\E_{\bE}^{(R)}\qty(\sum_{i\in y}\qty[\S_y(i)\qty(\frac{1}{q}-\alpha(t)\bar{\S_y}(i))+\S_y(i)\qty[\alpha(t)\S_y^\top \bar{\S_y} - \sum_{i\in y}\S_y(i)/q]]\e_i^\top\e_y)\\
        ={}&\eta d\alpha(t)\E_{y}\E_{\bE}^{(R)}\qty(\sum_{i\in y}\qty[\S_y(i)\qty(\frac{1}{q}-\alpha(t)\bar{\S_y}(i))+\S_y(i)\qty[\alpha(t)\S_y^\top \bar{\S_y} - \sum_{i\in y}\S_y(i)/q]])\tag{$\e_i^\top\e_y=1$ for $i\in y$}\\
        ={}&\eta dq\alpha(t)\E_{y}\E_{\bE}^{(R)}\qty(\qty[\S_y(i)\qty(\frac{1}{q}-\alpha(t)\bar{\S_y}(i))+\S_y(i)\qty[\alpha(t)\S_y^\top \bar{\S_y} - \sum_{i\in y}\S_y(i)/q]])\tag{Symmetry among $i\in y$}
    \end{align*}
    Expand $\S_y^\top \bar{\S_y}$ and we have:
    \begin{align*}
        \S_y^\top \bar{\S_y} &= \sum_{i\in y}\S_y(i)\bar{s_+}+\sum_{i\not\in y}\S_y(i)\bar{s_-}
    \end{align*}
    Plug it back into Term 1:
    \begin{align*}
        &\eta dq\alpha(t)\E_{y}\E_{\bE}^{(R)}\qty(\qty[\S_y(i)\qty(\frac{1}{q}-\alpha(t)\bar{\S_y}(i))+\S_y(i)\qty[\alpha(t)\S_y^\top \bar{\S_y} - \sum_{i\in y}\S_y(i)/q]])\\
        ={}&\eta dq\alpha(t)\E_{y}\E_{\bE}^{(R)}\qty(\qty[\S_y(i)\qty(\frac{1}{q}-\alpha(t)\bar{s_+})+\S_y(i)\qty(\sum_{i\in y}\S_y(i)\qty(\alpha(t)\bar{s_+}-\frac{1}{q})+\alpha(t)\sum_{i\not\in y}\S_y(i)\bar{s_-})])\\
        ={}&\eta dq\E_{y}\E_{\bE}^{(R)}\qty(\qty[\S_y(i)(1-q\S_y(i))\qty(\frac{1}{q}-\alpha(t)\bar{s_+})+\alpha(t)\S_y(i)\qty(\sum_{i\not\in y}\S_y(i)\bar{s_-})]) \tag{$\S_y(i)=s_+$ are equal for $i\in y$}\\
        ={}&\eta dq\E_{y}\E_{\bE}^{(R)}\qty(\qty[\S_y(i)(1-q\S_y(i))\qty(\frac{1}{q}-\alpha(t)\bar{s_+})+\alpha(t)\S_y(i)\qty(\sum_{i\not\in y}\S_y(i)\frac{(1-q\bar{s_+})}{T-q})]) \tag{$q\bar{s_+}+(T-q)s_+=1$}\\
        ={}&\eta dq\E_{y}\E_{\bE}^{(R)}\qty(\qty[s_+(1-qs_+)\qty(\frac{1}{q}-\alpha(t)\bar{s_+})+\alpha(t)s_+(1-qs_+)\frac{(1-q\bar{s_+})}{T-q}]) \tag{$qs_++\sum_{j\not\in y}\S_y(j)=1$}\\
        ={}&\eta dq\E_{y}\E_{\bE}^{(R)}\qty(s_+(1-qs_+)\qty(\frac{1}{q}-\alpha(t)\frac{T\bar{s_+}-1}{T-q}))
    \end{align*}
    Therefore, 
    \begin{align*}
        \text{Term 1}=\frac{\eta d}{d_e}\E_{y}\E_{\bE}^{(R)}\qty(s_+(1-qs_+)\qty(1-q\alpha(t)\frac{T\bar{s_+}-1}{T-q}))\bI_{d_e}
    \end{align*}
   \paragraph{Term 2:} Terms with $i \not\in y$. We use a similar technique from above.
   \begin{align*}
        &\eta d\alpha(t)\E_{y}\E_{\bE}^{(R)}\qty(\sum_{i\not \in y}\e_i\qty[-\alpha(t)\S_y(i)\bar{\S_y}(i)+\alpha(t)\S_y(i)(\S_y^\top \bar{\S_y})-\frac{1}{q}\S_y(i)\sum_{j\in y}\S_y(j)]\e_y^\top)\\
        ={}&\eta d\alpha(t)\E_{y}\E_{\bE}^{(R)}\sum_{i\not\in y}\qty(-\qty[\S_y(i)\qty(\alpha(t)\bar{s_-}-\alpha(t)\sum_{i\in y}\S_y(i)\bar{s_+}-\alpha(t)\sum_{i\not\in y}\S_y(i)\bar{s_-}+\frac{1}{q}\sum_{j\in y}\S_y(j))]\e_i^\top \e_y)\\
        ={}&\eta d\alpha(t)\E_{y}\E_{\bE}^{(R)}\sum_{i\not\in y}\qty(-\qty[\S_y(i)\qty(\alpha(t)\frac{(1-q\bar{s_+})}{T-q}-\alpha(t)qs_+\bar{s_+}-\alpha(t)(1-qs_+)\frac{(1-q\bar{s_+})}{T-q}+s_+)]\e_i^\top \e_y)\\
        ={}&\eta d\alpha(t)\E_{y}\E_{\bE}^{(R)}\sum_{i\not\in y}\qty(-\qty[\S_y(i)s_+\qty(1-\frac{q\alpha(t)(T\bar{s_+}-1)}{T-q})]\e_i^\top \e_y)\\
        \geq{}& -\eta \alpha(t)d\cdot\frac{\delta}{1-2\delta}\E_{y}\E_{\bE}^{(R)}\sum_{i\not\in y}\qty[\S_y(i)s_+\qty(1-\frac{q\alpha(t)(T\bar{s_+}-1)}{T-q})]\\
        ={}& -\eta d\alpha(t)\frac{\delta}{1-2\delta}\E_{y}\E_{\bE}^{(R)}(1-qs_+)\qty(1-\frac{q\alpha(t)(T\bar{s_+}-1)}{T-q})s_+
    \end{align*}
    The last line is due to $\sum_{i\not\in y}\S_y(i)=1-\sum_{i\in y}\S_y(i)=1-qs_+.$
    Thus we have
    $$\text{Term 2}\succeq -\frac{\eta d\alpha(t)\delta}{(1-2\delta)d_e}\E_{y}\E_{\bE}^{(R)}(1-qs_+)\qty(1-\frac{q\alpha(t)(T\bar{s_+}-1)}{T-q})s_+ \bI_{d_e}$$
    Similarly, we can also have the upper bound for Term 2.
    $$\text{Term 2}\preceq \frac{\eta d\alpha(t)\delta}{(1-2\delta)d_e}\E_{y}\E_{\bE}^{(R)}(1-qs_+)\qty(1-\frac{q\alpha(t)(T\bar{s_+}-1)}{T-q})s_+ \bI_{d_e}$$
    Combine two terms, we have the lower bound for $C(t)$'s dynamics:
    $$C(t+1)\geq C(t) + \eta \frac{1-3\delta}{1-2\delta}\cdot\frac{d\alpha(t)}{d_e}\E_{y}\E_{\bE}^{(R)}(1-qs_+)\qty(1-\frac{q\alpha(t)(T\bar{s_+}-1)}{T-q})s_+$$
    $$C(t+1)\leq C(t) + \eta \frac{1-\delta}{1-2\delta}\cdot\frac{d\alpha(t)}{d_e}\E_{y}\E_{\bE}^{(R)}(1-qs_+)\qty(1-\frac{q\alpha(t)(T\bar{s_+}-1)}{T-q})s_+$$
    By induction, we complete the proof.
\end{proof}

\paragraph{Remark.} Similar to the one-hot case, after proving the induction lemma about the evolving direction of $\W(t)$ and $\V(t)$, the optimization problem can be reduced to analyzing the two variable dynamics of $C(t)$ and $\alpha(t)$:
\begin{align*}
    C(t+1)&\approx C(t) + \frac{\eta d\alpha(t)}{d_e}\E_{y}\E_{\bE}^{(R)}(1-qs_+)\qty(1-\frac{q\alpha(t)(T\bar{s_+}-1)}{T-q})s_+\\
    \alpha(t+1)&=\alpha(t) + \eta\bar{s_+}\qty(1-q\alpha \bar{s_+}+\frac{\alpha(1-q\bar{s_+})^2}{(T-q)\bar{s_+}})
\end{align*}
Then the rest of the proof is to analyze the dynamics of $C(t)$ and $\alpha(t)$ and prove that $C(t)\to +\infty$ and $\alpha(t)\to 1$ eventually, and calculate the convergence time by estimating the trajectory of the two variable dynamical systems. One can refer to the main paper for proof ideas.
\begin{proof}[Proof of \Cref{thm: joint training stochastic pe}]
    After \Cref{Lemma: induction hypothesis for joint training stochastic PE} shows that $\V$ and $\W$ are always along the ground-truth direction: $\V(t) = \alpha(t)\randomVdirection,\W(t)=C(t)\randomWdirection$, the dynamics of the parameter matrices then can be characterized by two scalar variable $\alpha(t)$ and $C(t)$. Our update rules becomes 
\begin{align*}
    C(t+1)&\geq C(t) + \eta \frac{1-3\delta}{1-2\delta}\cdot\frac{d\alpha(t)}{d_e}\E_{y}\E_{\bE}^{(R)}(1-qs_+)\qty(1-\frac{q\alpha(t)(T\bar{s_+}-1)}{T-q})s_+\\
    C(t+1)&\leq C(t) + \eta \frac{1-\delta}{1-2\delta}\cdot\frac{d\alpha(t)}{d_e}\E_{y}\E_{\bE}^{(R)}(1-qs_+)\qty(1-\frac{q\alpha(t)(T\bar{s_+}-1)}{T-q})s_+\\
    \alpha(t+1)&=\alpha(t) + \eta\bar{s_+}\qty(1-q\alpha \bar{s_+}+\frac{\alpha(1-q\bar{s_+})^2}{(T-q)\bar{s_+}})
\end{align*}
    Along the trajectory, $s_+\leq \frac{1}{q}$ by its definition: since $\W$ is along $\randomWdirection$, all $i\in y$ has the same softmax probability $\S_y^{(t)}(i)$, and thus they cannot exceed the upper bound $1/q$. 
    
    Note that when $s_+$ is fixed, $\alpha(t)$ has a stationary point $\alpha^*(t) = \frac{(T-q)\bar{s_+}}{Tq\bar{s_+}^2-2q\bar{s_+}+1}$. With some calculation, we can rewrite the $\alpha$ dynamics into:
    $$\alpha(t+1) = \alpha(t) +\eta\bar{ s_+}\qty(1-\frac{\alpha(t)}{\alpha^*(t)}).$$
    Observe that when $\bar{s_+}\in(\frac{1}{T},\frac{1}{q}),\alpha^*(t)\geq 1,$ and it achieves maximum at $\bar{s_+}=\frac{1}{\sqrt{Tq}}$.

To characterize this limit above when $t\to \infty$, we need to prove the following two arguments:
    \begin{enumerate}
        \item $C(t)$ is non-decreasing for all $t\geq0$.
        \item $\alpha(t)$ goes through 2 phases: 
        
        \textbf{Phase I.} $\alpha$ monotonically grows to $1-0.1\sqrt{\frac{q(T-q)\epsilon}{dT}}$ at some time $t_1$.
        
        \textbf{Phase II.} $\alpha$ stays within an interval whose upper bound is close to $\alpha^*$ after time $t_1$: 
        \begin{equation*}
            \alpha(t)\in\qty[1-0.1\sqrt{\frac{q(T-q)\epsilon}{dT}}, \qty(1+\frac{(8Td-d_e)(1-q\bar{s_+})}{8Tdq\bar{s_+}(T\bar{s_+}-1)})\alpha^*(t)]\tag{IH2}
            \label{IH: alpha stays near alpha star with stochastic pe}
        \end{equation*}
    \end{enumerate}

    \textbf{Phase I.}
     In this phase, we inductively prove that both $\alpha(t), C(t)$ are non-decreasing. For $t=0$, by zero initialization, $C(0)=C(1)=0,\alpha(1)=\frac{\eta}{T}>0$. 

Suppose it holds for some $t<t_1$ before $\alpha$ reaches $1-0.1\sqrt{\frac{q(T-q)\epsilon}{dT}}$. Then for $\alpha(t+1)$, we have the update rule:
     \begin{align*}
         \alpha(t+1)&=\alpha(t) +\eta \bar{s_+}\qty(1-\frac{\alpha(t)}{\alpha^*(t)})\\
         &\geq \alpha(t) + \eta \frac{1}{T}(1-\frac{\alpha(t)}{\alpha^*(t)})\tag{Induction Hyp. $\bar{s_+}\geq \frac{1}{T}$}\\
        &\geq \alpha(t) + \eta \frac{1}{T}(1-{\alpha(t)})\tag{$\alpha^*\geq 1$}\\
        &\geq \alpha(t)+0.1\eta\sqrt{\frac{q(T-q)\epsilon}{dT^3}}.\tag{$\alpha\leq 1-0.1\sqrt{\frac{q(T-q)\epsilon}{dT}}$}
     \end{align*}
        So $\alpha$ is non-decreasing. Meanwhile, for $C(t+1)$:
     \begin{align*}
         C(t+1)&\geq C(t) + \eta \frac{1-3\delta}{1-2\delta}\cdot\frac{d\alpha(t)}{d_e}\E_{y}\E_{\bE}^{(R)}(1-qs_+)\qty(1-\frac{q\alpha(t)(T\bar{s_+}-1)}{T-q})s_+\\
         &\geq C(t) + \frac{7\eta}{8}\cdot\frac{\alpha d}{d_e}\E_{y}\E_{\bE}^{(R)}(1-qs_+)\qty(1-\frac{q(T\bar{s_+}-1)}{T-q})s_+\tag{$\alpha<1,(1-Ts_+)<0$}\\
         &=C(t) + \frac{7\eta}{8}\cdot\frac{\alpha d}{d_e}\E_{y}\E_{\bE}^{(R)}(1-qs_+)\frac{T-q+q-Tq\bar{s_+}}{T-q}s_+\\
         &=C(t)+ \frac{7\eta}{8}\cdot\frac{\alpha dT}{d_e(T-q)}\E_{y}\E_{\bE}^{(R)}(1-qs_+)(1-q\bar{s_+})s_+\geq C(t)
     \end{align*}
     so they are both non-decreasing. Then we need to upper bound the time $t_1$ for $\alpha(t)$ to reach $1-0.1\sqrt{\frac{q(T-q)\epsilon}{dT}}$: by the update above we have
     \begin{align*}
         \alpha(t+1)&=\alpha(t) +\eta \bar{s_+}\qty(1-\frac{\alpha(t)}{\alpha^*(t)})\\
         &\geq \alpha(t) + \eta \frac{1}{T}(1-\frac{\alpha(t)}{\alpha^*(t)})\tag{Induction Hyp. $\bar{s_+}\geq \frac{1}{T}$}\\
        &\geq \alpha(t) + \eta \frac{1}{T}(1-{\alpha(t)})\tag{$\alpha^*\geq 1$}\\
        \Rightarrow 1-\alpha(t+1)&\leq (1-\eta/T)(1-\alpha(t))\leq ...\leq (1-\eta/T)^{t}(1-\alpha(0)).
     \end{align*}
     Thus for $\alpha(t+1)\geq 1-0.1\sqrt{\frac{q(T-q)\epsilon}{dT}}$, it takes at most $O(\frac{T\log \frac{d}{\epsilon}}{\eta})$ iterations.

     \textbf{Phase II.} In this phase, we first consider the easiest case: $\bar{s_+}(t)< \frac{1}{\sqrt{Tq}}$. Under this condition, $\alpha^*(t+1)>\alpha^*(t)$, and using induction hypothesis 
     $$\alpha^*(t)-\alpha(t+1)=\alpha^*(t)-\alpha(t)-\frac{\eta \bar{s_+}(t)}{\alpha^*(t)}(\alpha^*(t)-\alpha(t)) = (1-\frac{\eta \bar{s_+}(t)}{\alpha^*(t)})(\alpha^*(t)-\alpha(t)) > 0,$$
     $\alpha(t+1)<\alpha^*(t+1)$ always holds, and the induction hypothesis holds for $t+1$.

    Then we only need to consider the case with $\bar{s_+}(t)\geq \frac{1}{\sqrt{Tq}}$. First, we check that within the induction hypothesis range, $C(t+1)\geq C(t)$.
    We know the following term in the gradient update of $C(t)$ should be greater than 0:
    \begin{align*}
        \qty(1-\frac{q(T\bar{s_+}-1)}{T-q}) 
        \geq{}&1-\qty(1+\frac{(8Td-d_e)(1-q\bar{s_+})}{8Tdq\bar{s_+}(T\bar{s_+}-1)})\alpha^*(t)\frac{q}{T-q}(T\bar{s_+}-1)\tag{$\alpha\leq \qty(1+\frac{(d-1)(1-qs_+)}{dqs_+(Ts_+-1)})\alpha^*(t),(1-Ts_+)<0$}\\
        \geq{}&\frac{d_e}{8Td}\frac{1-q\bar{s_+}}{Tq\bar{s_+}^2-2qs_++1}\geq 0.\tag{$\bar{s_+}\leq \frac{1}{q}$.}
    \end{align*}
    Therefore by the update rule:
    \begin{align*}
        C(t+1)&\geq C(t) + \eta \frac{1-3\delta}{1-2\delta}\cdot\frac{d\alpha(t)}{d_e}\E_{y}\E_{\bE}^{(R)}(1-qs_+)\qty(1-\frac{q\alpha(t)(T\bar{s_+}-1)}{T-q})s_+\\
         &\geq C(t) + \eta \frac{7d\alpha(t)}{8d_e}\E_{y}\E_{\bE}^{(R)}(1-qs_+)s_+\cdot \frac{d_e(1-q\bar{s_+})}{8Td(Tq\bar{s_+}^2-2q\bar{s_+}+1)}\geq C(t)
     \end{align*}
     Next, we first divide $\alpha(t)$'s possible range into two parts: $\alpha(t)\leq \qty(1+\frac{(4Td-d_e)(1-q\bar{s_+}(t))}{4dTq\bar{s_+}(t)(T\bar{s_+}(t)-1)})\alpha^*(t)$ and $\alpha(t)\in \qty[\qty(1+\frac{(4dT-d_e)(1-q\bar{s_+}(t))}{4dTq\bar{s_+}(t)(T\bar{s_+}(t)-1)})\alpha^*(t), \qty(1+\frac{(8dT-d_e)(1-q\bar{s_+}(t))}{8dTq\bar{s_+}(t)(T\bar{s_+}(t)-1)})\alpha^*(t)]$. For the first part, we prove the following statement \textbf{(S3)}:
     \begin{center}
         If $\alpha(t)\leq \qty(1+\frac{(4dT-d_e)(1-q\bar{s_+}(t))}{4dTq\bar{s_+}(t)(T\bar{s_+}(t)-1)})\alpha^*(t)$, the next step $$\alpha(t+1)\leq \qty(1+\frac{(8dT-d_e)(1-q\bar{s_+}(t+1))}{8dTq\bar{s_+}(t+1)(T\bar{s_+}(t+1)-1)})\alpha^*(t+1).$$
     \end{center}
    If \textbf{(S3)} is true, then we know once $\alpha(t)\leq \qty(1+\frac{(4dT-d_e)(1-q\bar{s_+}(t))}{4dTq\bar{s_+}(t)(T\bar{s_+}(t)-1)})\alpha^*(t)$, $\alpha(t+1)$ satisfy the induction hypothesis. After proving \textbf{(S3)}, the only part left is when $$\alpha(t)\in \left( \qty(1+\frac{(4dT-d_e)(1-q\bar{s_+}(t))}{4dTq\bar{s_+}(t)(T\bar{s_+}(t)-1)})\alpha^*(t),  \qty(1+\frac{(8dT-d_e)(1-q\bar{s_+}(t))}{8dTq\bar{s_+}(t)(T\bar{s_+}(t)-1)})\alpha^*(t) \right].$$
     
     We prove statement \textbf{(S3)} by proving 
     \begin{equation}
         \qty(1+\frac{(8dT-d_e)(1-q\bar{s_+}(t+1))}{8dTq\bar{s_+}(t+1)(T\bar{s_+}(t+1)-1)})\alpha^*(t+1)\geq \qty(1+\frac{(4dT-d_e)(1-q\bar{s_+}(t))}{4dTq\bar{s_+}(t)(T\bar{s_+}(t)-1)})\alpha^*(t)\tag{\textbf{S4}}
     \end{equation}
     When the inequality \textbf{(S4)} above is proved, then \textbf{(S3)} is proved. This is because: if $\alpha(t)<\alpha^*(t)$, then by update rule we have $$\alpha^*(t)-\alpha(t+1)=\alpha^*(t)-\alpha(t)-\frac{\eta \bar{s_+}(t)}{\alpha^*(t)}(\alpha^*(t)-\alpha(t)) = (1-\frac{\eta \bar{s_+}(t)}{\alpha^*(t)})(\alpha^*(t)-\alpha(t)) > 0,$$
     \begin{align*}
         \alpha(t+1) < \alpha^*(t) &\leq \qty(1+\frac{(4dT-d_e)(1-q\bar{s_+}(t))}{4dqT\bar{s_+}(t)(T\bar{s_+}(t)-1)})\alpha^*(t) \\&\leq \qty(1+\frac{(8dT-d_e)(1-q\bar{s_+}(t+1))}{8dqT\bar{s_+}(t+1)(T\bar{s_+}(t+1)-1)})\alpha^*(t+1)
     \end{align*}
     If $\alpha(t)\geq \alpha^*(t),$ then $\alpha(t+1)\leq \alpha(t)\leq \qty(1+\frac{(4dT-d_e)(1-q\bar{s_+}(t))}{4dqT\bar{s_+}(t)(T\bar{s_+}(t)-1)})\alpha^*(t)$, and therefore smaller than $\qty(1+\frac{(8dT-d_e)(1-q\bar{s_+}(t+1))}{8dqT\bar{s_+}(t+1)(T\bar{s_+}(t+1)-1)})\alpha^*(t+1)$. 
     
     Now we prove \textbf{(S4)} by expanding the $\bar{s_+}(t+1)$ using the update rule of $C(t)$. Denote $\Delta C(t):=C(t+1)-C(t).$ Since $\eta\leq \frac{d_e^2}{40d^2T}$, $\Delta C(t)<\frac{1}{5}$. Then we have
     \begin{align*}
         \bar{s_+}(t+1)&=\E_{y}\E_{\bE}^{(R)}\frac{1}{q+\sum_{i\not\in y}e^{-(C(t)+\Delta C(t))(1-\e_i^\top \e_y)}}\\
         &\leq \E_{y}\E_{\bE}^{(R)}\frac{1}{q+\sum_{i\not\in y}e^{-C(t)(1-\e_i^\top \e_y)}e^{\Delta C(t)(\e_i^\top \e_y-1)}}\\
         &\leq \E_{y}\E_{\bE}^{(R)}\frac{1}{q+\sum_{i\not\in y}e^{-C(t)(1-\e_i^\top \e_y)}(1-\Delta C(t)(1-\e_i^\top \e_y))}\\
         &\leq \E_{y}\E_{\bE}^{(R)}\left(\frac{1}{q+\sum_{i\not\in y}e^{-C(t)(1-\e_i^\top \e_y)}}\right.\\ 
         & \quad \quad \quad\left.+\frac{5}{4}\qty(\frac{1}{q+\sum_{i\not\in y}e^{-C(t)(1-\e_i^\top \e_y)}})^2\sum_{i\not\in y}e^{C(t)(\e_i^\top \e_y-1)}(1-\e_i^\top \e_y)\Delta C(t)\right)\tag{$\sum_{i\not\in y}e^{C(t)(\e_i^\top \e_y-1)}(1-\e_i^\top \e_y)\Delta C(t)<\frac{1}{5}$, due to $\eta\leq \frac{d_e^2}{40d^2T}$}\\
         &\leq \bar{s_+} + \frac{5}{4}\frac{1-\delta}{1-2\delta}\E_{y}\E_{\bE}^{(R)}s_+^2\qty(\frac{1}{s_+}-q)\Delta C(t)\tag{$\sum_{i\not\in y}e^{C(t)(\e_i^\top \e_y-1)}=\frac{1}{s_+}-q$}\\
         &=\bar{s_+} + \frac{45}{32}\E_{y}\E_{\bE}^{(R)}s_+\qty({1}-qs_+)\Delta C(t)\tag{$\delta<0.1$} \\
         &\leq \bar{s_+} + \frac{45}{32}\bar{s_+}\qty({1}-q\bar{s_+})\Delta C(t)\tag{$\E[{s_+}^2]\geq \E[{s_+}]^2$}
    \end{align*}

     Then we lower bound the decrement of $\alpha^*(t+1)$ and $\frac{(1-q\bar{s_+}(t+1))}{q\bar{s_+}(t+1)(T\bar{s_+}(t+1)-1)}$, respectively. Denote $\Delta s = \bar{s_+}(t+1)-\bar{s_+}(t)$.
     \begin{align*}
         \alpha^*(t+1)&=\frac{(T-q)\bar{s_+}(t+1)}{Tq\bar{s_+}(t+1)^2-2q\bar{s_+}(t+1)+1}\\
         &=\frac{(T-q)}{Tq\bar{s_+}(t+1)-2q+1/\bar{s_+}(t+1)}\\
         &\geq\frac{(T-q)}{Tq\bar{s_+}(t+1)-2q+1/\bar{s_+}(t+1)}\\
         &\geq \frac{(T-q)}{Tq\bar{s_+}(t)-2q+1/\bar{s_+}(t) + Tq\Delta s}\\
         &\geq \frac{(T-q)}{Tq\bar{s_+}(t)-2q+1/\bar{s_+}(t)} - \frac{(T-q)Tq\Delta s}{(Tq\bar{s_+}(t)-2q+1/\bar{s_+}(t))^2}\\
         &=\alpha^*(t)-\frac{(T-q)Tq\Delta s}{(Tq\bar{s_+}(t)-2q+1/\bar{s_+}(t))^2}\\
         &=\alpha^*(t)-\frac{Tq\bar{s_+}\Delta s}{(Tq\bar{s_+}^2(t)-2q\bar{s_+}(t)+1)}\alpha^*(t)
     \end{align*}
     \begin{align*}
         \frac{(1-q\bar{s_+}(t+1))}{q\bar{s_+}(t+1)(T\bar{s_+}(t+1)-1)}={}&\frac{1}{q\bar{s_+}(t+1)(T\bar{s_+}(t+1)-1)} - \frac{1}{(T\bar{s_+}(t+1)-1)}\\
         \geq{}& \frac{1}{q\bar{s_+}(t+1)(T\bar{s_+}(t+1)-1)} - \frac{1}{(T\bar{s_+}(t)-1)}\tag{$\bar{s_+}(t+1)\geq \bar{s_+}(t)$}\\
         ={}&\frac{1}{qT}\cdot \frac{1}{\bar{s_+}(t+1)}\cdot\frac{1}{\bar{s_+}(t+1)-\frac{1}{T}}- \frac{1}{(T\bar{s_+}(t)-1)}\\
         \geq{}& \frac{1}{qT}\qty(\frac{1}{\bar{s_+}(t)}-\frac{\Delta s}{\bar{s_+}^2(t)})\cdot\\&{}\qty(\frac{1}{\bar{s_+}(t)-\frac{1}{T}}-\frac{\Delta s}{(\bar{s_+}(t)-\frac{1}{T})^2})- \frac{1}{(T\bar{s_+}(t)-1)}\\
         \geq{}&\frac{(1-q\bar{s_+}(t))}{q\bar{s_+}(t)(T\bar{s_+}(t)-1)}-\frac{\Delta s}{qs_+(t)^2(Ts_+-1)}-\frac{\Delta s}{qs_+(t)(Ts_+-1)^2}\\
         ={}&\frac{(1-qs_+(t))}{q\bar{s_+}(t)(T\bar{s_+}(t)-1)}-\frac{\Delta s(2Ts_+-1)}{qs_+(t)^2(Ts_+-1)^2}
     \end{align*}
    Then plug in the original term, we have the lower bound for 
     \begin{align*}
         &\qty(1+\frac{(8dT-d_e)(1-q\bar{s_+}(t+1))}{8dTq\bar{s_+}(t+1)(T\bar{s_+}(t+1)-1)})\alpha^*(t+1)\\
         \geq{}&\qty(1+\frac{(4dT-d_e)(1-q\bar{s_+}(t))}{4dq\bar{s_+}(t)(T\bar{s_+}(t)-1)})\alpha^*(t)+\frac{d_e(1-q\bar{s_+}(t))}{8dTq\bar{s_+}(t)(T\bar{s_+}(t)-1)}\alpha^*(t)\\
         &-\qty(1+\frac{(8dT-d_e)(1-q\bar{s_+}(t+1))}{8dTq\bar{s_+}(t+1)(T\bar{s_+}(t+1)-1)})\frac{Tq\bar{s_+}\Delta s}{(Tq\bar{s_+}^2(t)-2q\bar{s_+}(t)+1)}\alpha^*(t)\\
         &-\frac{\Delta s(2T\bar{s_+}-1)}{q\bar{s_+}(t)^2(T\bar{s_+}-1)^2}\alpha^*(t)
     \end{align*}
     Since $\frac{(8dT-d_e)(1-q\bar{s_+}(t+1))}{8dTq\bar{s_+}(t+1)(T\bar{s_+}(t+1)-1)}\leq \frac{8dT-d_e}{8dT}\leq 1$ when $s_+\geq \frac{1}{\sqrt{Tq}}$, we have
    \begin{align*}
         &\qty(1+\frac{(8dT-d_e)(1-q\bar{s_+}(t+1))}{8dTq\bar{s_+}(t+1)(T\bar{s_+}(t+1)-1)})\alpha^*(t+1)\\
         \geq{}&\qty(1+\frac{(4dT-d_e)(1-q\bar{s_+}(t))}{4dq\bar{s_+}(t)(T\bar{s_+}(t)-1)})\alpha^*(t)+\frac{d_e(1-q\bar{s_+}(t))}{8dTq\bar{s_+}(t)(T\bar{s_+}(t)-1)}\alpha^*(t)\\
         &-\frac{2Tq\bar{s_+}\Delta s}{(Tq\bar{s_+}^2(t)-2q\bar{s_+}(t)+1)}\alpha^*(t)-\frac{\Delta s(2T\bar{s_+}-1)}{q\bar{s_+}(t)^2(T\bar{s_+}-1)^2}\alpha^*(t)\tag{\textbf{***}}
         \label{ineq: upper bound dynamics 2}
     \end{align*}
    Then we need to prove that (here $\bar{s_+}:= \bar{s_+}(t)$) to show \textbf{(S4)}.
     \begin{align*}
         \frac{d_e(1-q\bar{s_+}(t))}{8dTq\bar{s_+}(t)(T\bar{s_+}(t)-1)}\geq{}& \frac{2Tq\bar{s_+}\Delta s}{(Tq\bar{s_+}^2(t)-2q\bar{s_+}(t)+1)}+\frac{\Delta s(2T\bar{s_+}-1)}{q\bar{s_+}(t)^2(T\bar{s_+}-1)^2}\\
         \Longleftrightarrow{}\frac{d_e(1-q\bar{s_+})}{8dTq(T\bar{s_+}-1)}\geq{}& \frac{2Tq\bar{s_+}^2\Delta s}{(Tq\bar{s_+}^2-2q\bar{s_+}+1)}+\frac{\Delta s(2T\bar{s_+}-1)}{q\bar{s_+}(T\bar{s_+}-1)^2}
     \end{align*}
      We have that the right hand side has the following upper bound ($T\gg q$):
     \begin{align*}
         &\frac{2Tq\bar{s_+}^2\Delta s}{(Tq\bar{s_+}^2-2q\bar{s_+}+1)}+\frac{\Delta s(2T\bar{s_+}-1)}{q\bar{s_+}(T\bar{s_+}-1)^2}\\
         \leq{}&\frac{45Tq\bar{s_+}^3\qty({1}-q\bar{s_+})\Delta C(t)}{16(Tq\bar{s_+}^2-2q\bar{s_+}+1)}+\frac{45(2T\bar{s_+}-1)\bar{s_+}\qty({1}-q\bar{s_+})\Delta C(t)}{32q\bar{s_+}(T\bar{s_+}-1)^2}\tag{$\Delta s\leq \frac{45}{32}\bar{s_+}\qty({1}-q\bar{s_+})\Delta C(t)$}\\
         \leq{}&\frac{45Tq\bar{s_+}^3\qty({1}-q\bar{s_+})\Delta C(t)}{16q\bar{s_+}(T\bar{s_+}-1)}+\frac{135\bar{s_+}\qty({1}-q\bar{s_+})\Delta C(t)}{32q\bar{s_+}(T\bar{s_+}-1)}\tag{$s_+\leq 1/q, T\geq 4q$}\\
         \leq{}&\frac{45Tq\bar{s_+}^2\qty({1}-q\bar{s_+})\Delta C(t)}{16q(T\bar{s_+}-1)}+\frac{135\qty({1}-q\bar{s_+})\Delta C(t)}{32q(T\bar{s_+}-1)}      
     \end{align*}
         Let $\alpha=(1+\gamma \frac{1-q\bar{s_+}}{q\bar{s_+}(T\bar{s_+}-1)})\alpha^*,$ then we have the upper bound for the update
     \begin{align*}
         \Delta C(t)&\leq \eta \frac{1-\delta}{1-2\delta}\cdot\frac{d\alpha(t)}{d_e}\E_{y}\E_{\bE}^{(R)}(1-qs_+)\qty(1-\frac{q\alpha(t)(T\bar{s_+}-1)}{T-q})s_+\\
         &\leq \frac{9}{8}\eta (1-\gamma)(1+\gamma \frac{1-q\bar{s_+}}{q\bar{s_+}(T\bar{s_+}-1)}) \frac{d\alpha^*}{d_e(Tq\bar{s_+}^2-2q\bar{s_+}+1)}\bar{s_+}(1-q\bar{s_+})^2\tag{$\E[{s_+}^2]\geq \E[{s_+}]^2,\delta<0.1$}\\
         &\leq \eta(1-\gamma^2)\frac{9d(T-q)\bar{s_+}^2(1-q\bar{s_+})^2}{8d_e(Tq\bar{s_+}^2-2q\bar{s_+}+1)^2}\tag{UP1}\\
         &\leq \eta(1-\gamma^2)\frac{9d(T-q)(1-q\bar{s_+})^2}{8d_e(Tq\bar{s_+}-2q+1/\bar{s_+})^2}\\
         &\leq \eta(1-\gamma^2)\frac{9d(T-q)(1-\frac{\sqrt q}{\sqrt T})^2}{8d_e(2\sqrt{Tq}-2q)^2}\tag{$s_+\geq 1/\sqrt{Tq}$.}
         \\&\leq \frac{9\eta d}{32qd_e}.\tag{UP2}
     \end{align*}
     Since $\eta\leq \frac{d_e}{40d^2T}$, plug the upper bounds (UP1), (UP2) for $\Delta C(t)$ back to the two terms in \textbf{(***)} respectively:
     \begin{align*}
         &\frac{2Tq\bar{s_+}^2\Delta s}{(Tq\bar{s_+}^2-2q\bar{s_+}+1)}+\frac{\Delta s(2T\bar{s_+}-1)}{q\bar{s_+}(T\bar{s_+}-1)^2}\\
         \leq{}&\frac{45Tq\bar{s_+}^2\qty({1}-q\bar{s_+})\Delta C(t)}{16q(T\bar{s_+}-1)}+\frac{135\qty({1}-q\bar{s_+})\Delta C(t)}{32q(T\bar{s_+}-1)}\\
         \leq{}&\frac{405dT(T-q)\bar{s_+}^4\qty({1}-q\bar{s_+})^3}{128(T\bar{s_+}-1)d_e(Tq\bar{s_+}^2-2q\bar{s_+}+1)^2}+\frac{135\qty({1}-q\bar{s_+})\Delta C(t)}{32q(T\bar{s_+}-1)}\cdot \frac{9\eta d}{32q}\\
         \leq{}&\frac{405\eta dT(1-q\bar{s_+})}{128q(T\bar{s_+}-1)(T-q)qd_e}+\frac{1215\qty({1}-q\bar{s_+})}{1024q(T\bar{s_+}-1)}\cdot \frac{\eta d}{qd_e}\\
         \leq{}&\frac{405\eta d(1-q\bar{s_+})}{128q(T\bar{s_+}-1)d_e}+\frac{1215\qty({1}-q\bar{s_+})}{1024q(T\bar{s_+}-1)}\cdot \frac{\eta d}{qd_e}\tag{$q\geq 2, T\geq 4q$}\\
         \leq{}&\frac{d_e(1-q\bar{s_+})}{8dTq(T\bar{s_+}-1)}.
     \end{align*}
     Therefore, we proved the inequality. And therefore, \textbf{(S2)} is proved, which also leads to \textbf{(S1)}. 

     Finally, we consider $\alpha(t)\in \left( \qty(1+\frac{(4dT-d_e)(1-q\bar{s_+}(t))}{4dTq\bar{s_+}(t)(T\bar{s_+}(t)-1)})\alpha^*(t),  \qty(1+\frac{(8dT-d_e)(1-q\bar{s_+}(t))}{8dTq\bar{s_+}(t)(T\bar{s_+}(t)-1)})\alpha^*(t) \right]$.
     Now, since by update rule,$$\alpha(t+1) = \alpha(t) + \eta \bar{s_+}\qty(1-\frac{\alpha(t)}{\alpha^*(t)})\leq \qty(1+\frac{(8dT-d_e)(1-q\bar{s_+}(t))}{8dTq\bar{s_+}(t)(T\bar{s_+}(t)-1)})\alpha^*(t) - \frac{\eta(4dT-d_e)(1-q\bar{s_+}(t))}{4dTq(T\bar{s_+}(t)-1)}$$
     Therefore, it's sufficient to prove that 
     \begin{align*}
         &\qty(1+\frac{(8dT-d_e)(1-q\bar{s_+}(t+1))}{8dTq\bar{s_+}(t+1)(T\bar{s_+}(t+1)-1)})\alpha^*(t+1)\\\geq{}& \qty(1+\frac{(8dT-d_e)(1-q\bar{s_+}(t))}{8dTq\bar{s_+}(t)(T\bar{s_+}(t)-1)})\alpha^*(t) - \frac{\eta(4dT-d_e)(1-q\bar{s_+}(t))}{4dTq(T\bar{s_+}(t)-1)}
     \end{align*}
     Note \Cref{ineq: upper bound dynamics 2} gives the lower bound for the left hand side:
         \begin{align*}
         &\qty(1+\frac{(8dT-d_e)(1-q\bar{s_+}(t+1))}{8dTq\bar{s_+}(t+1)(T\bar{s_+}(t+1)-1)})\alpha^*(t+1)\\
         \geq{}&\qty(1+\frac{(8dT-d_e)(1-q\bar{s_+}(t))}{8dTq\bar{s_+}(t)(T\bar{s_+}(t)-1)})\alpha^*(t)\\
         &-\frac{2Tq\bar{s_+}\Delta s}{(Tq\bar{s_+}^2(t)-2q\bar{s_+}(t)+1)}\alpha^*(t)-\frac{\Delta s(2T\bar{s_+}-1)}{q\bar{s_+}(t)^2(T\bar{s_+}-1)^2}\alpha^*(t)\tag{\textbf{\#\#}}
         \label{ineq: upper bound @1}
     \end{align*}
     Yet when $\alpha(t)\geq \qty(1+\frac{(4dT-d_e)(1-q\bar{s_+}(t))}{4dTq\bar{s_+}(t)(T\bar{s_+}(t)-1)})\alpha^*(t)$, we have a better upper bound for $\Delta C(t)$:
     \begin{align*}
         \Delta C(t)&\leq \eta \frac{1-\delta}{1-2\delta}\cdot\frac{d\alpha(t)}{d_e}\E_{y}\E_{\bE}^{(R)}(1-qs_+)\qty(1-\frac{q\alpha(t)(T\bar{s_+}-1)}{T-q})s_+\\
         &\leq \frac{9}{8}\eta (1-\gamma)(1+\gamma \frac{1-q\bar{s_+}}{q\bar{s_+}(T\bar{s_+}-1)}) \frac{d\alpha^*}{d_e(Tq\bar{s_+}^2-2q\bar{s_+}+1)}\bar{s_+}(1-q\bar{s_+})^2\tag{$\E[{s_+}^2]\geq \E[{s_+}]^2,\delta<0.1$}\\
         &\leq \eta(1-\gamma^2)\frac{9d(T-q)\bar{s_+}^2(1-q\bar{s_+})^2}{8d_e(Tq\bar{s_+}^2-2q\bar{s_+}+1)^2}\\
         &\leq \eta\qty(1-\qty(\frac{4dT-d_e}{4dT})^2)\frac{9d(T-q)\bar{s_+}^2(1-q\bar{s_+})^2}{8d_e(Tq\bar{s_+}^2-2q\bar{s_+}+1)^2}\\
         &\leq \frac{\eta}{2dT/d_e}\cdot\frac{9d(T-q)\bar{s_+}^2(1-q\bar{s_+})^2}{8d_e(Tq\bar{s_+}^2-2q\bar{s_+}+1)^2}\\
         &=\frac{9\eta(T-q)\bar{s_+}^2(1-q\bar{s_+})^2}{16T(Tq\bar{s_+}^2-2q\bar{s_+}+1)^2}
     \end{align*}
     Then we need to bound both terms in \Cref{ineq: upper bound @1} (for simplicity denote $s_+$ as $s_+(t)$):
     \begin{align*}
         &\frac{2Tq\bar{s_+}\Delta s}{(Tq\bar{s_+}^2(t)-2q\bar{s_+}(t)+1)}\alpha^*(t)
         =\frac{2Tq(T-q)\bar{s_+}^2\Delta s}{(Tq\bar{s_+}^2(t)-2q\bar{s_+}(t)+1)^2}\\
         \leq{}&\frac{45Tq(T-q)\bar{s_+}^3\qty({1}-q\bar{s_+})\Delta C(t)}{16(Tq\bar{s_+}^2(t)-2q\bar{s_+}(t)+1)^2}\\
         \leq{}&\frac{405\eta q(T-q)^2\bar{s_+}^5\qty({1}-q\bar{s_+})^3}{256(Tq\bar{s_+}^2(t)-2q\bar{s_+}(t)+1)^4}\tag{Plug in $\Delta C(t)$ bound.}\\
         \leq{}&\frac{405\eta }{256}\cdot (T-q)^2\cdot\frac{(1-q\bar{s_+})}{q\bar{s_+}(T\bar{s_+}-1)}\cdot\frac{q\bar{s_+}^5(1-q\bar{s_+})^2}{(Tq\bar{s_+}^2-2q\bar{s_+}+1)^3}\tag{$s_+\leq \frac{1}{q}$}\\
         ={}&\frac{405\eta }{256}\cdot (T-q)^2\cdot\frac{(1-q\bar{s_+})}{q(T\bar{s_+}-1)}\cdot\frac{q(1-q\bar{s_+})^2}{(Tq\bar{s_+}^2-2q\bar{s_+}+1)(Tq-2q/\bar{s_+}+1/\bar{s_+}^2)^2}\\
         \leq{}&\frac{405\eta }{256}\cdot (T-q)^2\cdot\frac{(1-q\bar{s_+})}{q(T\bar{s_+}-1)}\cdot\frac{q(1-q\bar{s_+})^2}{(Tq\bar{s_+}^2-2q\bar{s_+}+1)(Tq-q^2)^2} \tag{$s_+\leq \frac{1}{q}$}\\
         \leq{}&\frac{405\eta }{256}\cdot\frac{(1-q\bar{s_+})}{q(T\bar{s_+}-1)}\cdot\frac{(1-q\bar{s_+})^2}{(Tq\bar{s_+}^2-2q\bar{s_+}+1)q}\\
         \leq{}&\frac{405\eta }{256}\cdot\frac{(1-q\bar{s_+})}{q(T\bar{s_+}-1)}\cdot\frac{1}{2q}\tag{$s_+\geq \frac{1}{\sqrt{Tq}}$}\\
         \leq{}&\frac{405\eta }{1024}\cdot\frac{(1-q\bar{s_+})}{q(T\bar{s_+}-1)}\tag{$s_+\geq \frac{1}{\sqrt{Tq}}$}
     \end{align*}
     \begin{align*}
         &\frac{\Delta s(2T\bar{s_+}-1)}{q\bar{s_+}(t)^2(T\bar{s_+}-1)^2}\alpha^*(t)\\
         \leq{}&\frac{(2T\bar{s_+}-1)}{q\bar{s_+}(t)^2(T\bar{s_+}-1)^2}\alpha^*(t)\cdot \frac{45}{32}\bar{s_+}\qty({1}-q\bar{s_+})\Delta C(t)\tag{Plug in $\Delta s$}\\
         \leq {}&\frac{135(T-q)\qty({1}-q\bar{s_+})}{32q(T\bar{s_+}-1)(Tq\bar{s_+}^2-2q\bar{s_+}+1)}\cdot \frac{9\eta (T-q)\bar{s_+}^2(1-q\bar{s_+})^2}{16T(Tq\bar{s_+}^2-2q\bar{s_+}+1)^2}\tag{Plug in $\Delta C(t)$}\\
         ={}&\frac{1215\eta (T-q)^2\bar{s_+}^2\qty({1}-q\bar{s_+})^3}{512qT(T\bar{s_+}-1)(Tq\bar{s_+}^2-2q\bar{s_+}+1)^3}\\
         \leq{}&\frac{1215\eta }{512}\cdot \frac{(1-q\bar{s_+})}{q(T\bar{s_+}-1)}\frac{(1-q\bar{s_+})^2}{(Tq\bar{s_+}^2-2q\bar{s_+}+1)^2(Tq-2q/\bar{s_+}+1/\bar{s_+}^2)}\\
         \leq{}&\frac{1215\eta }{4096}\cdot\frac{(1-q\bar{s_+})}{q(T\bar{s_+}-1)}.
     \end{align*}     
     Therefore, by \Cref{ineq: upper bound @1} we have
     \begin{align*}
         &\qty(1+\frac{(8dT-d_e)(1-q\bar{s_+}(t+1))}{8dTq\bar{s_+}(t+1)(T\bar{s_+}(t+1)-1)})\alpha^*(t+1)\\
         \geq{}&\qty(1+\frac{(8dT-d_e)(1-q\bar{s_+}(t))}{8dTq\bar{s_+}(t)(T\bar{s_+}(t)-1)})\alpha^*(t)\\
         &-\frac{2Tq\bar{s_+}\Delta s}{(Tq\bar{s_+}^2(t)-2q\bar{s_+}(t)+1)}\alpha^*(t)-\frac{\Delta s(2T\bar{s_+}-1)}{q\bar{s_+}(t)^2(T\bar{s_+}-1)^2}\alpha^*(t)\\
         \geq{}&\qty(1+\frac{(8dT-d_e)(1-q\bar{s_+}(t))}{8dTq\bar{s_+}(t)(T\bar{s_+}(t)-1)})\alpha^*(t)-\qty(\frac{405}{1024}+\frac{1215}{4096})\frac{\eta(1-qs_+)}{q(Ts_+-1)}\\
         \geq{}&\qty(1+\frac{(8dT-d_e)(1-q\bar{s_+}(t))}{8dTq\bar{s_+}(t)(T\bar{s_+}(t)-1)})\alpha^*(t)-\frac{\eta(4dT-d_e)(1-q\bar{s_+}(t))}{4dTq(T\bar{s_+}(t)-1)}\tag{$dT\geq d_e$.}\\
         \geq{}&\alpha(t+1).
     \end{align*}
     Therefore, we finish the induction proof for (\ref{IH: alpha stays near alpha star with stochastic pe}). 

     With the induction hypothesis, we can analyze the upper bound of convergence time. For \textbf{Phase I}, we have the lower bound for $\Delta C(t)$ (note that $1/2<\alpha < 1$):
     \begin{align*}
         \Delta C(t+1)&\geq \eta \frac{1-3\delta}{1-2\delta}\cdot\frac{d\alpha(t)}{d_e}\E_{y}\E_{\bE}^{(R)}(1-qs_+)\qty(1-\frac{q\alpha(t)(T\bar{s_+}-1)}{T-q})s_+\\
         &= \frac{7\eta}{8}\cdot\frac{\alpha d}{d_e}\E_{y}\E_{\bE}^{(R)}(1-qs_+)\frac{T-q+q-Tq\bar{s_+}}{T-q}s_+\\
         &= \frac{7\eta dT}{16d_e(T-q)}\E_{y}\E_{\bE}^{(R)}(1-qs_+)(1-q\bar{s_+})s_+
     \end{align*}
     And for \textbf{Phase II}, we have the lower bound since $\alpha\leq \qty(1+\frac{(8Td-d_e)(1-q\bar{s_+})}{8Tdq\bar{s_+}(T\bar{s_+}-1)})\alpha^*(t)$:
     \begin{align*}
         \Delta C(t) &\geq \eta (1-\frac{8Td-d_e}{8Td})\qty(1+\frac{(8Td-d_e)(1-q\bar{s_+})}{8Tdq\bar{s_+}(T\bar{s_+}-1)})\alpha^*(t)\cdot \frac{7d\E_{y}\E_{\bE}^{(R)}(1-qs_+)(1-q\bar{s_+})s_+}{8d_e(Tq\bar{s_+}^2-2q\bar{s_+}+1)}\\
         &\geq \frac{\eta d_e}{8dT}\cdot \frac{7d(T-q)\bar{s_+}\E_{y}\E_{\bE}^{(R)}(1-qs_+)(1-q\bar{s_+})s_+}{8d_e(Tq\bar{s_+}^2-2q\bar{s_+}+1)^2}\\
         &\geq \frac{7\eta(T-q)\bar{s_+}\E_{y}\E_{\bE}^{(R)}(1-qs_+)(1-q\bar{s_+})s_+}{64T(Tq\bar{s_+}^2-2q\bar{s_+}+1)^2}
     \end{align*}

    Then we also divide the training trajectory into two stages as in \citet{huang2023context}: in the first stage from $(0,t_1]$, all possible $s_+$ grow to $1/2q$. In the second stage $(t_1,t_2]$, $C(t)$ grows large enough s.t. $\bar{s_+}\geq \frac{1}{q}-\frac{1}{q}\sqrt{\frac{q(T-q)\epsilon}{dT}}$.
    
    For the first stage, it's necessary that $C(t)\geq \frac{1-2\delta}{1-3\delta}\log\frac{T-q}{q}$. This is because by \Cref{Lemma: Estimation for softmax vector with W}, we have:
    \begin{align*}
          \frac{1}{q+(T-q)e^{-\frac{1-3\delta}{1-2\delta}C}} &\leq s_+\leq \frac{1}{q+(T-q)e^{-\frac{1-\delta}{1-2\delta}C}},   
    \end{align*}
    To make any $s_+(\bE)$ reaches this threshold $1/2q$, $C(t_1)\geq \frac{1-2\delta}{1-3\delta}\log\frac{T-q}{q}$ is necessary. At this time $t_1$, all $s_+\leq 1/2q$, which means the expectation $\bar{s_+}\leq 1/2q.$ 
    At this iteration, we can also upper bound all the attention score:
    $$ s_+\leq \frac{1}{q+(T-q)e^{-\frac{1-\delta}{1-2\delta}C(t)}}\leq \frac{1}{q+(T-q)^{-\frac{2\delta}{1-3\delta}}q^{\frac{1-\delta}{1-3\delta}}}$$
    That means
    \begin{align*}
        (1-qs_+) &\geq \qty(1-q\cdot\frac{1}{q+(T-q)^{-\frac{2\delta}{1-3\delta}}q^{\frac{1-\delta}{1-3\delta}}} )\\
        &= \qty(\frac{q^{\frac{2\delta}{1-3\delta}}}{(T-q)^{\frac{2\delta}{1-3\delta}}+q^{\frac{2\delta}{1-3\delta}}})\geq \qty(\frac{q}{T})^{\frac{2\delta}{1-3\delta}} 
    \end{align*}
    Since all $s_+$ satisfy the lower bound, the expectation $\bar{s_+}$ also has this lower bound. Therefore
    during this stage, by \textbf{Phase I} lower bound for $\Delta C(t)$ we have:
    \begin{align*}
        \Delta C(t)&\geq \frac{7\eta dT}{16d_e(T-q)}\E_{y}\E_{\bE}^{(R)}(1-qs_+)(1-q\bar{s_+})s_+\\
        &\geq \frac{7\eta d}{16d_e(T-q)}\cdot\qty(\frac{q}{T})^{\frac{4\delta}{1-3\delta}}  \tag{$s_+\geq \frac{1}{T},(1-qs_+) \geq \qty(\frac{q}{T})^{\frac{2\delta}{1-3\delta}} $}
    \end{align*}
    And for the \textbf{Phase II} lower bound, we have:
    \begin{align*}
         \Delta C(t) 
         &\geq \frac{7\eta(T-q)\bar{s_+}\E_{y}\E_{\bE}^{(R)}(1-qs_+)(1-q\bar{s_+})s_+}{64T(Tq\bar{s_+}^2-2q\bar{s_+}+1)^2}\\
         &\geq \frac{7\eta(T-q)\bar{s_+}^2 \qty(\frac{q}{T})^{\frac{4\delta}{1-3\delta}}}{64T(Tq\bar{s_+}^2-2q\bar{s_+}+1)^2}\\
         &\geq \frac{7\eta(T-q) \qty(\frac{q}{T})^{\frac{4\delta}{1-3\delta}}}{64T(Tq-2q/\bar{s_+}+1/\bar{s_+}^2)(T/q-1)}\\
         &\geq \frac{7\eta \qty(\frac{q}{T})^{\frac{4\delta}{1-3\delta}}}{64T(T-q)}
     \end{align*}
     Since $T\gg d_e =\Theta(q\log T)$, $\Delta C(t) \geq \frac{7\eta \qty(\frac{q}{T})^{\frac{4\delta}{1-3\delta}}}{64T(T-q)}$ through all the first stage. Therefore, for $C(t)$ to reach $\Theta(\log \frac{T-q}{q})$, it takes at most $O\qty({\log \frac{T-q}{q}} /{\frac{7\eta \qty(\frac{q}{T})^{\frac{4\delta}{1-3\delta}}}{64T(T-q)}})=\Tilde{O}(T^{\frac{2-2\delta}{1-3\delta}})$ iterations.

    For the second stage, we need $C(t_2)\geq \frac{1-2\delta}{1-3\delta}\log(\frac{T-q}{q}\sqrt{\frac{Td}{(T-q)q\epsilon}})$ so that $\bar{s_+}\geq \frac{1}{q}-\frac{1}{q}\sqrt{\frac{q(T-q)\epsilon}{dT}}$. Since all $s_+\geq \frac{1}{2q}$ during this period, we can lower bound the increment:
    \begin{align*}
         \Delta C(t) 
         &\geq \frac{7\eta(T-q)\bar{s_+}\E_{y}\E_{\bE}^{(R)}(1-qs_+)(1-q\bar{s_+})s_+}{64T(Tq\bar{s_+}^2-2q\bar{s_+}+1)^2}\\
         &\geq \frac{7\eta (T-q)\cdot \frac{1}{2q}\E_{y}\E_{\bE}^{(R)}(1-qs_+)(1-q\bar{s_+})\cdot\frac{1}{2q}}{64T(\frac{T}{q}-1)^2}\\
         &\geq \frac{7\eta }{128T(T-q)}\E_{y}(1-q\bar{s_+})^2\\
         &\geq \frac{7\eta q\epsilon}{128dT^2}\tag{Before $t_2$, $\bar{s_+}\leq \frac{1}{q}-\frac{1}{q}\sqrt{\frac{q(T-q)\epsilon}{dT}}$}
     \end{align*}
    Hence it takes at most $t_2\leq O(\frac{1-2\delta}{1-3\delta}\log(\frac{T-q}{q}\sqrt{\frac{Td}{2(T-q)q\epsilon}}) / \frac{7\eta q\epsilon}{128dT^2})=\tilde{O}(\frac{dT^2}{\eta \epsilon})$ iterations to converge.

    Finally, we check that if $\bar{s_+}\geq\frac{1}{q}-\frac{1}{q}\sqrt{\frac{q(T-q)\epsilon}{dT}}$, then the loss is smaller than $\epsilon$. This part of proof is the similar to \Cref{thm: joint training one hot}\footnote{By `similar', one can replace all the $s_+$ in the proof of \Cref{thm: joint training one hot} with $\bar{s_+}$ here.}. By \Cref{Lemma: stochastic pe loss expression}:
    \begin{align*}
        \mathcal{L}(\bm{\theta}(t)) &=\frac{1}{2}\E\qty[d\left\|\frac{1}{q}\Y-\alpha(t)\E_{\bE}^{(R)}[\S_y]\right\|^2]\\
        &=\frac{d}{2(T-q)}\qty((T-q)q\qty(\alpha(t)\bar{s_+} - \frac{1}{q})^2 + {\alpha(t)^2(1-q\bar{s_+})^2})
    \end{align*}
    While $\alpha(t)\in\qty[1-0.1\sqrt{\frac{q(T-q)\epsilon}{dT}},  \qty(1+\frac{(8Td-d_e)(1-q\bar{s_+})}{8Tdq\bar{s_+}(T\bar{s_+}-1)})\alpha^*(t)]:=[\alpha_1,\alpha_2]$, the loss value is upper bounded by $$\max_{j\in\{1,2\}}\frac{d}{2(T-q)}\qty((T-q)q\qty(\alpha_j \bar{s_+} - \frac{1}{q})^2 + {\alpha_j^2(1-q\bar{s_+})^2})$$

    For $\alpha_1 = 1-0.1\sqrt{\frac{q(T-q)\epsilon}{dT}}$, we have 
    \begin{align*}
        \mathcal{L}(\bm{\theta}(t))={}&\frac{d}{2(T-q)}\qty((T-q)q\qty(\alpha(t)\bar{s_+} - \frac{1}{q})^2 + {\alpha(t)^2(1-q\bar{s_+})^2})\\
        \leq{}&\frac{d}{2(T-q)}\qty(\frac{T-q}{q}\qty(q\alpha(t)\bar{s_+} - 1)^2 + {(1-q\bar{s_+})^2})\\
        \leq{}&\frac{d}{2(T-q)}\qty(\frac{T-q}{q}\qty(q\bar{s_+} - 1+0.1q\bar{s_+}\sqrt{\frac{q(T-q)\epsilon}{dT}})^2 + {(1-q\bar{s_+})^2})\\
        \leq{}&\frac{d}{2(T-q)}\qty(\frac{T-q}{q}\cdot \qty((1-q\bar{s_+})^2 +0.2\sqrt{\frac{q(T-q)\epsilon}{dT}}(1-q\bar{s_+})+0.01\cdot\frac{q(T-q)\epsilon}{dT})+ (1-q\bar{s_+})^2)\\
        \leq{}&\frac{d}{2(T-q)}\qty(\frac{T-q}{q}\cdot \qty((1-q\bar{s_+})^2 +0.2\sqrt{\frac{q(T-q)\epsilon}{dT}}\cdot\sqrt{\frac{q(T-q)\epsilon}{dT}}+0.01\cdot\frac{q(T-q)\epsilon}{dT}))\\
        {}& + \frac{d}{2(T-q)}(1-q\bar{s_+})^2\\
        \leq{}&\frac{d}{2(T-q)}\frac{T}{q} (1-q\bar{s_+})^2 + 0.21\epsilon\leq \epsilon.
    \end{align*}
    For $\alpha_2 = \qty(1+\frac{(8Td-d_e)(1-q\bar{s_+})}{8Tdq\bar{s_+}(T\bar{s_+}-1)})\alpha^*(t)$, denote $\Delta \alpha=\alpha_2-\alpha^*$. Then $\Delta \alpha \bar{s_+}\leq \frac{2\sqrt{\frac{q(T-q)\epsilon}{dT}}}{T}\alpha^*$. Also, for $\alpha^*$ we have a upper bound (using $\epsilon\leq \frac{dT}{100(T-q)q})$):
    \begin{align*}
        \alpha^*=\frac{T-q}{Tq\bar{s_+}-2q+\frac{1}{\bar{s_+}}}\leq \frac{1}{q\bar{s_+}}\leq \frac{1}{1-\sqrt{\frac{q(T-q)\epsilon}{dT}}}\leq 1+\frac{6}{5}\sqrt{\frac{q(T-q)\epsilon}{dT}}\leq \frac{28}{25}.
    \end{align*}
    Therefore, the loss can be bounded by
    \begin{align*}
        \mathcal{L}(\bm{\theta}(t))={}&\frac{d}{2(T-q)}\qty((T-q)q\qty(\alpha_2\bar{s_+} - \frac{1}{q})^2 + {\alpha_2^2(1-q\bar{s_+})^2})\\
        ={}&\frac{d}{2(T-q)}\qty((T-q)q\qty(\alpha^*\bar{s_+} - \frac{1}{q})^2 + {\alpha^*}^2(1-q\bar{s_+})^2)\\
        +&\frac{d}{2(T-q)}\qty(-(T-q)\frac{4\sqrt{\frac{q(T-q)\epsilon}{dT}}}{T}\alpha^*(1-q\bar{s_+}\alpha^*) + \frac{4q^2(T-q)^2\epsilon}{dT^3}{\alpha^*}^2)\\
        +&\frac{d}{2(T-q)} \qty(\frac{4\sqrt{\frac{q(T-q)\epsilon}{dT}}}{T}{\alpha^*}^2(1-q\bar{s_+})^2+\frac{4q(T-q)\epsilon}{dT^3}{\alpha^*}^2(1-q\bar{s_+})^2)\\
        ={}&\frac{d}{2(T-q)}\qty((T-q)q\qty(\alpha^* \bar{s_+} - \frac{1}{q})^2 + {\alpha^*}^2(1-q\bar{s_+})^2)\\
        +&\frac{d}{2(T-q)}\qty(-(T-q)\frac{4\sqrt{\frac{q(T-q)\epsilon}{dT}}}{T(Tq\bar{s_+}^2-2q\bar{s_+}+1)}\alpha^*(1-q\bar{s_+})^2 + \frac{4q^2(T-q)^2\epsilon}{dT^3}{\alpha^*}^2)\\
        +&\frac{d}{2(T-q)} \qty(\frac{4\sqrt{\frac{q(T-q)\epsilon}{dT}}}{T}{\alpha^*}^2(1-q\bar{s_+})^2+\frac{4q(T-q)\epsilon}{dT^3}{\alpha^*}^2(1-q\bar{s_+})^2)\\
        ={}&\frac{dT}{2q(T-q)}(q\bar{s_+}(t)-1)^2
        +\frac{d}{2(T-q)}\qty(-\frac{4\sqrt{\frac{q(T-q)\epsilon}{dT}}}{T\bar{s_+}}{\alpha^*}^2(1-q\bar{s_+})^2 + \frac{4q^2(T-q)^2\epsilon}{dT^3}{\alpha^*}^2)\\
        +&\frac{d}{2(T-q)} \qty(\frac{4\sqrt{\frac{q(T-q)\epsilon}{dT}}}{T}{\alpha^*}^2(1-q\bar{s_+})^2+\frac{4q(T-q)\epsilon}{dT^3}{\alpha^*}^2(1-q\bar{s_+})^2)\\
        \leq{} &\frac{\epsilon}{2} + \frac{d}{2(T-q)}\qty(\frac{4q^2(T-q)^2\epsilon}{dT^3}{\alpha^*}^2+\frac{4q(T-q)\epsilon}{dT^3}{\alpha^*}^2(1-q\bar{s_+})^2)\\
        \leq{}&\frac{\epsilon}{2} + \frac{d}{2(T-q)}\qty(\frac{4q^2(T-q)^2\epsilon}{dT^3}\qty(\frac{28}{25})^2+\frac{4q(T-q)\epsilon}{dT^3}(\frac{28}{25})^2)\leq \epsilon.
    \end{align*}
    when $T\geq 4q$. 

    In conclusion, after \textbf{Phase I} (at most $O(\frac{T\log \frac{d}{\epsilon}}{\eta})$ iterations) and \textbf{Phase II} (takes at most $\tilde{O}(\frac{T^{\frac{2-2\delta}{1-3\delta}}}{\eta}+\frac{dT^2}{\eta \epsilon})$ iterations), the population loss $\mathcal{L}(\bm{\theta}(t))\leq \epsilon$
    after time $\tilde{O}(\frac{T^{\frac{2-2\delta}{1-3\delta}}}{\eta}+\frac{T^2 d}{\eta\epsilon})$.
\end{proof}

After proving the convergence to the global minimizer, we can directly have the following corollary on the parameter $\V(t)$ and $\W(t)$. By the dynamics proved in the theorem, $\W(t)$ is always along the $\bI_{d_e}$ direction, and $\V(t)$ should converge to $\bI_d$. It coincides with the construction in \citet{sanford2023representational}, showing the constructed one-layer transformer can be learned with GD.
\begin{corollary}\label{corollary: joint train param}
    Under the condition of \Cref{thm: joint training stochastic pe}, after time $t\geq \tilde{O}(\frac{T^{\frac{2-2\delta}{1-3\delta}}}{\eta}+\frac{T^2 d}{\eta\epsilon})$, we have $\W(t)=C(t)\bI_{d_e},\V(t)=\alpha(t)\bI_d$, and $C(t)\geq \frac{1-2\delta}{1-3\delta}\log(\frac{T-q}{q}\sqrt{\frac{Td}{(T-q)q\epsilon}})$, $\alpha(t)\in \qty[1-0.1\sqrt{\frac{q(T-q)\epsilon}{dT}}, 1+\frac{8}{5}\sqrt{\frac{q(T-q)\epsilon}{dT}}]$.
\end{corollary}
\begin{proof}
    The lower bound for $C(t)$ is proved in the \Cref{thm: joint training stochastic pe}.
    
    By \Cref{IH: alpha stays near alpha star with stochastic pe}, $\alpha^*(t) = \frac{(T-q)\bar{s_+}}{Tq\bar{s_+}^2-2q\bar{s_+}+1}$ and $C(t)\geq \frac{1-2\delta}{1-3\delta}\log(\frac{T-q}{q}\sqrt{\frac{Td}{(T-q)q\epsilon}})$, we have 
    $$\alpha(t)\in\qty[1-0.1\sqrt{\frac{q(T-q)\epsilon}{dT}}, \qty(1+\frac{(8Td-d_e)(1-q\bar{s_+})}{8Tdq\bar{s_+}(T\bar{s_+}-1)})\alpha^*(t)], \bar{s_+}\geq\frac{1}{q}-\frac{1}{q}\sqrt{\frac{q(T-q)\epsilon}{dT}}.$$
    and 
    \begin{align*}
        \alpha^*(t)=\frac{T-q}{Tq\bar{s_+}-2q+\frac{1}{\bar{s_+}}}\leq \frac{1}{q\bar{s_+}}\leq \frac{1}{1-\sqrt{\frac{q(T-q)\epsilon}{dT}}}\leq 1+\frac{6}{5}\sqrt{\frac{q(T-q)\epsilon}{dT}}.
    \end{align*}
    Therefore for $\qty(1+\frac{(8Td-d_e)(1-q\bar{s_+})}{8Tdq\bar{s_+}(T\bar{s_+}-1)})\alpha^*(t)$ we also have an upper bound.
    \begin{align*}
        \qty(1+\frac{(8Td-d_e)(1-q\bar{s_+})}{8Tdq\bar{s_+}(T\bar{s_+}-1)})\alpha^*(t)\leq{}&\qty(1+\frac{6}{5}\sqrt{\frac{q(T-q)\epsilon}{dT}})\qty(1+\frac{q}{T-q}\sqrt{\frac{q(T-q)\epsilon}{dT}}) \\\leq{}& 1+\frac{8}{5}\sqrt{\frac{q(T-q)\epsilon}{dT}}
    \end{align*}
    The last inequality uses $T\geq 4{q}, \epsilon\leq \frac{dT}{100q(T-q)}.$
\end{proof}

\subsection{Length generalization and out-of-distribution guarantee}\label{appen sec: ood guarantee}
In this subsection, we present the strength of stochastic positional encoding, which is the strong out-of-distribution guarantee, including the \textit{length generalization} performance mentioned in the main paper. Empirically, in \Cref{sec: experiments} we observe stochastic positional encoding has superiority over a fixed set of near-orthogonal positional encodings, especially in length generalization tasks. Here, based on the global minimizer found by gradient descent, we can also present a theoretical guarantee for an out-of-distribution guarantee.

Recall our data model is:
The input tokens $\x_i, i=1,2,..., T$ are sampled from standard Gaussian distribution, and the $q$-sparse subset $y$ containing all the averaging indices is uniformly sampled from all $q$-subsets of the set $[T]$.
$$\X=(\x_1,\x_2,...,\x_T),\x_i\sim \mathcal{N}(0,\bI_d),$$
$$ y\sim \unif\qty(\binom{[T]}{q}),i\in[T]$$

Suppose our training objective is based on the population distribution with sequence length $T_1$ and subset size $q_1$. By ``out-of-distribution'', we include two different tasks: (1) generalize on unseen data points with $q_2$-subsets where $q_2>q$. (2) generalize on longer sequences with $T_2>T_1$. We know the out-of-distribution loss with sequence length $T_2$ and $q_2$ as
\begin{equation}
    \mathcal{L}^{(s)}_{T_2,q_2}(\bm{\theta}) = \frac{1}{2}\E_{\X,y\sim\mathcal{D}_{T_2,q_2}}\qty[\|\qsat(\X,y)-f^{(s)}_{\bm{\theta}}(\X,y)\|_2^2]
    \label{eqn: ood objective for qsa}
\end{equation}
With the constructed transformer in \ref{appen sec: A.3 tf expressivity}, as long as $d_e=\Theta(q_2\log T_2)$ where the positional encoding matrix $\bE'\in \R^{d_e\times T_2}$ with sequence length $T_2$ can satisfy RIP, the one-layer transformer can express $\qsa$ with any input sequence length $T\leq T_2$ and $q\leq q_2$. 
In this paper, we present the following corollary that demonstrates the GD-trained transformer with stochastic PE can also achieve good OOD performance in both subset size $q_2\geq q_1$ and $T_2\geq T_1$ under the condition that $\bE$ can satisfy RIP with maximal sequence length $T_{\max}$ and maximal $q_{\max}$. 
\begin{corollary}\label{appen coro: ood guarantee}
Suppose $q_{\max}=\Theta(1),d_e=\Theta(q_{\max}\log T_{\max}/\delta^2), \delta < 1/10, \eta\leq \frac{d_e^2}{40d^2T}$. For any $\epsilon\in (0,\frac{dT}{100(T-q)q})$, and we apply gradient descent with zero initialization with $q_1<q_{\max}, T_1<T_{\max}$ to train the model. Then we have the following out-of-distribution loss generalization guarantee with $q_2\in(q_1,q_{\max}],T_2\in(T_1,T_{\max})$: if the training time $t\geq \tilde{O}(\frac{T^{\frac{2-2\delta}{1-3\delta}}}{\eta}+\frac{T^2 d}{\eta\epsilon})$, we have:
    $$\mathcal{L}^{(s)}_{T_2,q_2}(\bm{\theta})\leq O\qty(\frac{T_2^2\epsilon}{T_1^2})$$
\end{corollary}
\emph{Proof idea.} The intuition behind the corollary is that training the transformer with the stochastic architecture on arbitrary $T$ and $q$ where the RIP condition holds can lead to exact the same global minima in the construction equivalent to \cite{sanford2023representational}. That means it naturally satisfies all possible $T'$ and $q'$, as long as the RIP condition holds.
\begin{proof}

By \Cref{corollary: joint train param}, we have $\W(t)=C(t)\randomWdirection,\V(t)=\alpha(t)\randomVdirection$, and $C(t)\geq \frac{1-2\delta}{1-3\delta}\log(\frac{T-q}{q}\sqrt{\frac{Td}{(T-q)q\epsilon}})$, $\alpha(t)\in \qty[1-0.1\sqrt{\frac{q(T-q)\epsilon}{dT}}, 1+\frac{8}{5}\sqrt{\frac{q(T-q)\epsilon}{dT}}]$. 

Denote $\Delta\alpha = \alpha -1$.
Similarly, we can have the following bounds for the $s_+'$:
$$1-\E_{\bE'}[\S_{y'}(i)]q_2 = 1-q_2\bar{s_+}'\leq \frac{\frac{(T_2-q_2)q_1}{T_1-q_1}\sqrt{\frac{(T_1-q_1)q_1\epsilon}{T_1d}}}{q_2+\frac{(T_2-q_2)q_1}{T_1-q_1}\sqrt{\frac{(T_1-q_1)q_1\epsilon}{T_1d}}}\leq \frac{\frac{(T_2-q_2)q_1}{T_1-q_1}\sqrt{\frac{(T_1-q_1)q_1\epsilon}{T_1d}}}{q_2}$$
By \Cref{Lemma: stochastic pe loss expression}, we have 
\begin{align*}
    \mathcal{L}^{(s)}_{T_2,q_2}(\bm{\theta}) &= \frac{d}{2}\E_{y'}\qty[\left\|\frac{1}{q_2}\Y_{T_2}-\alpha(t)\E_{\bE'}[\S_{y'}]\right\|_2^2]\\
    &=\frac{d}{2}\qty(\frac{1}{q_2}(1-q_2\alpha(t)\bar{s_+}')^2 +\alpha(t)^2\frac{(1-q\bar{s_+}')^2}{T_2-q_2})\\
    &\leq \frac{d}{2} \cdot\frac{T_2}{(T_2-q_2)q_2}\cdot\qty(1-q\bar{s_+}')^2-\frac{d}{q_2}\Delta \alpha q_2\bar{s_+}'(1-q\bar{s_+}')\\&\quad+\frac{d}{q_2}\Delta \alpha^2 (q\bar{s_+}')^2 +\frac{d}{2}\cdot \frac{2\Delta \alpha +(\Delta \alpha)^2}{T_2-q_2}(1-q\bar{s_+}')^2\tag{Plug in $\bar{s_+}', \Delta\alpha$ bounds}\\
    &\leq \frac{T_2(T_2-q_2)q_1^3}{2q_2^3T_1(T_1-q_1)}+\frac{(T_2-q_2)q_1^2\epsilon}{10q_2^2T_1} +\frac{32q_1(T_1-q_1)\epsilon}{25q_2T_1} +\frac{3(T_2-q_2)q_1^3\epsilon}{2q_2^2T_1(T_1-q_1)}\leq O\qty(\frac{T_2^2\epsilon}{T_1^2}).
\end{align*}
Therefore, we complete the proof.
\end{proof}

\subsection{Supplementary lemmas on conditional expectations}
\begin{lemma}[Flipping rows doesn't change the softmax]
Given a $q$-subset $y$, denote the softmax output with input $\bE$ as $\S_y(\bE)=\sm(C\bE^\top \e_y)$, we have for any $i$:
$$\S_y(\bE)=\S_y(\bE'), \bE_y^\top \bE_y =\bE_y'^\top \bE_y'$$
where $\bE' = \bm{F}_i\bE$, $\bm{F}_i = \diag(1,...,1,-1,1,...,1)\in \R^{d_e\times d_e}$ with the $i$-th entry being $-1$.\footnote{$\bm{F}_i$ is to flip the $i$-th row of the positional encoding matrix $\bE$.}
\label{Lemma: Flipping rows doesn't change the softmax}
\end{lemma}

\begin{proof}
    Notice that 
    $$\bE^\top \e_y = \bE^\top\bE_y (\bE_y^\top\bE_y)^{-1}\mathbf{1}_q.$$
    while 
    $$\bE'^\top \e_y' = \bE^\top\bm{F}_i^\top\bm{F}_i\bE_y (\bE_y^\top\bm{F}_i^\top\bm{F}_i\bE_y)^{-1}\mathbf{1}_q=\bE^\top\bE_y (\bE_y^\top\bE_y)^{-1}\mathbf{1}_q.$$
    The second identity is due to $\bm{F}_i^\top\bm{F}_i=\bI_{d_e}$. So the vector inside $\sm$ is the same, so the outputs are the same. For $\bE_y^\top \bE_y =\bE_y'^\top \bE_y'$, it's similar:
    $$\bE_y^\top \bE_y =\bE_y^\top\bm{F}_i^\top\bm{F}_i\bE_y=\bE_y'^\top \bE_y'$$
\end{proof}

\begin{lemma}[Switching rows doesn't change conditional expectation]
Given a $q$-subset $y$, denote the softmax output with input $\bE$ as $\S_y(\bE)=\sm(C\bE^\top \e_y)$, we have for any $i$:
$$\S_y(\bE)=\S_y(\bE'), \bE_y^\top \bE_y =\bE_y'^\top \bE_y' $$
where $\bE' = \bm{R}_{ij}\bE$, $\bm{R}_{ij}\in \R^{d_e\times d_e}$ where $\bm{R}_{ij}$ is the elementary matrix to switch the $i$-th row and $j$-th row of the positional encoding matrix $\bE$.
\label{Lemma: Switching rows doesn't change the softmax}
\end{lemma}
\begin{proof}
    Similar to \Cref{Lemma: Flipping rows doesn't change the softmax}, notice that 
    $$\bE^\top \e_y = \bE^\top\bE_y (\bE_y^\top\bE_y)^{-1}\mathbf{1}_q.$$
    while 
    $$\bE'^\top \e_y' = \bE^\top\bm{R}_{ij}^\top\bm{R}_{ij}\bE_y (\bE_y^\top\bm{R}_{ij}^\top\bm{R}_{ij}\bE_y)^{-1}\mathbf{1}_q=\bE^\top\bE_y (\bE_y^\top\bE_y)^{-1}\mathbf{1}_q.$$
    The second identity is due to $\bm{R}_{ij}^\top\bm{R}_{ij}=\bI_{d_e}$. So the vector inside $\sm$ is the same, so the outputs are the same.

    For $\bE_y^\top \bE_y =\bE_y'^\top \bE_y'$, it's similar:
    $$\bE_y^\top \bE_y =\bE_y^\top\bm{R}_{ij}^\top\bm{R}_{ij}\bE_y=\bE_y'^\top \bE_y'$$
\end{proof}

\begin{lemma}[Off-diagonal entries of the gradient have 0 expectation]
\label{Lemma: Off-diagonal entries of the gradient have 0 expectation}
Given a $q$-subset $y$, for any function $f(\S_y(\bE)):\R^{d_e\times T}\to\R$, we have for any $i\in[T]$:
$$u_k^\top \E_{\bE}^{(R)}\qty[f(\S_y(\bE))\e_i\e_y^\top]u_j^\top=0,j\not= k.$$
where $u_i$ is the one-hot vector $(\mathds{1}\{i=1\},\mathds{1}\{i=2\},...,\mathds{1}\{i=d_e\})$.
\end{lemma}
\begin{proof}
    First, the exact expression the $(k,j)$-th entry in the matrix is:
    \begin{align*}
        u_k^\top \E_{\bE}^{(R)}\qty[f(\S_y(\bE))\e_i\e_y^\top]u_j & = \E_{\bE}^{(R)}\qty[f(\S_y(\bE)) \e_{ik}\mathbf{1}_q^\top (\bE_y^\top \bE_y)^{-1}\begin{bmatrix}
            \e_{y_1,j}\\
            \e_{y_2,j}\\
            \vdots\\
            \e_{y_q,j}
        \end{bmatrix}]
    \end{align*}
    where $\e_{i,j}$ denote the $j$-th entry of position encoding vector $\e_i$, $\begin{bmatrix}
            \e_{y_1,j},
            \e_{y_2,j},
            \cdots,
            \e_{y_q,j}
        \end{bmatrix}^\top$ as $\bE_y[j]$, and $y_i$ is the $i$-th index in the subset $y$. We expand the conditional expectation (Here $\Pr^{(R)}(\cdot):=\Pr(\cdot|\bE\text{ has $(q,\delta)$-RIP})$):
    \begin{align*}
         &u_k^\top \E_{\bE}^{(R)}\qty[f(\S_y(\bE))\e_i\e_y^\top]u_j\\ ={}& \E_{\bE}^{(R)}\qty[f(\S_y(\bE)) \e_{ik}\mathbf{1}_q^\top (\bE_y^\top \bE_y)^{-1}\bE_y[j]]\\
        ={}&\E_{\bE}^{(R)}\qty[f(\S_y(\bE)) \e_{ik}\mathbf{1}_q^\top (\bE_y^\top \bE_y)^{-1}\bE_y[j]\Big|\e_{ik}=\frac{1}{\sqrt{d_e}}]{\Pr}^{(R)}(\e_{ik}=\frac{1}{\sqrt{d_e}})\\
        +&\E_{\bE}^{(R)}\qty[f(\S_y(\bE)) \e_{ik}\mathbf{1}_q^\top (\bE_y^\top \bE_y)^{-1}\bE_y[j]\Big|\e_{ik}=-\frac{1}{\sqrt{d_e}}]{\Pr}^{(R)}(\e_{ik}=-\frac{1}{\sqrt{d_e}})
    \end{align*}
    Note that by expansion of condition expectation:
    \begin{align*}
        &\E_{\bE}^{(R)}\qty[f(\S_y(\bE)) \e_{ik}\mathbf{1}_q^\top (\bE_y^\top \bE_y)^{-1}\bE_y[j]\Big|\e_{ik}=\frac{1}{\sqrt{d_e}}]{\Pr}^{(R)}(\e_{ik}=\frac{1}{\sqrt{d_e}})\\
        ={}&\sum_{\xi_p,p\not=i}\E_{\bE}^{(R)}\qty[f(\S_y(\bE)) \e_{ik}\mathbf{1}_q^\top (\bE_y^\top \bE_y)^{-1}\bE_y[j]\Big|\e_{ik}=\frac{1}{\sqrt{d_e}},\e_{pk}=\xi_p,p\not=i]\\
        &\cdot {\Pr}^{(R)}\qty(\e_{pk}=\xi_p,p\not=i\Big|\e_{ik}=\frac{1}{\sqrt{d_e}}){\Pr}^{(R)}(\e_{ik}=\frac{1}{\sqrt{d_e}})\tag{$\xi_p$ is selected in $\pm \frac{1}{\sqrt{d_e}}$}\\
        ={}&-\sum_{\xi_p,p\not=i}\E_{\bE}^{(R)}\qty[f(\S_y(\bE)) \e_{ik}\mathbf{1}_q^\top (\bE_y^\top \bE_y)^{-1}\bE_y[j]\Big|\e_{ik}=-\frac{1}{\sqrt{d_e}},\e_{pk}=-\xi_p,p\not=i]\\
        &\cdot {\Pr}^{(R)}\qty(\e_{pk}=\xi_p,p\not=i,\e_{ik}=\frac{1}{\sqrt{d_e}})\tag{\Cref{Lemma: Flipping rows doesn't change the softmax}, $f(\S_y(\bE)),\bE_y^\top \bE_y$ are not changed, and $\bE_{y}[j]$ is unrelated to $\e_{pk}$}\\
        ={}&-\sum_{\xi_p,p\not=i}\E_{\bE}^{(R)}\qty[f(\S_y(\bE)) \e_{ik}\mathbf{1}_q^\top (\bE_y^\top \bE_y)^{-1}\bE_y[j]\Big|\e_{ik}=-\frac{1}{\sqrt{d_e}},\e_{pk}=-\xi_p,p\not=i]\\
        &\cdot {\Pr}^{(R)}\qty(\e_{pk}=-\xi_p,p\not=i,\e_{ik}=-\frac{1}{\sqrt{d_e}})\tag{Flipped row have same probability}\\
        ={}&-\E_{\bE}^{(R)}\qty[f(\S_y(\bE)) \e_{ik}\mathbf{1}_q^\top (\bE_y^\top \bE_y)^{-1}\bE_y[j]\Big|\e_{ik}=-\frac{1}{\sqrt{d_e}}]{\Pr}^{(R)}(\e_{ik}=-\frac{1}{\sqrt{d_e}})\\
    \end{align*}
    Therefore they cancel out and the expectation is 0.
\end{proof}

\begin{lemma}[Diagonal entries of the gradient are the same]
\label{Lemma: Diagonal entries of the gradient are the same}
Given a $q$-subset $y$, for any function $f(\S_y(\bE)):\R^{d_e\times T}\to\R$, we have for any $i\in[T],j,k\in [d_e]$:
$$u_k^\top \E_{\bE}^{(R)}\qty[f(\S_y(\bE))\e_i\e_y^\top]u_k^\top=u_j^\top \E_{\bE}^{(R)}\qty[f(\S_y(\bE))\e_i\e_y^\top]u_j^\top.$$
where $u_i$ is the one-hot vector $(\mathds{1}\{i=1\},\mathds{1}\{i=2\},...,\mathds{1}\{i=d_e\})$.
\end{lemma}

\begin{proof}
    First, the exact expression of the $(k,k)$-th entry in the matrix (which is diagonal) is:
    \begin{align*}
        u_k^\top \E_{\bE}^{(R)}\qty[f(\S_y(\bE))\e_i\e_y^\top]u_k & = \E_{\bE}^{(R)}\qty[f(\S_y(\bE)) \e_{ik}\mathbf{1}_q^\top (\bE_y^\top \bE_y)^{-1}\bE_y[k]]
    \end{align*}
    Consider another PE matrix $\bE' = \bm{R}_{kj}\bE$, which also satisfy the $(q,\delta)$-RIP and have same conditional probability with $\bE$. Then we have
    \begin{align*}
        u_k^\top \E_{\bE}^{(R)}\qty[f(\S_y(\bE))\e_i\e_y^\top]u_k & = \E_{\bE}^{(R)}\qty[f(\S_y(\bE)) \e_{ik}\mathbf{1}_q^\top (\bE_y^\top \bE_y)^{-1}\bE_y[k]]\\
        &=\E_{\bE'}^{(R)}\qty[f(\S_y(\bE')) \e_{ik}'\mathbf{1}_q^\top (\bE_y'^\top \bE_y')^{-1}\bE_y'[k]]\tag{Change of variable}\\
        &=\E_{\bE'}^{(R)}\qty[f(\S_y(\bE)) \e_{ik}'\mathbf{1}_q^\top (\bE_y^\top \bE_y)^{-1}\bE_y'[k]]\tag{\Cref{Lemma: Switching rows doesn't change the softmax}}\\
        &=\E_{\bE}^{(R)}\qty[f(\S_y(\bE)) \e_{ij}\mathbf{1}_q^\top (\bE_y^\top \bE_y)^{-1}\bE_y[j]]\tag{Change back the variable}\\
        &=u_j^\top \E_{\bE}^{(R)}\qty[f(\S_y(\bE))\e_i\e_y^\top]u_j^\top.
    \end{align*}
\end{proof}

\begin{lemma}[Conditional expectation of the softmax vector]
\label{Lemma: Conditional expectation of the softmax vector}
Given a $q$-subset $y$ and $\bE$ satisfies $(q,\delta)$-RIP, for the softmax probability vector $\S_y(\bE) = \sm(C\bE^\top \e_y)$, we have for any $i,j \in y$ or $i,j\not\in y$,
$$\E^{(R)}_{\bE}\qty[\S_y(i)] = \E^{(R)}_{\bE}\qty[\S_y(j)]$$
\end{lemma}
\begin{proof}
    For $i,j\in y$, we have $\e_i^\top \e_y = \e_j^\top \e_y=1$, so 
    $$\S_y(i) = \frac{e^{C}}{qe^C+\sum_{i\not \in y}e^{C\e_i^\top \e_y}}= \S_y(j)$$
    and thus their expectations are the same. For $i,j\not\in y$, considered a switched PE matrix $\bE'=\bm{R}_{ij}\bE$. By \Cref{Lemma: Switching rows doesn't change the softmax}, and the probability for $\E$ and $\E'$ are the same:
    \begin{align*}
        \E^{(R)}_{\bE}\qty[\S_y(i)] &=\E^{(R)}_{\bE}\frac{e^{C\e_i^\top \e_y}}{qe^C+\sum_{k\not \in y}e^{C\e_k^\top \e_y}}\\
        &=\E^{(R)}_{\bE}\frac{e^{C\e_j^\top \e_y}}{qe^C+\sum_{k\not \in y}e^{C\e_k^\top \e_y}} = \E^{(R)}_{\bE}\qty[\S_y(j)]
    \end{align*}
\end{proof}

\begin{lemma}[Estimation for softmax vector with $\W=C \randomWdirection$]
\label{Lemma: Estimation for softmax vector with W}
Given a $q$-subset $y$ and $\bE$ satisfies $(q,\delta)$-RIP, for the softmax probability vector $\S_y(\bE) = \sm(C\bE^\top \e_y)$, we have for any $i\in y$:
    \begin{align*}
          \frac{1}{q+(T-q)e^{-\frac{1-3\delta}{1-2\delta}C}} &\leq \S_y(i)\leq \frac{1}{q+(T-q)e^{-\frac{1-\delta}{1-2\delta}C}},   \\
          \frac{1}{q+(T-q)e^{-\frac{1-3\delta}{1-2\delta}C}}&\leq \E^{(R)}_{\bE}\qty[\S_y(i)]\leq\frac{1}{q+(T-q)e^{-\frac{1-\delta}{1-2\delta}C}}.
    \end{align*}
\end{lemma}
\begin{proof}
By \Cref{Lemma: RIP matrices has dual certificate}, we know $\e_i^\top \e_y = 1$ for $i\in y$, and $\left|\e_i^\top \e_y\right|\leq \frac{\delta}{1-2\delta}$. Then we know for any $C$, we have
\begin{align*}
\frac{e^C}{qe^C+(T-q)e^{\frac{\delta}{1-2\delta}C}} &\leq \S_y(i)\leq \frac{e^C}{q+(T-q)e^{-\frac{\delta}{1-2\delta}C}},   \\
 \Leftrightarrow   \frac{1}{q+(T-q)e^{-\frac{1-3\delta}{1-2\delta}C}} &\leq \S_y(i)\leq \frac{1}{q+(T-q)e^{-\frac{1-\delta}{1-2\delta}C}}
\end{align*}
Since all individual value is bounded within $\qty[\frac{1}{q+(T-q)e^{-\frac{1-3\delta}{1-2\delta}C}}, \frac{1}{q+(T-q)e^{-\frac{1-\delta}{1-2\delta}C}}]$, the expectation has the same bound. Thus, we complete the proof.
\end{proof}
\subsection{From stochastic to fixed: after training}\label{appen sec: from stochastic to fix}
Here, we continue the discussion at the end of \Cref{sec: stochastic positional encoding}: our one-layer transformer with stochastic positional encoding can be `derandomized' after the two layers $\W,\V$ are already trained. That is, when in the training phase, we use the stochastic positional encoding in the model and run gradient descent till convergence. But after training, we can keep all the parameters unchanged, sample only one near-orthogonal $\bE$ satisfying $(q,\delta)$-RIP, fix it as the set of PE, and use that model as our trained model. The following corollary of \Cref{thm: joint training stochastic pe} can characterize the perfect interpolation ability of the trained transformer.
\begin{corollary}\label{corollary: spe to fixed}
    Under the condition of \Cref{thm: joint training stochastic pe}, after time $t\geq \tilde{O}(\frac{T^{\frac{2-2\delta}{1-3\delta}}}{\eta}+\frac{T^2 d}{\eta\epsilon})$, we keep $$\W(t)=C(t)\randomWdirection,\V(t)=\alpha(t)\randomVdirection$$ unchanged. Then, consider any arbitrary $\bE\in \R^{d_e\times T}$ satisfying $(q,\delta)$-RIP, consider the model $$f_{\bm{\theta}}(\X,y)=\V\Z\sm(\Z^\top \W \zq)$$
    we have the non-stochastic training objective \Cref{eqn: training objective for qsa}:
    $$\mathcal{L}_{T,q}(\bm{\theta}) = \frac{1}{2}\E_{\X,y\sim\mathcal{D}_{T,q}}\qty[\|\qsa(\X,y)-f_{\bm{\theta}}(\X,y)\|_2^2]\leq \epsilon.$$
\end{corollary}
\begin{proof}
    Use \Cref{corollary: joint train param}, we have for any $y$ and $\bE$ satisfying $(q,\delta)$-RIP be bounded by \Cref{Lemma: Estimation for softmax vector with W}:
    $$\frac{1}{q+(T-q)e^{-\frac{1-3\delta}{1-2\delta}C}} \leq \S_y(i)\leq \frac{1}{q+(T-q)e^{-\frac{1-\delta}{1-2\delta}C}},\forall i\in y.$$
    Since $\W=C(t)\randomWdirection$, all $\S_y(i)$ are the same. Then, we denote $s_+=\S_y(i)$ and use the same argument in the proof of \Cref{thm: joint training stochastic pe} when verifying $\mathcal{L}(t)\leq \epsilon$ to complete the proof.
\end{proof}

This corollary implies that in practice, once we use the fresh sampled randomized PE in \citet{shen2023positional} or \citet{ruoss2023randomized} on our tasks, it's only necessary to sample the $\bE$ during training. At the time for evaluation, we can fix one valid PE, preventing us from do more sampling. This avoids `memorizing' all the random number within the model, and therefore does not violate the memory lower bound condition.

\section{Experiment details} \label{sec: experiment details appn}
In this section, we describe in detail our experimental setup, as well as present additional plots that were not highlighted in the main text. For all of our experiments, we use PyTorch \citep{paszke2019pytorch}, run on NVIDIA RTX A6000s.

\textbf{Setup.} For all of our experiments, we choose $T=200$ for our context length, and $q=3$, $d=5$, and $d_e = 170.$ The trainable weight matrices $\W$ and  $\V$ have two initialization case: zero-initialization at $\0$, or randomly initialize with the standard deviation $\sigma^2 = \frac{1}{d+d_e}$. 

At each iteration, we sample a batch of $n=128$ training data points $(\X, y)$, and for each OOD task, we fix a set of $n_{\text{test}}=128$ data points $(\X, y)$. For our length generalization tasks, we look at $T_{\text{test}}=\{250, 300, 350, 400\}$, and for our $q$-generalization tasks, we look at $q_{\text{test}} = \{5, 6, 7, 8\}$. We use the reparameterized model \cref{main eqn: practical tf}:
\begin{align*}
    f (\X,y)=\V\Z\sm(\Z^\top \W \z_{\text{query}}),\ \Z = [\X^\top \ \bE^\top]^\top.
\end{align*}
where we use the whole input matrix and the query:
\begin{equation}
    [\Z,\z_{\text{query}}] := \begin{bmatrix}
        \x_1 &\x_2 &\cdots &\x_{T-1}&\x_T&\x_{\text{query}}\\
        \e_{1}& \e_{2}& \cdots&\e_{T-1}&\e_{T}&\e_{y}
    \end{bmatrix} \in \mathbb{R}^{(d+d_e) \times (T+1)}.
\end{equation}

\textbf{Positional encoding sampling.} We also enforce a weaker version of the RIP assumption by sampling positional encodings so that they are pairwise near-orthogonal, as defined by some dot product threshold hyperparameter; it turns out that such a choice of positional encodings will satisfy the restricted isometry and orthogonality property, as we would expect for a near-orthogonal matrix. In practice, this is implemented using rejection sampling.

For the experiments where we run with a fixed architecture, the train and test positional encodings are fixed beforehand. 
For the experiments where stochastic architecture is used, this is simulated at each iteration by sampling $T$ positional encodings for training at each iteration and to be used for the entire $(\X, y)$ batch at that iteration, and for validation, a single set of $T_{\text{test}}$ positional encodings is sampled for each validation set at each iteration. 

\textbf{Training details.}
We run simulations on three different settings: 

(1) Attending the entire input matrix $[\Z,\z_{\text{query}}]$ as described in \Cref{main eqn: practical tf} and training with SGD (with zero initialization/random initialization of $\W$ and $\V$).

(2) Attending the entire input matrix $[\Z,\z_{\text{query}}]$ as described in \Cref{main eqn: practical tf} and training with Adam (with zero initialization/random initialization of $\W$ and $\V$).

(3) An additional experiment run on smaller $d,d_e$ for illustration and training with small random initialization and SGD with annealing.

\textbf{Fixed vs. stochastic architecture.}
For the fixed architecture, we sample $T_{\max}=400$ positional encodings at the start of training. For each iteration of training, we use the first 200 for the fixed architecture, and for each validation task, we use the respective prefix of positional encodings ($T=250, 300, 350, 400$). For the stochastic architecture, we sample positional encodings at every step for training, as well as for validation. When sampled, we use rejection sampling to ensure the positional encoding matrix satisfy near-orthogonality.

Our experiments demonstrate that even though both fixed and stochastic PE architectures lead to in-distribution population loss of 0, the out-of-distribution validation performances are different, both for length generalization and generalization on unseen $q_{\text{test}}$-subsets. The following sections describe the experiments that were run.  

\subsection{SGD from zero initialization}
\label{appendix subsec: sgd from 0 init}
When we attend $[\Z,\z_{\text{query}}]$ and train with GD, we run with $\eta=1$, then annealing to $\eta=1/3$ at iteration 50000. We run until iteration 100000. The learning rate schedule is to prevent potential instability, for instance via the edge-of-stability phenomenon \citep{cohen2021gradient, li2022analyzing, damian2022self} or loss spikes in transformer training \citep{chowdhery2023palm, dehghani2023scaling, wortsman2024smallscale}.

\begin{figure}[H]
    \centering
    \includegraphics[width=\textwidth]{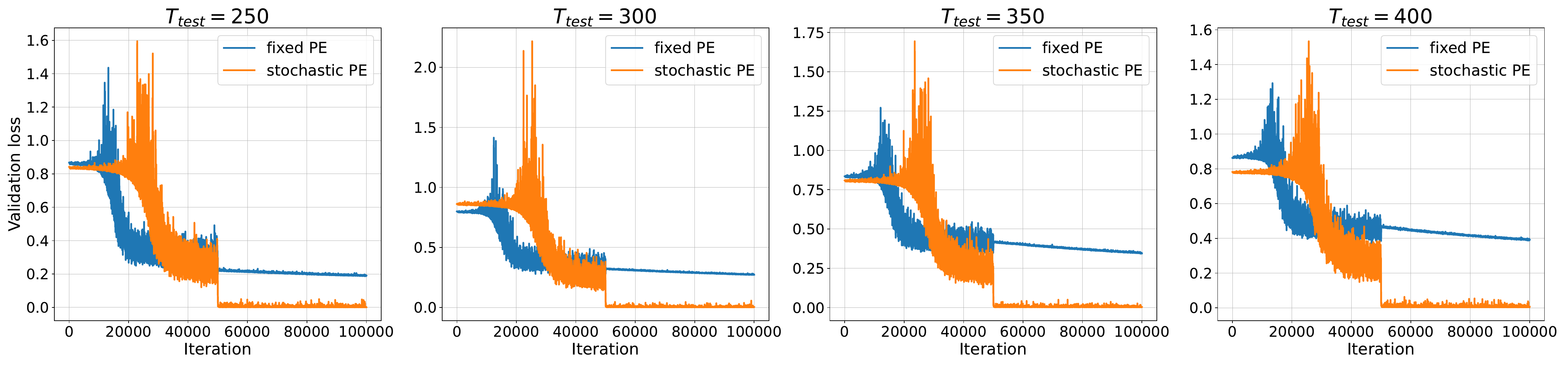}
    \includegraphics[width=\textwidth]{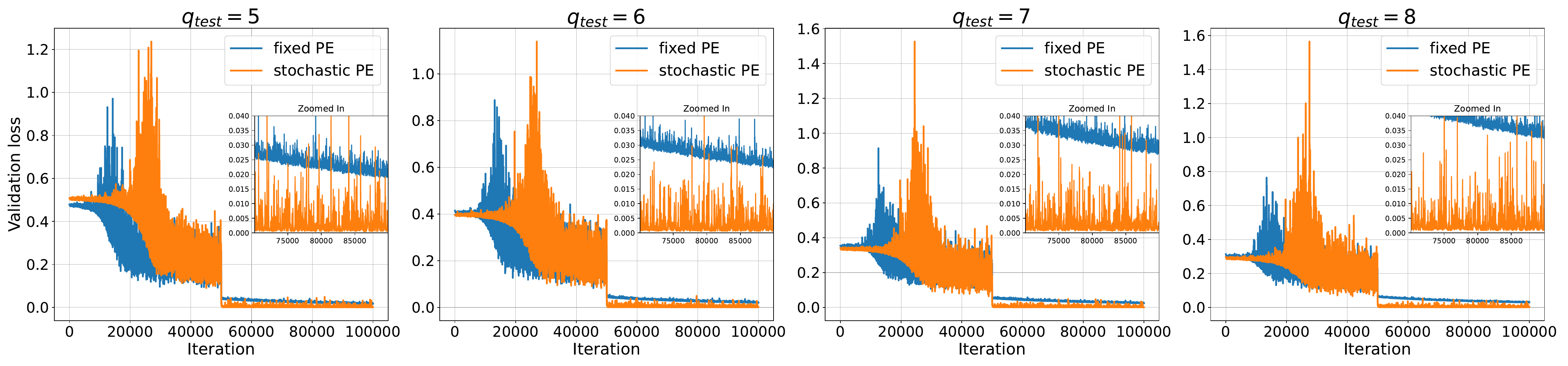}
    \caption{The \textbf{length generalization} performance and OOD performance on \textbf{unseen $q_{\text{test}}$-subsets}. 
    \textbf{Top: Length generalization.} Note that stochastic PE converges to 0 validation loss, whereas a fixed PE is unable to do so; all of the fixed PE end up with validation loss at least 0.15. \textbf{Bottom: Generalization to unseen $q_{\text{test}}$-subsets.} Note that while both stochastic and fixed PE can converge to 0 validation loss in the long run, stochastic PE converges slightly quicker, as seen by the zoomed in versions of the plots near the end of training. Additionally, the fixed PE's validation performance gets worse as $q_{\text{test}}$ increases. }
    \label{fig: first appendix exp}
\end{figure}

\subsection{Adam from zero initialization}

Here, we attend $[\Z,\z_{\text{query}}]$ and train with default Adam settings until iteration 100000.

\begin{figure}[H]
\vspace{-0.8cm}
    \centering
    \includegraphics[width=\textwidth]{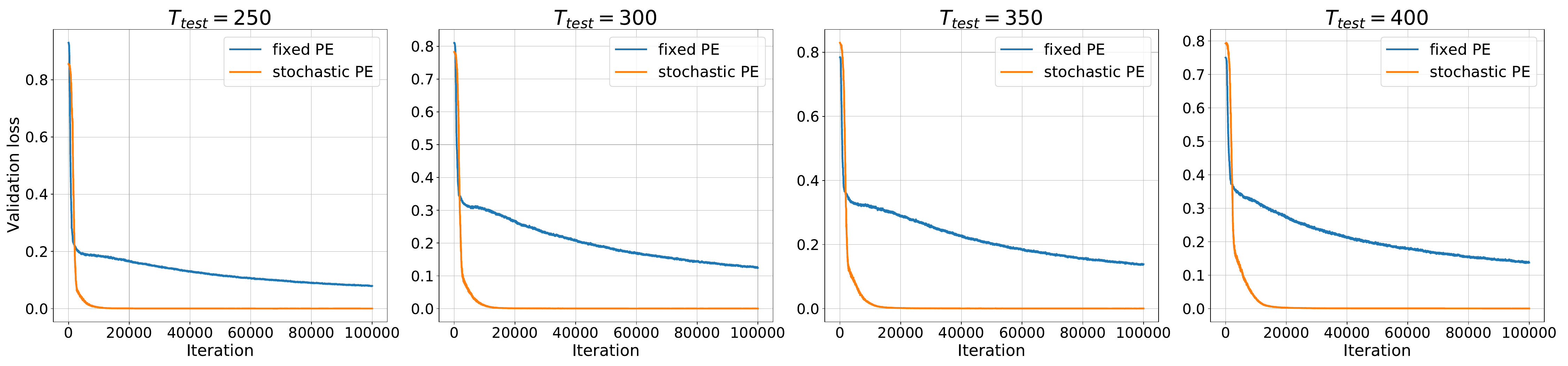}
    \includegraphics[width=\textwidth]{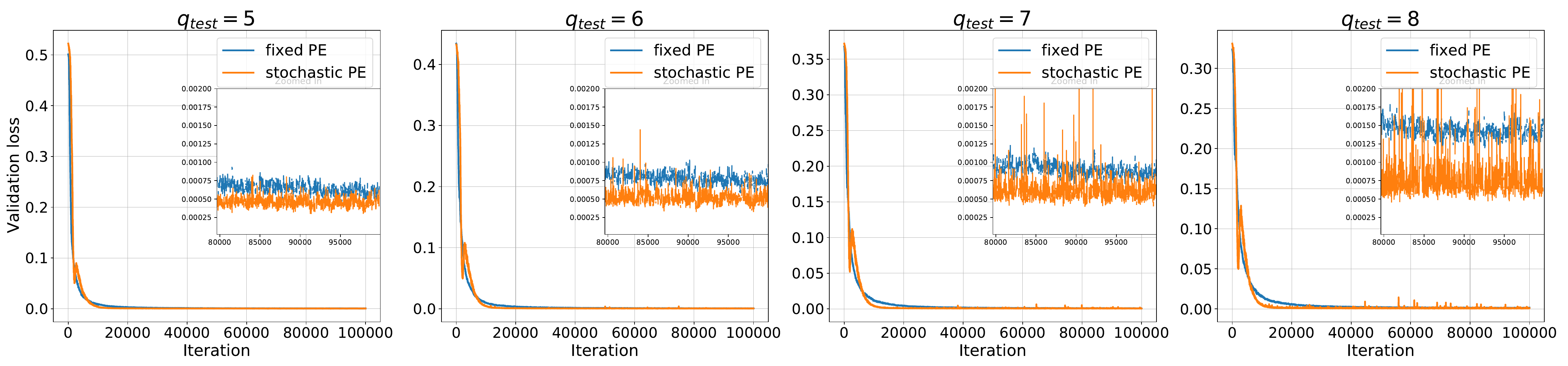}
    \vspace{-0.5cm}
    \caption{The training trajectory of Adam. The length generalization advantage for stochastic positional encoding is similar to the description of \Cref{fig: first appendix exp}. While Adam may allow the validation loss for the $q_{\text{test}}$ to converge to 0 in the long run, for all practical purposes related to early stopping, stochastic PE dominates in such OOD performance. }
\end{figure}
\vspace{-0.8cm}
\subsection{SGD from random initialization}
Similar to \Cref{appendix subsec: sgd from 0 init}, we attend $[\Z,\z_{\text{query}}]$ and train with GD, we run with $\eta=1$, then annealing to $\eta=1/3$ for iteration 50000 to 100000. We observe similar results as in \Cref{appendix subsec: sgd from 0 init}.
\begin{figure}[H]
\vspace{-0.2cm}
    \centering
    \includegraphics[width=\textwidth]{nonzero_init_GD/T.pdf}
    \includegraphics[width=\textwidth]{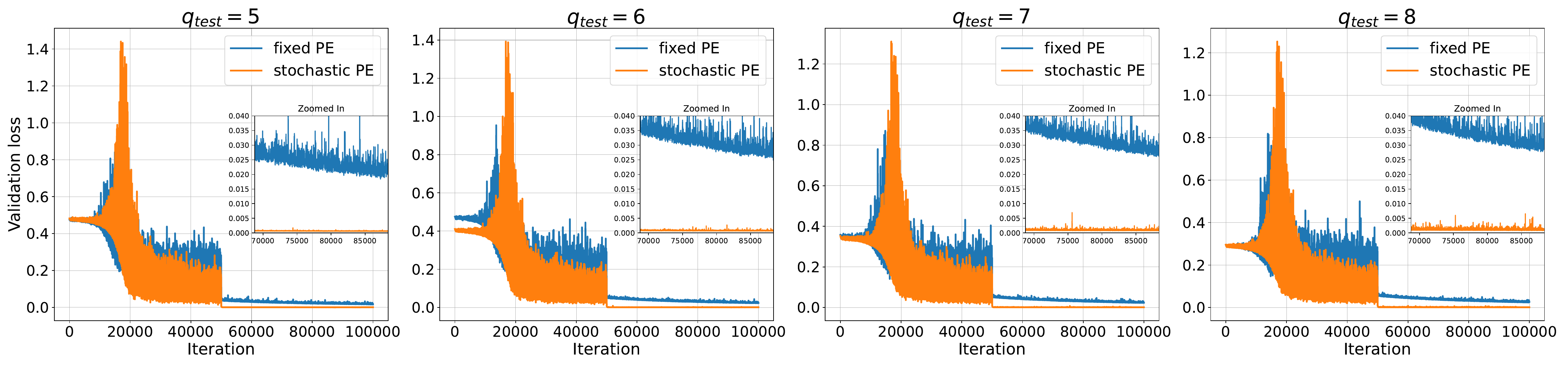}
    \vspace{-0.8cm}
    \caption{The training trajectory of SGD with random initialization. See description of \Cref{fig: first appendix exp}. Note that the overall dynamics are similar to the zero initialization case; as before, the length generalization advantage of stochastic positional encoding is evident. Moreover, stochastic positional encoding achieves better out-of-distribution loss compared to fixed positional encoding.}
\end{figure}

\vspace{-0.8cm}
\subsection{Adam from random initialization}
\begin{figure}[H]
\vspace{-0.2cm}
    \centering
    \includegraphics[width=\textwidth]{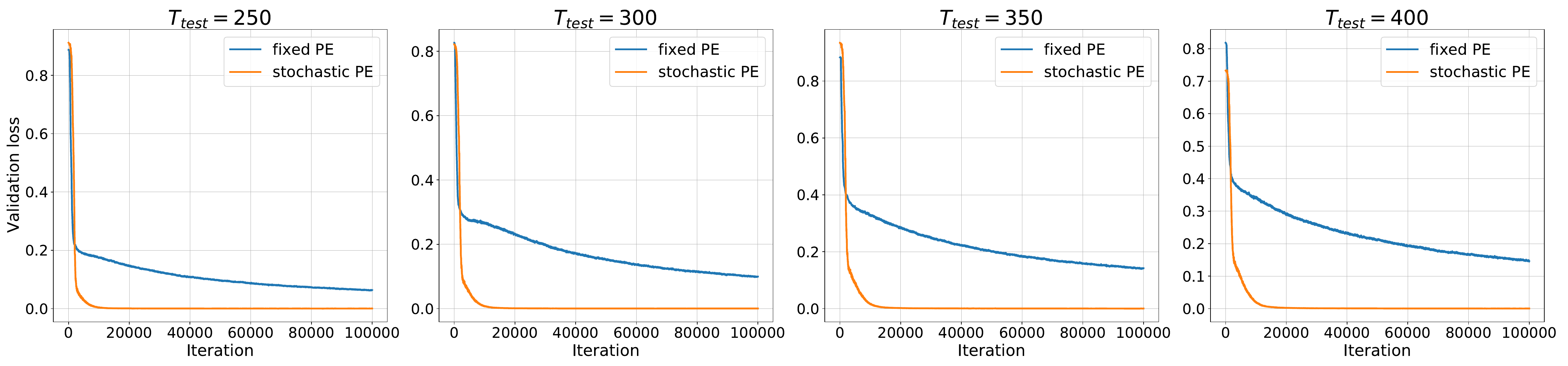}
    \includegraphics[width=\textwidth]{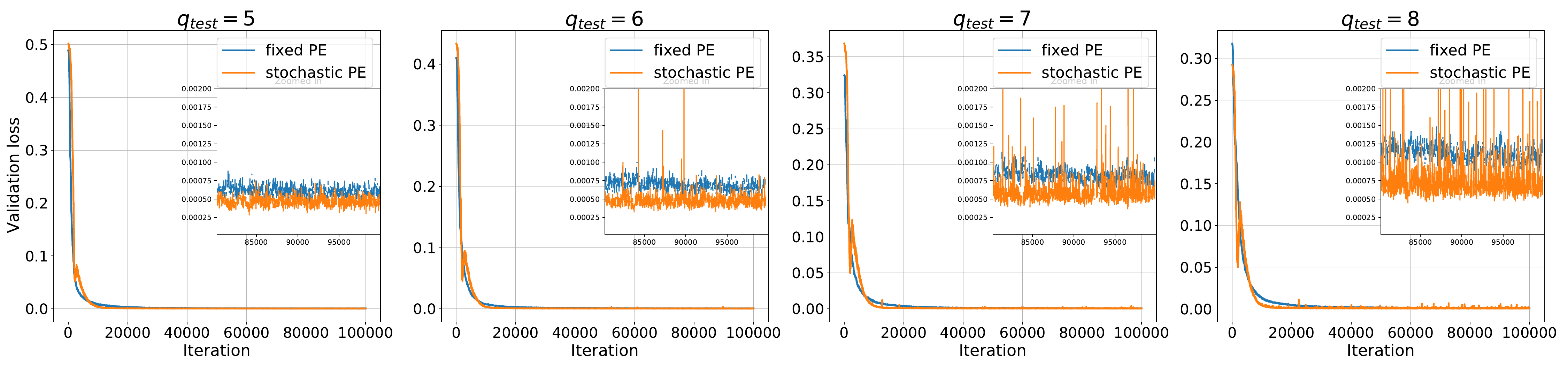}
    \vspace{-0.8cm}
    \caption{The training trajectory of Adam with random initialization is similar to the zero initialization case. See description of \Cref{fig: first appendix exp}. While Adam may allow the validation loss for the $q_{\text{test}}$ to converge to 0 in the long run, for all practical purposes related to early stopping, stochastic PE dominates in such OOD performance. }
\end{figure}

\subsection{Additional heat maps in the setting of \Cref{fig:heat map}}
Here, for the sake of illustrations, we train with $T=10$, $d=20$, and $d_e=20$. In addition, we use small Gaussian initialization with an entrywise standard deviation of $1/\sqrt{d_e}$. The heat maps can be seen as follows; GD eventually converges to the ground truth directions of $\W^\star$ and $\V^\star.$

\begin{figure}[H]
    \centering
    \includegraphics[width=0.23\textwidth]{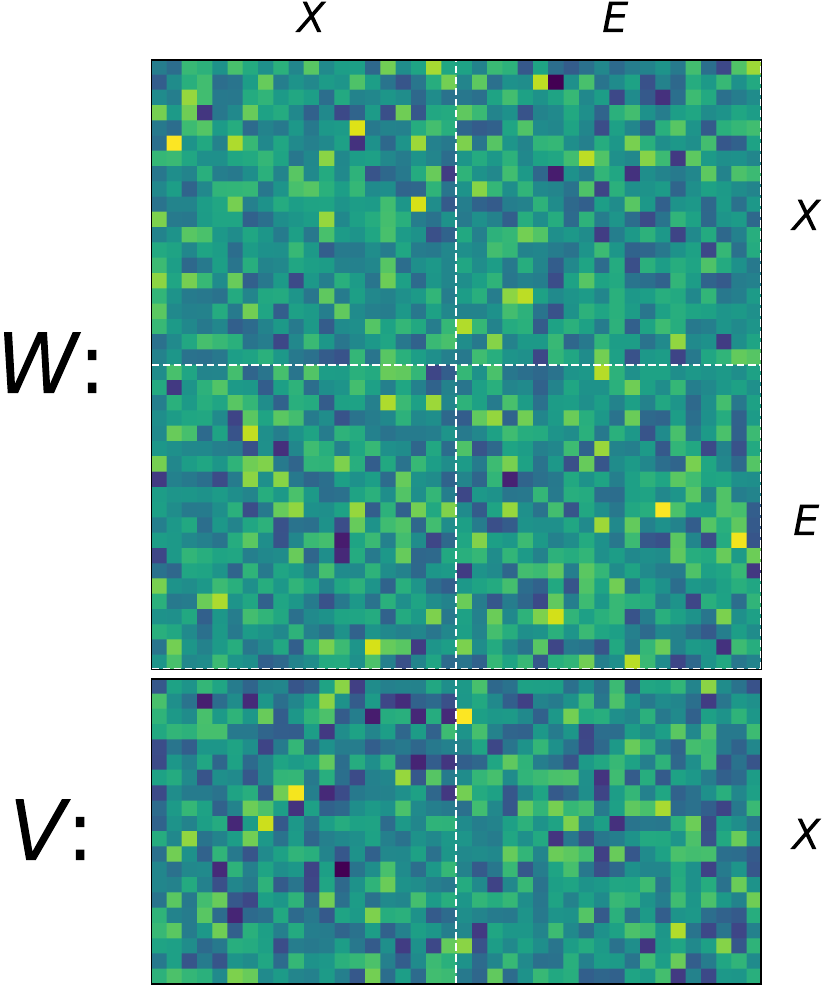}
    \includegraphics[width=0.23\textwidth]{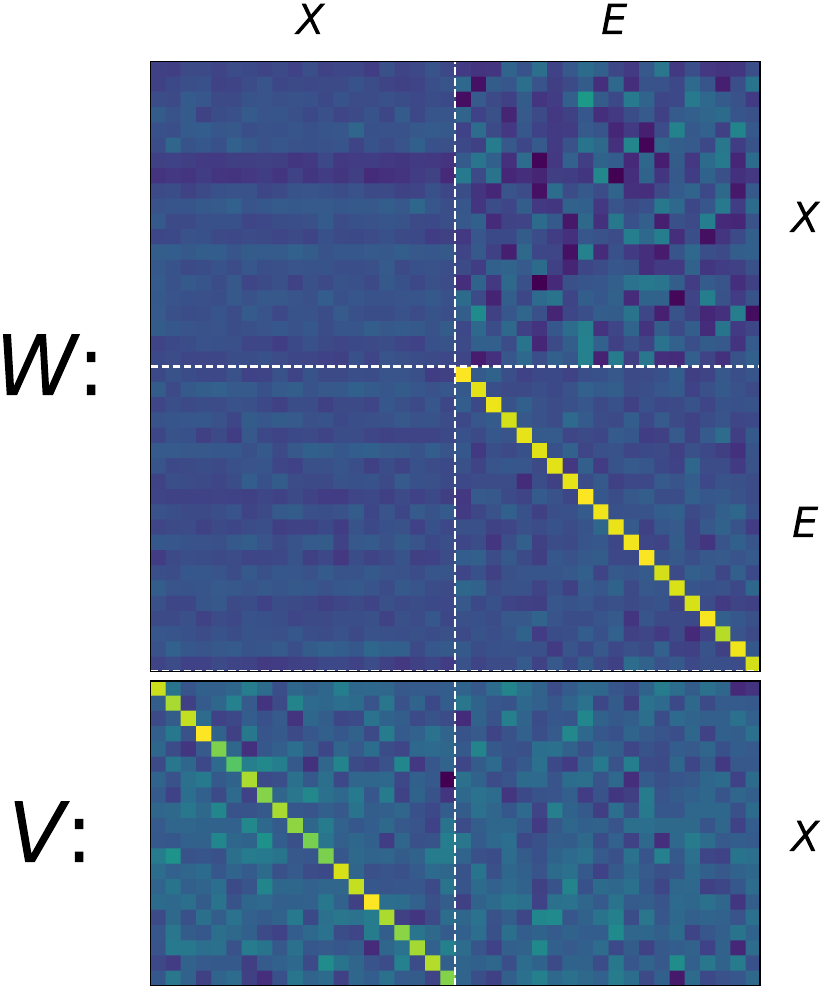}
    \includegraphics[width=0.23\textwidth]{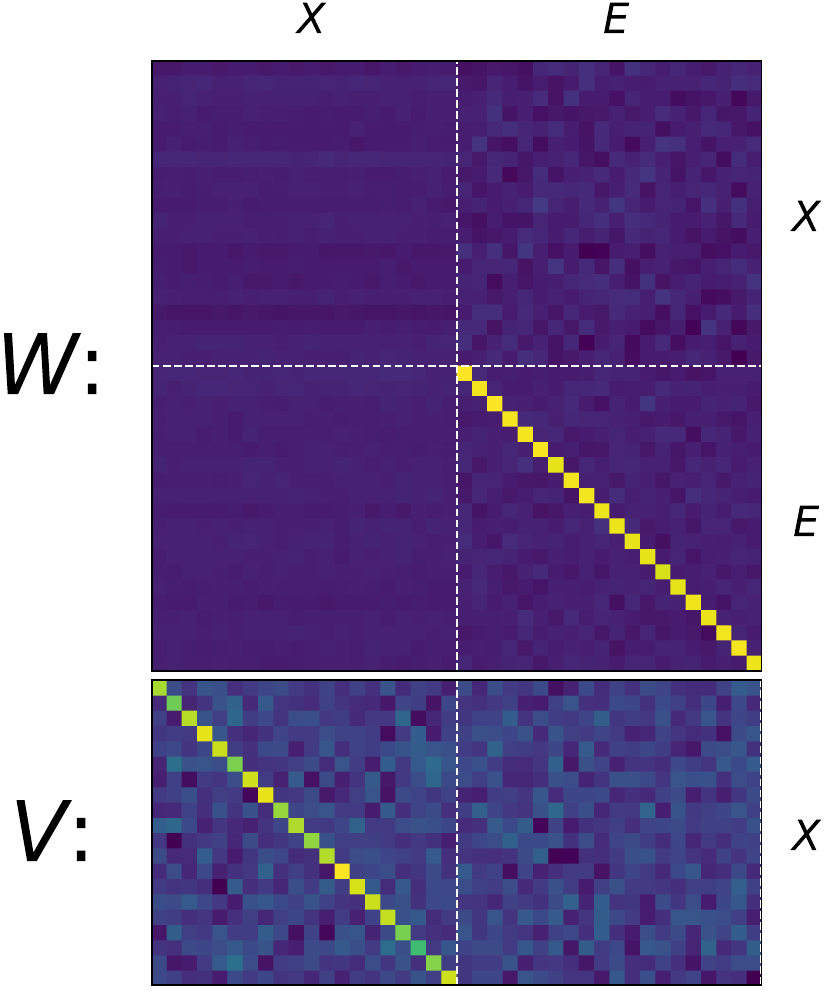}
    \includegraphics[width=0.23\textwidth]{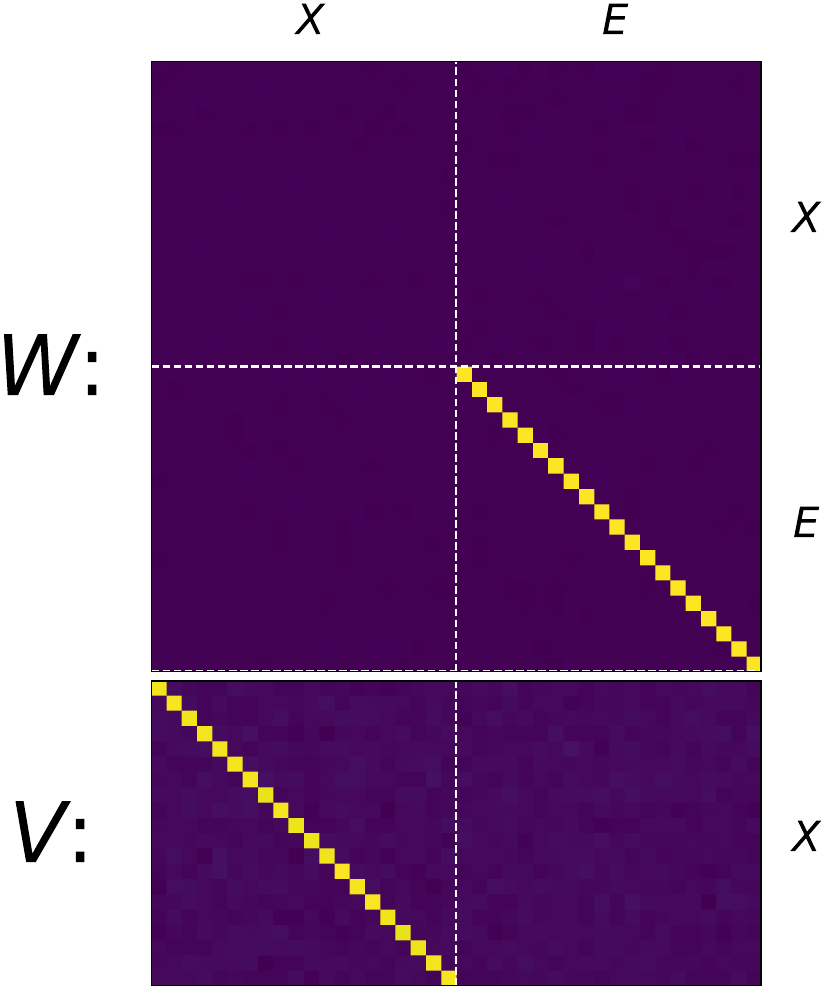}
    
    \caption{\textbf{More figures in interpretable training (iterations $t=0, 100, 500, 50000$):} For the practical model \Cref{main eqn: practical tf}, we present the heat map of the self-attention layer $\W$ and the value matrix $\V$ after convergence. We initialize $\W,\V$ randomly at $t=0$. We can observe that during training, in $\W$ only the \textbf{sub-block that attends to the positional encodings}, $\bE$ gradually converges to identity $\bI_{d_e}$ direction, while all other entries gradually converge to near 0. Similar phenomenon happens in $\V$, only the \textbf{sub-block that attends to the input tokens $\X$} gradually converges to $\bI_d$ direction with all other converging to near 0.}
    \label{fig:enter-label}
\end{figure}

\end{document}